\documentclass[final, sn-mathphys,Numbered, english,iicol]{sn-jnl}

\usepackage{graphicx}%
\usepackage{multirow}%
\usepackage{amsmath,amssymb,amsfonts}%
\usepackage{amsthm}%
\usepackage{mathrsfs}%
\usepackage[title]{appendix}%
\usepackage{xcolor}%
\usepackage{textcomp}%
\usepackage{manyfoot}%
\usepackage{booktabs}%
\usepackage{algorithm}%
\usepackage{algorithmicx}%
\usepackage{algpseudocode}%
\usepackage{listings}%

\usepackage[T1]{fontenc}
\usepackage{lmodern}
\usepackage{martin}
\usepackage[mathcal]{euscript}
\usepackage{longtable}
\usepackage{multirow}
\usepackage{amssymb}
\usepackage{pifont}
\usepackage{widetext}

\AfterEndEnvironment{strip}{\leavevmode}
\setcitestyle{authoryear}

\begin{document}

\title{A Classification of Critical Configurations for any Number of Projective Views}


\author*[1]{\fnm{Martin} \sur{Bråtelund}}\email{mabraate@math.uio.no}

\affil*[1]{\orgdiv{Department of Mathematics}, \orgname{University of Oslo}, \orgaddress{\street{Moltke Moes vei 35}, \city{Oslo}, \postcode{0316}, \country{Norway}}}



\maketitle

\setlength{\parindent}{0 pt }
\setlength{\parskip}{2.25 ex plus 0.75 ex minus 0.3 ex }


\begin{abstract}{}
Structure from motion is the process of recovering information about cameras and 3D scene from a set of images. Generally, in a noise-free setting, all information can be uniquely recovered if enough images and image points are provided. There are, however, certain cases where unique recovery is impossible, even in theory; these are called \emph{critical configurations}. We use a recently developed algebraic approach to classify all critical configurations for any number of projective cameras. We show that they form well-known algebraic varieties, such as quadric surfaces and curves of degree at most 4. This paper also improves upon earlier results both by finding previously unknown critical configurations and by showing that some configurations previously believed to be critical are in fact not.
\end{abstract}

\section*{Acknowledgments}
I would like to thank my two supervisors, Kristian Ranestad and Kathlén Kohn, for their help and guidance, for providing me with useful insights, and for their belief in my work. This work was supported by the Norwegian National Security Authority.

\section{Introduction}
One of the oldest and most well-studied problems in computer vision is that of \emph{structure from motion}, where given a set of $2$-dimensional images the task is to reconstruct a scene in $3$-space and find the camera positions in this scene. Early works on this problem predate even the first computers, with works going back to Taylor and Lagrange \cite{sturm2011history}. Over time, many techniques have been developed for solving these problems for varying camera models and under different assumptions on the space and image points \cite{maybank1992theory, maybank1993theory, hartley1997lines, hartleyzisserman}. In general, when assuming noiseless images and with enough images and enough points in each image, one can uniquely recover all information about the original scene. There are, however, some configurations of cameras and points whose 3D information can not be uniquely recovered from the images. These are called \emph{critical configurations}.

Much work has been done to understand critical configurations for various settings \cite{buchanan1988twistedcubic, hartley2000, hartleyKahlAstrom2001, HKCalibrated, bertolini2007criticalOneView, HK, bertolini2020criticalRank, BunchananCriticalLines, duff2020pl, duff2019plmp, zuzana2008, larsson2017efficient}, with results dating as far back as 1941 \cite{Krames1941}. While interesting from a purely theoretical viewpoint, critical configurations also play a part in practical applications. In real-life reconstruction problems, critical configurations rarely show up when enough data is available (noise alone should be enough to avoid the critical configurations), nonetheless, it has been shown that as the configurations approach the critical ones, reconstruction algorithms can become less stable \cite{stability, HK, bertoliniStability}. Moreover, when doing reconstruction with little data available, the chances of running into critical configurations greatly increase.
 
We study the critical configurations for projective cameras observing points using the approach from \cite{twoViews, threeViews}. The idea is as follows (full details in \cref{sec:approach}): A set of $n$ cameras $\textbf{P}$ defines a rational map $\phi_{\textbf{P}}\from\p3\dashrightarrow(\p2)^n$. A set of $n$ cameras $\textbf{P}$ and a set of points $X\subset\p3$ is referred to as a \emph{configuration}. A configuration is critical if there exists another set of cameras $\textbf{Q}$ and points $Y$ such that $\phi_{\textbf{P}}(X)=\phi_{\textbf{Q}}(Y)$. The map $\phi_{\textbf{P}}$ takes $\p3$ to a variety in $(\p2)^{n}$ called the \emph{multi-view variety}. Since $X$ and $Y$ both map into the intersection of the multi-view varieties of $\phi_{\textbf{P}}$ and $\phi_{\textbf{Q}}$, one can classify critical configurations by classifying all possible intersections of multi-view varieties.

Using this approach, we give a complete classification of the critical configurations for any number of projective views. An earlier classification appears in \cite{HK}, using a different approach. While \cite{HK} laid the groundwork and served as an inspiration for this paper, the classification it gives is incomplete. It is missing some of the critical configurations for three or more views, and more importantly, the main theorem for critical configurations for four or more views does not hold. This is explained in more detail in section \cref{sec:counterexample}, where we give a counterexample. 

A complete classification for the case of one and two views is given in \cite{twoViews}, and a classification for three views is given in \cite{threeViews}, although the latter is missing at least one critical configuration. For the sake of completeness, the results from these two works are included in this paper, sometimes with new proofs. The classification of the critical configurations for two views can be found in \cref{thr:critical_conf_for_two_views,fig:critical_quadrics}, for three views in \cref{thr:critical_configuration_three_views,fig:12critical}, and for four or more views in \cref{thr:critical_configuration_four_views}.

\cref{sec:background} introduces notation and the main concepts and gives some preliminary results. \cref{sec:approach} describes the approach used for the classification of critical configurations. \cref{sec:one_view_case} briefly covers the one-view case before the two-view case is covered in \cref{sec:preliminary_two_views}. Here we introduce the fundamental matrix and show that for two views, the critical configurations lie on quadric surfaces. \cref{sec:preliminary_three_views,sec:critical_configurations_for_three_views} cover the case of three cameras showing that the critical configurations for three views tend to lie on curves of degree at most four (see \cref{fig:12critical}). Lastly, \cref{sec:four_views} shows that most of these configurations remain critical when more cameras are added, thus completing the classification.
%

\section{Background}
\label{sec:background}
As this paper deals with a very similar topic as \cite{twoViews, threeViews}, this section is largely the same as Section 2 in those works. We refer the reader to \cite{cox1994ideals} for the basics on ideals and varieties and to \cite{hartleyzisserman} for the basics on algebraic vision and multi-view geometry.

Let $\mathbb{C}$ denote the complex numbers, and let $\p{n}$ denote the projective space over the vector space $\mathbb{C}^{n+1}$. Projection from a point $p\in\p3$ is a linear map
\begin{equation*}
P\from \p3\dashrightarrow\p2.
\end{equation*}
We refer to such a projection and its projection center $p$ as a \emph{camera} and its \emph{camera center} in $\p3$\footnote{Following established terminology, we use the words \emph{camera} and \emph{view} interchangeably}. Following this theme, we refer to points in $\p3$ as \emph{space points} and points in $\p2$ as \emph{image points}. Similarly, $\p2$ will be referred to as an \emph{image}.

Once coordinates are chosen in $\p3$ and $\p2$, a camera can be represented by a $3\times4$ matrix of full rank called the \emph{camera matrix}. The camera center is then given as the nullspace (or kernel) of the matrix. For the most part, we make no distinction between a camera and its camera matrix, referring to both simply as cameras. We use the \emph{real projective pinhole camera model}, meaning that we require a camera matrix to be of full rank and to have only real entries.

\begin{remark}
\label{rem:basis_chosen}
Throughout the paper, whenever we talk about cameras it is to be understood that a choice of coordinates has been made, both on the images and 3-space, so that each camera comes with a camera matrix.
\end{remark}

\begin{remark}
Throughout the paper, when dealing with multiple cameras, we always assume that all camera centers are distinct.
\end{remark}

\begin{definition}
Given an $n$-tuple of cameras $\textbf{P}=(P_1,\ldots,P_n)$, with camera centers $p_1,\ldots,p_n$. We define the \emph{joint camera map} to be the map
\begin{align*}
\phi_\textbf{P}=P_1\times\cdots\times P_n\from\p3&\dashrightarrow(\p2)^{n},\\
x&\mapsto (P_1(x),\ldots ,P_n(x)).
\end{align*}
\end{definition}

\begin{lemma}
\label{lem:joint-camera_map_is_isomorphism}
Let $(P_1,\ldots,P_n)$ be a tuple of $n\geq2$ cameras. If the camera centers do not all lie on a line, the joint camera map $\phi_\textbf{P}$ is injective. If the camera centers all lie on a line, the joint camera map sends this line to a point and is injective everywhere else.
\end{lemma}

\begin{proof}
Let $\textbf{x}=(x_1,\ldots,x_n)\in(\p2)^{n}$ be a point in the image of $\phi_\textbf{P}$. The preimage of $\textbf{x}$ is the intersection $\bigcap\limits_{i=1}^{n}l_{P_i}$ where $l_{P_i}=\overline{P_i^{-1}(x_i)}$ is a line through $p_i$. As such, they intersect in a single point unless they are all equal to a line passing through all camera centers. Hence the joint camera map is injective outside of a potential line spanned by the camera centers.
\end{proof}

\begin{definition}
We call the closure of the image of the joint camera map $\phi_\textbf{P}$ the \emph{multi-view variety} of $P_1,\ldots,P_n$, and denote it by $\mathcal{M}_P$. The set of all multi-homogeneous polynomials vanishing on $\mathcal{M}_P$ is an ideal that we denote as the \emph{multi-view ideal}.
\end{definition}

\textbf{Notation.} While it follows from \cref{lem:joint-camera_map_is_isomorphism} that the multi-view variety is always irreducible, we use the term \emph{variety} to also include reducible algebraic sets.

\begin{theorem}[{\cite[Theorem 3.7]{Tomas}}]
\label{thr:ideal_of_multiview_variety}
The ideal of the multi-view variety is generated only by bilinear and trilinear polynomials. In particular, it is generated by the determinant of
\begin{align*}
\begin{bmatrix}
P_i'&\textnormal{\textbf{x}}&\textnormal{\textbf{0}}\\
P_j'&\textnormal{\textbf{0}}&\textnormal{\textbf{y}}
\end{bmatrix},
\end{align*}
for each pair of cameras, and the $7\times7$ minors of 
\begin{align*}
\begin{bmatrix}
P_i'&\textnormal{\textbf{x}}&\textnormal{\textbf{0}}&\textnormal{\textbf{0}}\\
P_j'&\textnormal{\textbf{0}}&\textnormal{\textbf{y}}&\textnormal{\textbf{0}}\\
P_k'&\textnormal{\textbf{0}}&\textnormal{\textbf{0}}&\textnormal{\textbf{z}}
\end{bmatrix},
\end{align*}
for each triple. Here \textnormal{\textbf{x}}, \textnormal{\textbf{y}} and \textnormal{\textbf{z}} are the $3\times1$ vectors with variables in the $i$-th, $j$-th and $k$-th image respectively.
\end{theorem}

\begin{definition}
\label{def:reconstruction}
Given a set of points $S\subset(\p2)^{n}$, a \emph{reconstruction} of $S$ is a configuration of cameras and points $(P_1,\ldots,P_n,X)$ such that $S=\phi_\textbf{P}(X)$ where $\phi_\textbf{P}$ is the joint camera map of the cameras $P_1,\ldots,P_n$. 
\end{definition}

\begin{remark}
By \cref{def:reconstruction} it is clear that for a set $S$ to have a reconstruction, it needs to lie in the image of some joint camera map. Following established conventions for the study of critical configurations, we slacken this requirement somewhat by considering a set to have a reconstruction as long as it lies in the \emph{closure} of the image of some joint camera map, in other words, in a multi-view variety. If $S$ lies in the closure, but not the image, at least one of the points $x\in X$ in the reconstruction will be a camera center.
\end{remark}

\begin{definition}
Given a configuration of cameras and points $(P_1,\ldots,P_n,X)$, we refer to $\phi_\textbf{P}(X)\subset(\p2)^{n}$ as the \emph{images} of $(P_1,\ldots,P_n,X)$.
\end{definition}

\begin{remark}
Every configuration of cameras and points is a reconstruction of its images, so every configuration (of cameras and points) is a reconstruction and vice versa.
\end{remark}

Given a set of image points $S\subset(\p2)^{n}$ as well as a reconstruction $(P_1,\ldots,P_n,X)$ of $S$, any scaling, rotation, translation, or more generally, any real projective transformation of $(P_1,\ldots,P_n,X)$ does not change the images, giving rise to a large family of reconstructions of $S$. However, we are not interested in differentiating between these reconstructions.

\begin{definition}
\label{def:equvalent_configurations}
Given a set of points $S\subset(\p2)^{n}$, let $(P_1,\ldots,P_n,X)$ and $(Q_1,\ldots,Q_n,Y)$ be two reconstructions of $S$. The two reconstructions are considered \emph{equivalent} if there exists an element $H\in\PGL(4)$, such that
\begin{align*}
&H(X)=Y,\\
&P_iH^{-1}=Q_i,\quad \forall i.
\end{align*}
\end{definition}

From now on, whenever we talk about a configuration of cameras and points, it is to be understood as unique up to action by $\PGL(4)$, and two configurations will be considered different only if they are nonequivalent. As such, we consider a reconstruction to be unique if it is unique up to action by $\PGL(4)$.

\begin{definition}
\label{def:conjugate_configuration/point}
Two configurations of cameras and points $(P_1,\ldots,P_n,X)$ and $(Q_1,\ldots,Q_n,Y)$ are called \emph{conjugate configurations} if they are nonequivalent reconstructions of the same set. That is, if they satisfy $\phi_\textbf{P}(X)=\phi_\textbf{Q}(Y)$, but not the conditions in \cref{def:equvalent_configurations}. Pairs of points $(x,y)\in X\times Y$ are called \emph{conjugate points} if $\phi_\textbf{P}(x)=\phi_\textbf{Q}(y)$.
\end{definition}

\begin{definition}
\label{def:critical_configuration}
A configuration of cameras and points $(P_1,\ldots,P_n,X)$ is said to be a \emph{critical configuration} if it has at least one conjugate configuration. A critical configuration $(P_1,\ldots,P_n,X)$ is said to be \emph{maximal} if there exists no critical configuration $(P_1,\ldots,P_n,X')$ such that $X\subsetneq X'$.
\end{definition}

Hence, a configuration is critical if and only if the images it produces do not have a unique reconstruction.

\begin{remark}
Various definitions of critical configurations exist. For instance, \cite{Krames1941} considers the cone with two cameras on the same line to be critical, while it fails to be critical by our definition. \cite{twoViews, threeViews} define it using the blow-up, which, when blowing back down, is equivalent to our definition. Our definition is the same as the one in \cite{HK}.
\end{remark}

\begin{definition}
\label{def:set_of_critical_points}
Let $\textbf{P}$ and $\textbf{Q}$ be two $n$-tuples of cameras. Define
\begin{align*}
I=\Set{(x,y)\in\p3\times\p3\mid\phi_\textbf{P}(x)=\phi_\textbf{Q}(y)}.
\end{align*}
The projection of $I$ to each coordinate gives us two varieties, $X$ and $Y$, where we denote $X$ as the \emph{set of critical points of $\textnormal{\textbf{P}}$ with respect to $\textnormal{\textbf{Q}}$}, and similarly for $Y$.
\end{definition}

This definition is motivated by the following fact:

\begin{proposition}
\label{prop:critical_points_give_critical_configuration}
Let $\textbf{P}$ and $\textbf{Q}$ be two $n$-tuples of cameras such that there is no $H\in\PGL(4)$ satisfying $P_iH^{-1}=Q_i$ for all $i$. Let $X$ be the critical points of $P$ with respect to $Q$ and conversely for $Y$. Then $(\textbf{P},X)$ is a critical configuration, with $(\textbf{Q},Y)$ as its conjugate. Furthermore, $(\textbf{P},X)$ is maximal with respect to $\textbf{Q}$ in the sense that if there exists a critical configuration $(\textbf{P},X')$ with $X\subsetneq X'$, then its conjugate consists of cameras different from $\textbf{Q}$.
\end{proposition}
\begin{proof}
(2.12 in \cite{twoViews}): It follows from \cref{def:set_of_critical_points} that for each point $x\in X$, we have a conjugate point $y\in Y$. Hence the two configurations have the same images. Inequivalence follows from the fact that the cameras lie in different orbits under $\PGL(4)$, so the second point in \cref{def:equvalent_configurations} can not be satisfied. Hence they are both critical configurations, conjugate to one another. 

The (partial) maximality follows from the fact that if we add a point $x_0$ to $X$ that does not lie in the set of critical points, there is (by \cref{def:set_of_critical_points}) no point $y_0$ such that $\phi_\textbf{P}(x_0)=\phi_\textbf{Q}(y_0)$.
\end{proof}

\cref{sec:preliminary_two_views,sec:critical_configurations_for_three_views,sec:four_views} classify all maximal critical configurations for 2, 3, and 4+ views respectively. The reason we focus primarily on the maximal ones is that every critical configuration is contained in a maximal one and (when working with more than one camera) the converse is true as well, any subconfiguration of a critical configuration is itself critical.

We conclude this section with a final, useful property of critical configurations, namely that the only property of the cameras we need to consider when exploring critical configurations is the position of their camera centers (i.e. change of coordinates in the images does not affect criticality). 

\begin{proposition}[{\cite[Proposition 3.7]{HK}}]
\label{prop:only_camera_centers_matter}
Let $(P_1,\ldots,P_n)$ be $n$ cameras with centers $p_1,\ldots,p_n$, and let $(P_1,\ldots,P_n,X)$ be a critical configuration. \newline If $(\mathcal{P}_1,\ldots,\mathcal{P}_n)$ is a set of cameras sharing the same camera centers, the configuration $(\mathcal{P}_1,\ldots,\mathcal{P}_n,X)$ is critical as well.
\end{proposition}
\begin{proof}
Since $P_i$ and $\mathcal{P}_i$ share the same camera center and the camera center determines the camera uniquely up to a choice of coordinates, there exists some $H_i\in\PGL(3)$ such that $\mathcal{P}_i=H_iP_i$. Let $(Q_1,\ldots,Q_n,Y)$ be a conjugate to $(P_1,\ldots,P_n,X)$. Then $(H_1Q_1,\ldots,H_nQ_n,Y)$ is a conjugate to $(H_1P_1,\ldots,H_nP_n,X)=(\mathcal{P}_1,\ldots,\mathcal{P}_n,X)$, so this configuration is critical as well.
\end{proof}

\section{Approach}
\label{sec:approach}
Consider a critical configuration $(P_1,\ldots,P_n,X)$. Since it is critical, there exists a conjugate configuration $(Q_1,\ldots,Q_n,Y)$ giving the same images in $(\p2)^{n}$. The two sets of cameras define two joint-camera maps $\phi_\textbf{P}$ and $\phi_\textbf{Q}$.
\begin{center}
\begin{tikzcd}[ampersand replacement=\&, column sep=small]
\p3 \arrow[rd, "\phi_\textbf{P}"] \& \& \p3 \arrow[ld, "\phi_\textbf{Q}"'] \\
 \& (\p2)^{n} \& 
\end{tikzcd}
\end{center}

The image of $X$ lies in the multi-view variety $\mathcal{M}_P\subseteq(\p2)^{n}$ and the image of $Y$ lies in $\mathcal{M}_Q$. As such, the two sets of points $X$ and $Y$ lie in such a way that they both map (with their respective maps) into the intersection of the two multi-view varieties $\mathcal{M}_P\cap\mathcal{M}_Q$.

\begin{remark}
If two sets of cameras the second equation in \cref{def:equvalent_configurations} their multi-view varieties are equal. Hence the multi-view varieties being different (which is the case we study throughout the paper) is enough to ensure that the two configurations are inequivalent.
\end{remark}

The two multi-view varieties $\mathcal{M}_P$ and $\mathcal{M}_Q$ are irreducible 3-folds, so their intersection is either a surface, a curve, or a set of points. Taking the preimage under $\phi_{\textbf{P}}$, $\mathcal{M}_P\cap\mathcal{M}_Q$ pulls back to a surface, curve, or point set in $\p3$, which is exactly the set of critical points of $\textbf{P}$ with respect to $\textbf{Q}$. As such, \emph{classifying all maximal critical configurations can be done by classifying all possible intersections between two multi-view varieties.} The configuration of points in $\p3$ can be found by taking the pullback of $\mathcal{M}_P\cap\mathcal{M}_Q$ under $\phi_{\textbf{P}}$.

\begin{proposition}
\label{cor:critical_configurations_lie_on_quadrics_and_cubics}
Let $(P_1,\ldots,P_n,X)$ be a maximal critical configuration. Then $X$ is the intersection of quadric surfaces $S^{ij}$ containing $p_i,p_j$, one such surface for each pair of cameras, and cubic surfaces $S^{ijk}$ containing $p_i,p_j,p_k$, one for each triple of cameras.
\end{proposition}
\begin{proof}
Since $(P_1,\ldots,P_n,X)$ is a maximal critical configuration, $X$ is the set of critical points of $P$ with respect to some other set of cameras $Q$ (from the conjugate configuration). Then $X$ is the pullback of  $\mathcal{M}_P\cap\mathcal{M}_Q$ under $\phi_{\textbf{P}}$. Since $\phi_{\textbf{P}}$ is the product of linear maps, it pulls a polynomial of multidegree $(d_1,\ldots,d_n)$ in $(\p2)^{n}$ back to a polynomial of degree $d_1+\cdots+d_n$ with multiplicity $d_i$ in the camera center $p_i$ in $\p3$. By \cref{thr:ideal_of_multiview_variety}, the ideal of the multi-view variety $\mathcal{M}_Q$ is generated only by bilinear and trilinear polynomials, one for each pair/triple of cameras. The pullback of $\mathcal{M}_P$ is all of $\p3$, the pullback of each bilinear polynomial in the ideal of $\mathcal{M}_Q$ is a quadric surface $S^{ij}$ and the pullback of each trilinear polynomial is a cubic surface. Thus $X$ is the intersection of such surfaces.
\end{proof}

\section{Critical configurations for one view}
\label{sec:one_view_case}
Reconstruction of a 3D scene from the image of one projective camera is generally considered impossible, so most papers start with the two-view case. Still, for the sake of completeness, we give a summary of the critical configurations for one camera. 

Let $P$ be a camera, and let $p$ be its camera center, we then have the joint camera map:
\begin{align*}
\phi_P\from\p3\dashrightarrow\p2.
\end{align*}
For any point $x\in\p2$, the fiber over $x$ is a line through $p$, so no point can be uniquely recovered. From this, one might assume that every configuration with one camera is critical. However, this is only the case if our configuration consists of sufficiently many points.

\begin{theorem}
A configuration of one point and one camera is never critical. A configuration of one camera and $n>1$ points is critical if and only if the camera center along with the $n$ points span a space of dimension less than $n$.
\end{theorem}

\begin{proof}
For the first part, note that up to a projective transformation, there exists only one configuration of one point and one camera. In other words, any configuration of one point and one camera can be taken to any other such configuration by simply changing coordinates. By \cref{def:equvalent_configurations} this makes them equivalent, which means that only one reconstruction exists. 

The same turns out to be the case if the configuration is such that the camera center along with the $n>1$ points span a space of dimension $n$. If they do then the points and camera center lie in general position. However, for four or fewer points in $\p3$ (fewer than three space points + one camera center), there exists only one configuration (up to action with $\PGL(4)$) where all points are in general position. Hence any configuration where the camera center along with the $n$ points span a space of dimension $n$ can not be critical. This can only happen for $n\leq3$ since $\p3$ itself is of dimension 3.

Now it only remains to show that a configuration is critical if the camera center along with the $n>1$ points span a space of dimension less than $n$. Indeed, if the points along with the camera center span a space of dimension less than $n$, the image points span a space of dimension $m<n-1$. Then there are at least two nonequivalent reconstructions; one reconstruction where the points span a space of dimension $m$ not containing the camera center and one where they span a space of dimension $m+1$ which contains the camera center. For instance, if one has three image points lying on a line, there are two reconstructions: one where the three space points span a plane containing the camera center, and one where the three space points lie on a line.
\end{proof}

\section{Critical configurations for two views}
\label{sec:preliminary_two_views}

\subsection{The fundamental matrix}
We start the study of the case of two cameras $P_1$ and $P_2$ by understanding their multi-view variety. The two cameras define the joint camera map:
\begin{align*}
\phi_P\from \p3\dashrightarrow\pxp.
\end{align*}

By \cref{lem:joint-camera_map_is_isomorphism}, the multi-view variety is an irreducible singular $3$-fold in $(\p2)^{n}$. In particular, for two views, it is a 3-fold in $\pxp$, so it is described by a single bilinear polynomial. This polynomial can be represented\footnote{For every bilinear polynomial $f($\textbf{x},\textbf{y}$)$ there exists a matrix $F$ such that $f($\textbf{x},\textbf{y}$)$=$\textnormal{\textbf{x}}^T$F\textbf{y}.} by a $3\times3$ matrix called the \emph{fundamental matrix} of $P_1$, $P_2$, which we denote $F_P$. See \cite[Section~9.2]{hartleyzisserman} for a geometric construction of the fundamental matrix. We use the construction in \cite{Tomas} where the fundamental matrix (or rather, its polynomial representation) is given as the determinant of a $6\times6$ matrix.

\begin{definition}
\label{def:epipole}
The \emph{epipoles} $e_{P_i}^{j}$ are the image points we get by mapping the $j$-th camera center to the $i$-th image
\begin{align*}
e_{P_i}^{j}=P_i(p_j).
\end{align*}
\end{definition}

\begin{proposition}[The fundamental matrix]{\cite[Sections 9.2 and 17.1]{hartleyzisserman}}
\label{lem:fundamental_form}
For two cameras $P_1,P_2$, the multi-view variety $\mathcal{M}_P\subset\pxp$ is the vanishing locus of a single, bilinear, rank 2\footnote{The rank of a bilinear polynomial is the rank of its matrix representation.} polynomial, given as follows:
\begin{align*}
\textnormal{\textbf{x}}^{T}F_P\textnormal{\textbf{y}}\coloneq\det\begin{bmatrix}
P_i&\textnormal{\textbf{x}}&\textnormal{\textbf{0}}\\
P_j&\textnormal{\textbf{0}}&\textnormal{\textbf{y}}
\end{bmatrix},
\end{align*}
where \textnormal{\textbf{x}} and \textnormal{\textbf{y}} are the variables in the first and second image respectively. The left and right nullspaces of $F_P$ are the two epipoles $e_{P_1}^{2}$ and $e_{P_2}^{1}$ respectively
\end{proposition}

\begin{proof}
Starting with the second claim, observe that setting $\textnormal{\textbf{x}}=e_{P_1}^{2}$ ensures that the left $6\times5$ block of the matrix above is of rank 4, meaning the $6\times6$ matrix has determinant 0. The same happens if one sets $\textnormal{\textbf{y}}=e_{P_2}^{1}$. Hence
\begin{align}
\label{eq:epipole_is_nullspace}
(e_{P_1}^{2})^{T}F_P=F_Pe_{P_2}^{1}=0.
\end{align}

For the first claim, it follows from \cref{lem:joint-camera_map_is_isomorphism} that the multi-view variety for two cameras is an irreducible 3-fold. Hence it is generated by a single polynomial. Let $(\textbf{x},\textbf{y})$ be a generic point in the multi-view variety $\mathcal{M}_P$. Then there exists a point $\textbf{X}\in\p3$ such that $P_1(\textbf{X})=\lambda_1\textbf{x}$ and $P_2(\textbf{X})=\lambda_2\textbf{y}$, hence
\begin{align*}
\begin{bmatrix}
P_i&\textbf{x}&\textbf{0}\\
P_j&\textbf{0}&\textbf{y}
\end{bmatrix}\begin{bmatrix}
\textbf{X}\\
-\lambda_1\\
-\lambda_2
\end{bmatrix}=0.
\end{align*}
Since the matrix on the left has a non-zero nullspace, the determinant $\textnormal{\textbf{x}}^{T}F_P\textnormal{\textbf{y}}$ has to vanish on $(\textbf{x},\textbf{y})$. Now we need only show that the determinant is irreducible to prove that $\textnormal{\textbf{x}}^{T}F_P\textnormal{\textbf{y}}$ generates the multi-view ideal. Irreducibility follows from the fact that $F_P$ is of rank 2 (since its nullspace is just a point), whereas a reducible bilinear polynomial has a matrix representation that is always of rank 1 since in this case, the matrix is just the product of two vectors, each representing one of the two factors.
\end{proof}

\begin{remark}
Since the entries in the camera matrix are real, so are the entries in the fundamental matrix.
\end{remark}

For each pair of cameras, we get a fundamental matrix, which is a $3\times3$ matrix of rank 2. The following result states that the converse is also true, i.e. that any $3\times3$ matrix of rank 2 is the fundamental matrix for some pair of cameras. 

\begin{theorem}
\label{thr:fundforms_and_camera_pairs_are_1:1}
There is a $1:1$ correspondence between $3\times3$ matrices of rank 2 (defined up to scale), and pairs of cameras $P_1,P_2$ (up to action by $\PGL(4)$).
\end{theorem}

\begin{proof}
By \cref{lem:fundamental_form}, the fundamental matrix of two cameras is a $3\times3$ matrix of rank 2. The converse follows from Theorem 9.13. in \cite{hartleyzisserman}.
\end{proof}

With these results, we can move on to classifying all the critical configurations for two views. We start with a special type of critical configuration:

\subsection{Trivial critical configurations}

\begin{definition}
\label{def:non_trivial_configuration}
A configuration $(P_1,\ldots,P_n,X)$ is said to be a \emph{non-trivial critical configuration} if it has a conjugate configuration $(Q_1,\ldots,Q_n,Y)$ satisfying
\begin{align}
\label{eq:non_trivial_configuration}
F_P^{ij}\neq F_Q^{ij}, \quad \forall i,j\leq n.
\end{align}
Here $F_P^{ij}$ denotes the fundamental matrix of $P_i,P_j$ and likewise for $F_Q^{ij}$.
\end{definition}

Critical configurations not satisfying this property exist, they are called trivial. If a critical configuration is trivial, then by \cref{thr:fundforms_and_camera_pairs_are_1:1} this means that, after a change of coordinates, $Q_i=P_i$ and $Q_j=P_j$. Since the cameras are the same, \cref{lem:joint-camera_map_is_isomorphism} tells us that the sets $X$ and $Y$ are equal, with the possible exception of any points lying on the line spanned by the two camera centers. It is a well-known fact that no number of cameras can differentiate between points lying on a line containing all the camera centers, hence the name \enquote{trivial}.

A classification of the trivial critical configurations for any number of views can be found in \cite{HK} Section 4, or in \cite{hartleyzisserman} Chapter 22. The focus of this paper is the non-trivial critical configurations, so for the remainder of the paper, when we talk about critical configurations, it is to be understood that they satisfy \cref{eq:non_trivial_configuration}. This, in turn, ensures that the second item of \cref{def:equvalent_configurations} is never satisfied, so there is no need to prove inequivalence between each configuration and its conjugates.

\subsection{Critical configurations for two views}
\label{subsec:crit_for_two_views}
Let us consider a non-trivial critical configuration $(P_1,P_2,X)$. Since it is critical, there exists a conjugate configuration $(Q_1,Q_2,Y)$ giving the same images in $\pxp$. So the points $X$ lie on the pullback $\phi_P^{-1}(\mathcal{M}_Q)$, which is a quadric surface $S_P$. Similarly, $Y$ lies on a quadric surface $S_Q$. Since the multi-view variety $\mathcal{M}_Q$ is described by the fundamental matrix $F_Q$, we can compute the matrix of $S_P$ (and $S_Q$):
\begin{align}
\begin{aligned}
\label{eq:SP_and_SQ}
S_P&=P_1^{T}F_QP_2,\\
S_Q&=Q_1^{T}F_PQ_2.
\end{aligned}
\end{align}

The two quadrics have the following properties:
\begin{lemma}[{Lemma 5.10 in \cite{HK}}]
\label{lem:properties_of_the_quadrics}
\begin{enumerate}
\item The quadric $S_P$ contains the camera centers $p_1$, $p_2$.
\item The quadric $S_P$ is ruled (contains real lines).
\end{enumerate}
\end{lemma}

\begin{proof}
A point $x\in\p3$ lies on $S_P$ if and only if $x^{T}S_Px=0$. Since $P_i(p_i)=0$, the camera centers lie on $S_P$. For the second item, recall that $F_Qe_{P_2}^{1}=0$, so the line $g_{P_2}^{1}=\overline{P_2^{-1}(e_{P_2}^{1})}$ lies on $S_P$.
\end{proof}

Up to action by $\PGL(4)$, there are only four quadrics containing lines. These are illustrated in \cref{fig:ruled_quadrics}. 

\begin{figure*}[]
\begin{center}
\includegraphics[width = 0.90\textwidth]{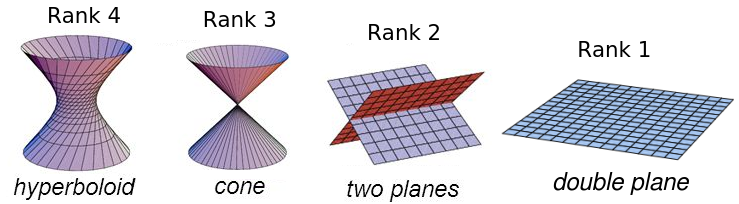}
\end{center}
\caption{All types of real ruled quadrics.}
\label{fig:ruled_quadrics}
\end{figure*}

The discussion so far can be summarized as follows:
\begin{theorem}[{Lemma 5.10 in \cite{HK}}]
\label{thr:critical_configurations_lie_on_quadrics}
Let $(P_1,P_2,X)$ be a non-trivial critical configuration. Then there exists a ruled quadric $S_P$ passing through the camera centers $p_1$, $p_2$ containing $X$.
\end{theorem}

\cref{thr:critical_configurations_lie_on_quadrics} tells us that all non-trivial critical configurations have their points and camera centers lying on a ruled quadric. The converse, however, is not always true. For instance, we will soon show that cameras and camera centers all lying on a cone is not a critical configuration if the cone contains the line spanned by the two camera centers (see \cref{fig:critical_quadrics}). Let us instead give a partial converse:

\begin{lemma}
\label{lem:critical_if_pullback_of_rank2_form}
Let $(P_1,P_2,X)$ be a configuration of cameras and points such that $X$ lies on a ruled quadric $S_P$ that passes through both the camera centers. Then for each real $3\times3$ (defined up to scale) matrix $F_Q$ of rank 2 such that:
\begin{align*}
S_P=P_1^{T}F_QP_2,
\end{align*}
there exists a conjugate configuration to $(P_1,P_2,X)$.
\end{lemma}

\begin{proof}
Assume such a matrix $F_Q$ exists. By \cref{thr:fundforms_and_camera_pairs_are_1:1} there exists a pair of cameras $(Q_1,Q_2)$ such that $F_Q$ is their fundamental matrix. Since $X$ lies on $S_P$, we have 
\begin{align*}
\phi_{P}(X)\subset\mathcal{M}_Q.
\end{align*}
Then for every point $x\in X$ we can find a point $y$ such that
\begin{align*}
\phi_P(x)=\phi_Q(y).
\end{align*}
Let $Y$ be the set of these points $y$. Then $(Q_1,Q_2,Y)$ is conjugate to $(P_1,P_2,X)$. This can be repeated for each bilinear polynomial of rank 2, giving unique conjugate configurations.
\end{proof}

The problem of determining which configurations are critical is now reduced to finding out which quadrics are the pullbacks of real bilinear polynomials of rank 2.

Let $\p{3\times3}$ denote the space of $3\times3$ matrices defined up to scale. We have that $\p{3\times3}$ is isomorphic to $\p8$. The fundamental matrix $F_P$ of the pair $(P_1,P_2)$ is an element in $\p{3\times3}$. 

\begin{lemma}
\label{lem:correspndence_between_lines_and_quadrics}
There is a $1:1$ correspondence between the set of real quadrics in $\p3$ passing through $p_1,p_2$, and the set of real lines in $\p{3\times3}$ passing through $F_P$.
\end{lemma}

\begin{proof}
Let $L\subset \p{3\times3}$ be a line through $F_P$, and let $F_0\neq F_P$ be some point on $L$. Every point $F\neq F_P$ on $L$ can be written as $F_0+\lambda F_P$ for some $\lambda\in\mathbb{R}$. We then have:
\begin{align*}
P_1^{T}FP_2&=P_1^{T}F_0P_2+\lambda P_1^{T}F_PP_2^{T},\\
&=P_1^{T}F_0P_2+\lambda\cdot 0,\\
&=P_1^{T}F_0P_2.
\end{align*} 
Hence, the matrix $F$ pulls back to the same quadric for all points $F\neq F_P$ on $L$.

For the $1:1$, note that a quadric $S$ passing through $p_1$ and $p_2$ has seven degrees of freedom, the same as a line in $\p{3\times3}$ through a fixed point.
\end{proof}

Using this $1:1$ correspondence and \cref{lem:critical_if_pullback_of_rank2_form}, the problem has been reduced to determining which quadrics correspond to lines in $\p{3\times3}$ containing at least one viable polynomial of rank 2. Let $\p{3\times3}_2\subset\p{3\times3}$ be the set of $3\times3$ matrices with rank at most 2. $\p{3\times3}_2$ is a hypersurface of degree 3, so a generic line $L$ contains two polynomials of rank 2 in addition to $F_P$. There are also other possibilities, listed in the tables below\footnote{The full computations behind these tables can be found in \cite{twoViews}}.

We start with the cases where the line $L$ corresponding to $S_P$ has a finite number of intersections with the rank 2 locus $\p{3\times3}_2$. The cases where $L$ contains at least one real rank 2 matrix $F$ different from $F_P$ are covered in \cref{tab:1}. The ones where it does not, are covered in \cref{tab:2}.

\begin{table*}[htbp!]
\begin{tabular}{@{}p{0.45\textwidth}p{0.50\textwidth}@{}}
\hline
\textbf{Intersection points} & \textbf{$S_P$} \\ \hline
All three intersection points are distinct real points & A smooth ruled quadric, camera centers not on a line \\ \hline
Two intersection points $F_P$ and $F_Q$, $L\cap\p{3\times3}_2$ has multiplicity 2 at $F_P$ & A smooth ruled quadric, camera centers lie on a line \\ \hline
Two intersection points $F_P$ and $F_Q$, $L\cap\p{3\times3}_2$ has multiplicity 2 at $F_Q$ & A cone, two camera centers not on a line, neither camera on a vertex \\ \hline
\end{tabular}
\caption{The cases where $L$ intersects the rank 2 locus finitely many times.}
\label{tab:1}
\end{table*}
\begin{remark}
In the published version of \cite{twoViews}, there is an error in \cref{tab:1} (not repeated here). In the published version, the two bottom columns are reversed.
\end{remark}

\begin{table*}[htb]
\begin{tabular}{@{}p{0.45\textwidth}p{0.50\textwidth}@{}}
\hline
\textbf{Intersection points} & \textbf{$S_P$} \\ \hline
The two other intersection points are complex conjugates & A smooth non-ruled quadric \\ \hline
The two other intersection points are of rank 1 & Union of two planes, camera centers in different planes \\ \hline
The two intersection points are equal to $F_P$ & A cone, both camera centers on a line, neither camera center at the vertex \\ \hline
\end{tabular}
\caption{The cases where there are no viable matrices.}
\label{tab:2}
\end{table*}
In the case where $L$ is contained in $\p{3\times3}_2$, all the matrices on $L$ are of rank 2 (with the possible exception of at most 2 that can be of rank 1). As such, rather than looking at where the intersections are, we look at the epipoles of the matrices in $L$. See \cref{tab:3}.

\begin{table*}[htb]
\begin{tabular}{@{}p{0.45\textwidth}p{0.50\textwidth}@{}}
\hline
\textbf{Epipoles} & \textbf{$S_P$} \\ \hline
All matrices have different epipoles & Two planes, camera centers lying in same plane \\ \hline
All matrices share the same right epipole, the left epipoles trace a line\footnotemark & Two planes, one camera center on the intersection of the planes \\ \hline
All matrices share the same right epipole, the left epipoles trace a conic\footnotemark[\value{footnote}] & Cone, one camera center at the vertex \\ \hline
All matrices share the same right and left epipole & Two (possibly complex) planes, both camera centers lying on the intersection of the planes \\ \hline
All matrices share the same right and left epipole AND the two rank 1 matrices on $L$ coincide & Double plane (as a set it is equal to a plane, but every point has multiplicity 2) \\ \hline
\end{tabular}
\caption{The cases where $L$ lies in the rank 2 locus.}
\label{tab:3}
\end{table*}
\footnotetext{The statement also holds if we swap "left" and "right".}

With this, we know exactly which quadrics form critical configurations. These are summarized in \cref{tab:configurations_and_their_conjugates}. To distinguish these from other ruled quadrics we give the following definition:
\begin{definition}
Let $P_1,P_2$ be two cameras. A quadric $S$ is said to be a \emph{critical quadric} (with respect to $P_1,P_2$) if $S$ and the camera centers $p_1,p_2$ form one of the configurations in \cref{tab:configurations_and_their_conjugates} (illustrated in \cref{fig:critical_quadrics}). 
\end{definition}

\begin{theorem}
\label{thr:critical_conf_for_two_views}
A configuration $(P_1,P_2,X)$ is non-trivially critical if and only if the two camera centers and all the points $X$ lie on a critical quadric.
\end{theorem}

\begin{table*}[htbp]
\begin{tabular}{@{}p{0.50\textwidth}p{0.30\textwidth}p{0.125\textwidth}@{}}
\toprule
\textbf{Quadric $S_P$} & \textbf{Conjugate quadric} & \textbf{Conjugates}\\
\midrule
Smooth ruled quadric, camera centers not on a line & Same & 2 \\ \midrule
Smooth ruled quadric, camera centers on a line & Cone, camera centers not on a line & 1\\ \midrule
Cone, camera centers not on a line & Smooth ruled quadric, camera centers on a line & 1\\ \midrule
Cone, one camera center at vertex, other one not & Same & $\infty$ \\ \midrule
Two planes, camera centers in the same plane & Same &$\infty$ \\ \midrule
Two planes, one camera center on the singular line & Same &$\infty$ \\ \midrule
Two planes, camera centers on the singular line\footnotemark & Same &$\infty$ \\ 
\midrule
A double plane, camera centers in the plane & Same &$\infty$ \\
\bottomrule
\end{tabular}
\caption{A list of all possible non-trivial critical configurations for two views, and their conjugates, as well as the number of conjugates for each configuration.}
\label{tab:configurations_and_their_conjugates}
\end{table*}
\footnotetext{This includes the case where the two planes are complex conjugates. In this case, the only real part of the quadric is the line spanned by the camera centers.}
\begin{figure*}[]
\begin{center}
\includegraphics[width = 0.9\textwidth]{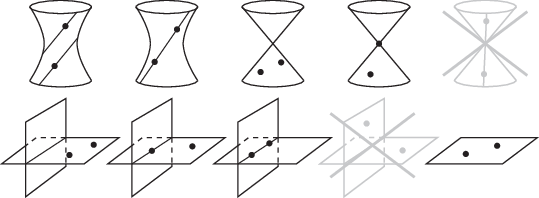}
\end{center}
\caption{Illustration of all configurations of two points on a ruled quadric. Apart from the two marked with crosses, all these configurations are critical.}\label{fig:critical_quadrics}
\end{figure*}

Recall that by \cref{def:equvalent_configurations}, we required two conjugate configurations to not be projectively equivalent. Yet by \cref{thr:critical_conf_for_two_views}, most critical configurations have conjugates that are of the same type. Now, while there is indeed some $H\in\PGL(4)$ taking any smooth quadric $S_P$ to any other smooth quadric $S_Q$, we will soon show that the map taking a point in $S_P$ to its conjugate on $S_Q$ certainly does not lie in $\PGL(4)$. In the next subsections, we describe the map taking a point to its conjugate, to show that it is not a projective transformation.

\subsection{Epipolar lines}
Before we can describe the map taking a point to its conjugate, we need to point out a special pair of lines on $S_P$. Given two pairs of cameras $(P_1,P_2)$ and $(Q_1,Q_2)$, let $S_P$ be the quadric $P_1^{T}F_QP_2$ and define
\begin{align}
\label{eq:g1}
\begin{aligned}
g_{P_1}^{2}=\overline{P_1^{-1}(e_{Q_1}^{2})}=\overline{P_1^{-1}(Q_1(q_2))},\\
g_{P_2}^{1}=\overline{P_2^{-1}(e_{Q_2}^{1})}=\overline{P_2^{-1}(Q_2(q_1))}.
\end{aligned}
\end{align}
The lines $g_{P_1}^{2}$ and $g_{P_2}^{1}$ both lie on $S_P$. They are the pullback of the epipoles $e_{Q_i}^{j}$ from the other set of cameras, so we call them \emph{epipolar lines}.

Epipolar lines are key in understanding the relation between points on $S_P$ and its conjugate $S_Q$. They also play an important role in the study of critical configurations for more than 2 cameras, so let us give a brief analysis of these lines. 

\begin{lemma}[{Lemma 5.10, Definition 5.11 in \cite{HK}}]
\label{lem:permissible_lines}
\begin{enumerate}
\item The line $\gp{i}{j}$ lies on $S_P$ and passes through $p_i$.
\item Any point lying on both $\gp{1}{2}$ and $\gp{2}{1}$ is a singular point on $S_P$.
\item Any point in the singular locus of $S_P$ that lies on one of the lines also lies on the other.
\item If $S_P$ is the union of two planes, $\gp{1}{2}$ and $\gp{2}{1}$ lie in the same plane.
\end{enumerate}
\end{lemma}
The first two properties are taken from Lemma 5.10 in \cite{HK}. The last two are neither stated nor proven in \cite{HK}. Nevertheless, the authors seem to have been aware of all four properties.
\begin{proof}
\begin{enumerate}
\item The line $\gp{i}{j}$ is defined as the preimage of a point in the $i$-th image, every such preimage is a line through $p_i$. Next, let $x\in\gp{2}{1}$, then
\begin{align*}
x^{T}S_Px&=x^{T}P_1^{T}F_QP_2x,\\
&=x^{T}P_1^{T}F_Qe_{Q_2}^{1},\\
&=0,
\end{align*} 
so the line lies on $S_P$.
\item A point $x$ is a singular point on $S_P$ if it satisfies $(S_P+S_P^{T})x=0$, let $x\in \gp{1}{2}\cap\gp{2}{1}$, then
\begin{align*}
(S_P+S_P^{T})x&=P_1^{T}F_QP_2x+x^{T}P_1^{T}F_QP_2,\\
&=P_1^{T}F_Qe_{Q_2}^{1}+(e_{Q_1}^{2})^{T}F_QP_2,\\
&=0+0.
\end{align*}
\item If $x\in\gp{1}{2}$ lies in the singular locus of $S_P$, then we have
\begin{align*}
(S_P+S_P^{T})x&=(P_1^{T}F_QP_2+P_2^{T}F_Q^{T}P_1)x,\\
&=P_1^{T}F_QP_2x+(e_{Q_1}^{2})^{T}F_QP_2,\\
&=P_1^{T}F_QP_2x.
\end{align*}
However, since $x$ lies in the singular locus of $S_P$, this expression is equal to zero. Since all camera matrices are of full rank, the only way we can get zero is if $x=p_2$ or if $P_2x=e_{Q_2}^{1}$. In either case, it follows that $x$ lies on $\gp{2}{1}$ as well.
\item[4.] Assume there exist cameras $P_1,P_2,Q_1,Q_2$ such that $S_P$ is the union of two planes and the epipolar lines $g_{P_i}^{j}$ lie in different planes. When the quadric $S_P$ consists of two planes, one of the planes, which we denote by $\Pi$, will (by \cref{thr:critical_conf_for_two_views}) contain both camera centers. As such, the only way that the epipolar lines can lie in different planes is if one of the camera centers, say $p_2$, lies on the intersection of the two planes and the other does not (if both lie on the intersection, then by 3., the epipolar lines are both equal to the intersection of the two planes). 

Recall that the quadric $S_P$ and its conjugate $S_Q$ (also two planes) are both pullbacks of the surface $\mathcal{M}_P\cap\mathcal{M}_Q\subset\pxp$. The map $\phi_P$ takes the plane $\Pi$ to the product of two lines in $\pxp$. The line in the first image passes through the epipole $e_{Q_1}^{2}$, whereas the line in the second image does not pass through the epipole $e_{Q_2}^{1}$ (this is because $\Pi$ contains one of the epipolar lines but not the other). The problem is now that neither of the planes on $S_Q$ can map to the product of these two lines since any such plane would have to be both 
\begin{enumerate}
\item a plane passing through $q_2$ but not through $q_1$ (because in the second image, the line does not pass through $e_{Q_2}^{1}$) and
\item a plane passing through both $q_1$ and $q_2$ (because the plane maps to a line in both images),
\end{enumerate}
which gives us a contradiction. It follows that there are no $P_1,P_2,Q_1,Q_2$ such that $S_P$ is the union of two planes and the epipolar lines $g_{P_i}^{j}$ lie in different planes. So whenever $S_P$ is the union of two planes, the epipolar lines lie in the same plane. \qedhere
\end{enumerate}
\end{proof}

\begin{definition}
\label{def:permissible_lines}
Let $S_P$ be a quadric surface and let $p_1,p_2$ be two distinct points on $S_P$. A pair of lines $g_{P_1}^{2},g_{P_2}^{1}$ is called \emph{permissible} if $g_{P_1}^{2},g_{P_2}^{1}$ satisfy the four conditions in \cref{lem:permissible_lines}.
\end{definition}

Hence epipolar lines are always permissible. The next result shows that for each pair of permissible lines on $S_P$, we can find a conjugate configuration making them the epipolar lines.

\begin{proposition}
\label{prop:conjugates_and_permissible_lines}
Let $P_1,P_2$ be two cameras, and let $S_P$ be a quadric passing through their camera centers. The configuration $(P_1,P_2,S_P)$ is critical if and only if $S_P$ contains a permissible pair of lines. Furthermore, if $p_1,p_2$ do not both lie in the singular locus of $S_P$, then for each pair of permissible lines on $S_P$, there is a unique fundamental matrix $F_Q$ such that $S_P=P_1^{T}F_QP_2$ and such that the two permissible lines are the epipolar lines on $S_P$.\end{proposition}
\begin{proof}
The first part can be proven by comparing the quadrics in \cref{tab:configurations_and_their_conjugates} to the set of quadrics containing the required lines, and noting that they are the same. We leave this to the reader. 

The second part is immediate for the three cases where there is a finite number of conjugates since the number of conjugates is equal to the number of pairs of permissible lines and each conjugate comes with its unique choice of lines. For the remaining quadrics, (cone and two planes), we observe from \vref{tab:1,tab:2,tab:3} that each quadric has a one-dimensional family of conjugates, each conjugate coming with a pair of epipolar lines. These epipolar lines are unique, as long as the cameras do not both lie in the singular locus. This one-dimensional family matches exactly the one-dimensional family of pairs of permissible lines on these quadrics.
\end{proof}

\subsection{Maps between quadrics}
Let us now have a closer look at the map $\psi$ taking a point to its conjugate. Given two pairs of cameras $P_1,P_2$ and $Q_1,Q_2$ with camera centers $p_1,p_2$ and $q_1,q_2$ respectively, let the quadrics $S_P$ and $S_Q$ and the epipolar lines $g$ be defined as before (\cref{eq:SP_and_SQ,eq:g1}). 

Let $\psi_P$ be the map taking a point on $S_P$ to its conjugate:
\begin{align*}
\psi_P \from S_P &\dashrightarrow S_Q,\\
x&\mapsto l_1\cap l_2,
\end{align*} 
where $l_i$ is the line $\overline{Q_i^{-1}(P_i(x))}$. This map is not well-defined for all points on $S_P$, since some points might have more than one conjugate (see 2. and 3. in \cref{lem:relations_between_quadrics}).

\begin{lemma}
\label{lem:relations_between_quadrics}
The map $\psi_P$ has the following properties:
\begin{enumerate}
\item \label{item1}The conjugate to any point on the epipolar line $g_{P_i}^{j}$, apart from the camera center $p_i$ itself, is the camera center $q_j$.
\item \label{item2}The conjugate to the camera center $p_j$ is the epipolar line $g_{Q_i}^{j}$
\item \label{item3}The conjugate to the intersection of the two epipolar lines $g_{P_1}^{2}$ and $g_{P_2}^{1}$ (if it exists) is the line spanned by the two camera centers $q_1,q_2$.
\item \label{item4} The map $\psi_P$ is a bijective map everywhere else.
\end{enumerate}
\end{lemma}
\begin{proof}
\begin{enumerate}
\item Let $x\in g_{P_i}^{j}$. Then $l_i$ is the line spanned by the two camera centers $q_1$ and $q_2$, while $l_j$ is a line through $q_j$. Thus the two lines intersect in $q_j$.
\item Same argument as in 1. but starting with the line $g_{Q_i}^{j}$ instead.
\item Let $x\in g_{P_1}^{2}\cap g_{P_1}^{2}$. Then $l_1$ and $l_2$ are both equal to the line spanned by the two camera centers $q_1$ and $q_2$, so any point on this line is conjugate to $x$.
\item If $x$ does not lie on either epipolar line, the two lines $l_1$ and $l_2$ are distinct, and meet in a point $y$ that is not a camera center. For any such point $y$, the lines $l_1$ and $l_2$ are the only pair of such lines meeting in this point (since $l_i$ passes through $q_i$). Hence $x$ only has $y$ as its conjugate and vice versa. \qedhere
\end{enumerate}
\end{proof}

\begin{remark}
\label{rem:method_for_computing}
When the two pairs of cameras $P_1,P_2$ and $Q_1,Q_2$ are known, one can explicitly compute the map $\psi_P$ as a rational function. For some representatives of a general point $X\in S_P$ and its conjugate $\psi_P(X)\in S_Q$, we have
\begin{align*}
&\overbrace{\begin{pmatrix}
Q_1&P_1X&0\\
Q_2&0&P_2X
\end{pmatrix}}^{A}\cdot\begin{pmatrix}
\psi_P(X)\\1\\1
\end{pmatrix}\\
=&\begin{pmatrix}
P_1&Q_1\psi_P(X)&0\\
P_2&0&Q_2\psi_P(X)
\end{pmatrix}\cdot\begin{pmatrix}
X\\1\\1
\end{pmatrix}=0.
\end{align*}
Using the adjugate of $A$, we can get an expression for $\psi_P$ as a rational function in the homogeneous coordinates of $X$.
\end{remark}
With these results in place, we can give two formulas for how $\psi_P$ acts on the curves on $S_P$:

\begin{definition}
\label{def:type_of_curve}
Let $S_P$ be a smooth quadric or cone and let $C_P$ be a curve on $S_P$. We say that $C_P$ is of type $(a,b,c_1,c_2)$, where $c_i$ is the multiplicity of $C_P$ in the camera center $p_i$ and where $a$ and $b$ is:
\begin{itemize}
\item If $S_P$ is smooth, $a$ is the number of times $C_P$ intersects a generic line in the same family as the epipolar lines $g_{P_i}^{j}$, and $b$ is the number of times it intersects the lines in the other family, (meaning $(a,b)$ is the bidegree of $C_P$).
\item If $S_P$ is a cone, $a$ is the number of times $C_P$ intersects each line outside of the vertex, and $b$ the number of times it intersects each line.
\end{itemize}
\end{definition}

\begin{proposition}[{\cite[Lemma 8.32]{HK}},{\cite[Proposition 5.7]{twoViews}}]
\label{prop:bidegree_on_P_->_bidegree_on_Q}
Let $S_P$ be a smooth quadric or a cone, and let $C_P\subset S_P$ be a curve of type $(a,b,c_1,c_2)$, such that $C_P$ does not contain either of the epipolar lines. Then the conjugate curve $C_Q\subset S_Q$ is of type $(a,a+b-c_1-c_2,a-c_2,a-c_1)$.
\end{proposition}

The action of $\psi_P$ on the lines on a quadric/cone is illustrated in \cref{fig:conjugates1,fig:conjugates2}.

\begin{proof}
For a full proof using intersection theory, see the proof in \cite{twoViews}. Here we give a shorter and more intuitive version of the proof. Let the conjugate curve $C_Q$ be of type $(a',b',c_1',c_2')$. The curve $C_P$ intersects $g_{P_1}^{2}$ $a$ times, with $c_1$ of these intersection point being in the camera center $p_1$. By \cref{lem:relations_between_quadrics} \cref{item1}, every point on $g_{P_i}^{j}$ (apart from the camera center itself) is mapped to $q_2$. It follows that $c_2'=a-c_1$, and likewise for $c_1'$. Since this argument goes both ways, we get that $c_2=a'-c_1'=a'-(a-c_2)$. It follows that $a'=a$.

To find an expression for $b'$, we consider how $\psi_P$ treats various lines on $S_P$. We consider lines on $S_P$ different from the epipolar lines. So by \cref{lem:relations_between_quadrics} \cref{item4}, their conjugate is a curve on $S_Q$. Consider first a line through $p_1$ different from $g_{P_1}^{2}$. This maps to a point in the first image, and a line in the second. Hence its conjugate is a line on $S_Q$ passing through $q_1$ (type (1,0,1,0)). Next, we consider a line that intersects the epipolar lines but does not pass through any camera centers. In each image, this line appears as a line passing through $e_{Q_i}^{j}$. Back-projecting these lines using the cameras $Q$, we get two planes, both passing through the two camera centers $q_1,q_2$. The intersection of these two planes with $S_Q$ is a conic curve passing through both camera centers (type (1,1,1,1)). Lastly, if we take a line on $S_P$ in the same family as the epipolar lines, we get a line in each image that does not pass through any epipoles. Back-projecting these lines using the cameras $Q$, we get two distinct planes, whose intersection is a line not intersecting the epipolar lines (type (0,1,0,0)). To summarize, the map $\psi_P$ maps:
\begin{align*}
(1,0,1,0)&\mapsto(1,0,1,0), & &(1,0,0,1)\mapsto(1,0,0,1),\\
(1,0,0,0)&\mapsto(1,1,1,1), & &(0,1,0,0)\mapsto(0,1,0,0).
\end{align*}
Due to $\psi_P$ being a rational map, we have $b'=\lambda_aa+\lambda_bb+\lambda_{c_1}c_1+\lambda_{c_2}c_2$. Solving for $\lambda$ using the equations above, we get $b'=a+b-c_1-c_2$.
\end{proof}

\begin{definition}
Let $S_P$ be two planes with two camera centers not both lying on the intersection and let $C_P$ be a curve on $S_P$. We say that $C_P$ is of type $(a,b,c_0,c_1,c_2)$, where $a$ is the degree of the curve in the plane with the epipolar lines, $b$ the degree of the curve in the other plane, $c_0$ is the multiplicity of $C_P$ in the intersection of the epipolar lines $g_{P_i}^{j}$ and $c_1,c_2$ is the multiplicity in the camera centers $p_1,p_2$ respectively. 
\end{definition}

\begin{proposition}[{\cite[Proposition 5.9]{twoViews}}]
\label{prop:bidegree_on_P_->_bidegree_on_Q_planes}
Let $S_P$ be two planes with two camera centers not both lying on the intersection and let $C_P\subset S_P$ be a curve of type $(a,b,c_0,c_1,c_2)$ such that $C_P$ does not contain either of the epipolar lines or the line spanned by the camera centers. Then the conjugate curve $C_Q\subset S_Q$ is of type $(2a-c_0-c_1-c_2,b,a-c_1-c_2,a-c_0-c_2,a-c_0-c_1)$.
\end{proposition}

\begin{proof}
See \cite{twoViews} for a fully detailed proof. We give a shorthand version here. Let the conjugate curve $C_Q$ be of type $(a',b',c_0',c_1',c_2')$, and denote the intersection of the two epipolar lines on $S_P$ and the two on $S_Q$ by $x_P$ and $x_Q$ respectively. The curve $C_P$ intersects $g_{P_1}^{2}$ $a$ times, with $c_1$ of these intersection point being in the camera center $p_1$ and $c_0$ of these being in $x_P$. Every point on $g_{P_i}^{j}$ (apart from $p_1$ and $x_P$) is mapped to $q_2$. It follows that $c_2'=a-c_0-c_1$, and likewise for $c_1'$. Moreover, $C_P$ intersects the line $\overline{p_1p_2}$ $a$ times, with $c_1+c_2$ of these intersection points being camera centers. By \cref{lem:relations_between_quadrics} \cref{item3} the conjugate to all these intersection points is the point $x_Q$ on $S_Q$, so $c_0'=a-c_1-c_2$. The formula goes both ways, so we get that 
\begin{align*}
c_0&=a'-c_1'-c_2',\\
&=a'-(a-c_0-c_2)-(a-c_0-c_1),\\
&=a'-2a+2c_0+c_1+c_2.
\end{align*}
It follows that $a'=2a-c_0-c_1-c_2$. Lastly, since the map $\psi_P$ acts as an isomorphism on the plane not containing the camera centers, we have $b'=b$.
\end{proof}

\begin{figure}[]
\begin{center}
\includegraphics[width = \linewidth]{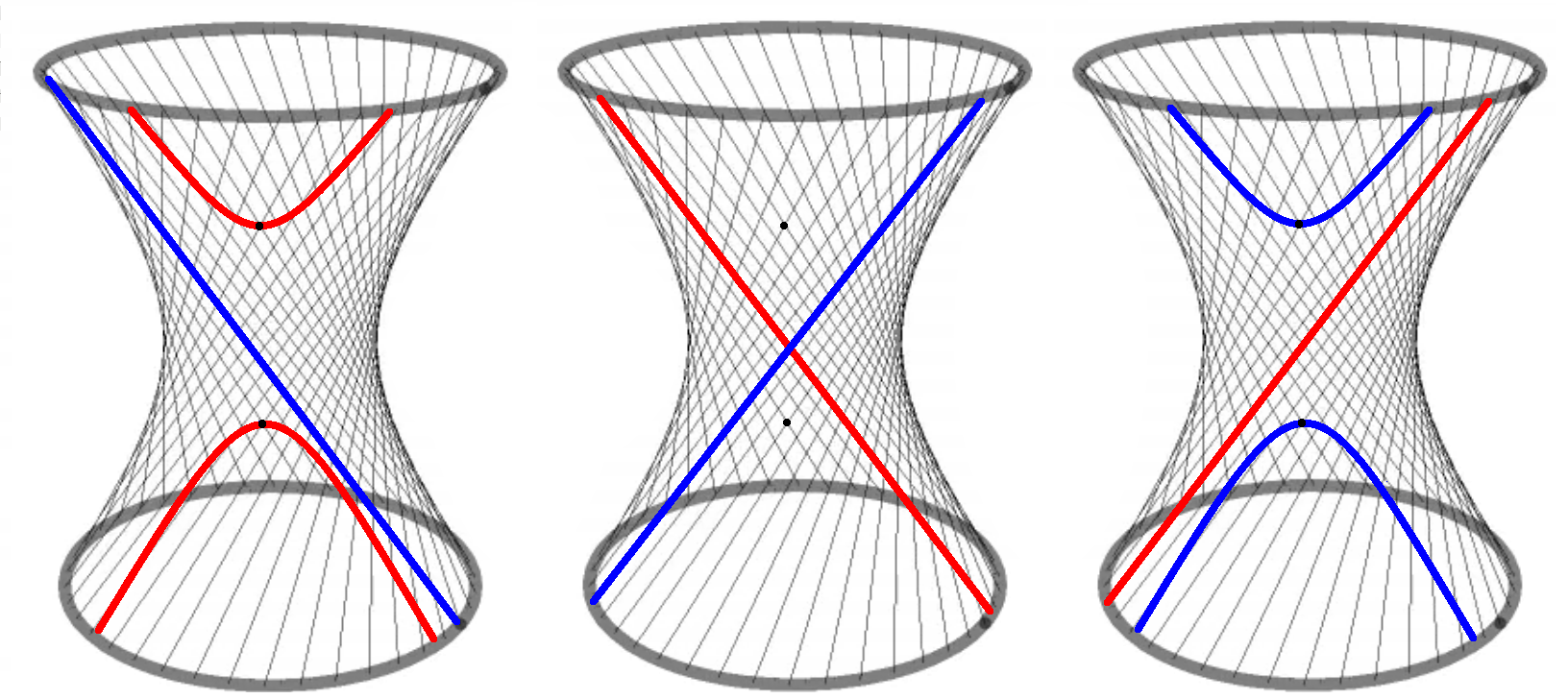}
\end{center}
\caption{The smooth quadric has two conjugates (original in the middle). The left is the one we get if we take the lines in the blue family to be the epipolar lines, whereas the one on the right is the one we get if we choose the red. In both cases the lines in the family with the epipolar lines are preserved, whereas the lines in the other family are mapped to conic curves.}
\label{fig:conjugates1}
\end{figure}

\begin{figure}[]
\begin{center}
\includegraphics[width = 0.7\linewidth]{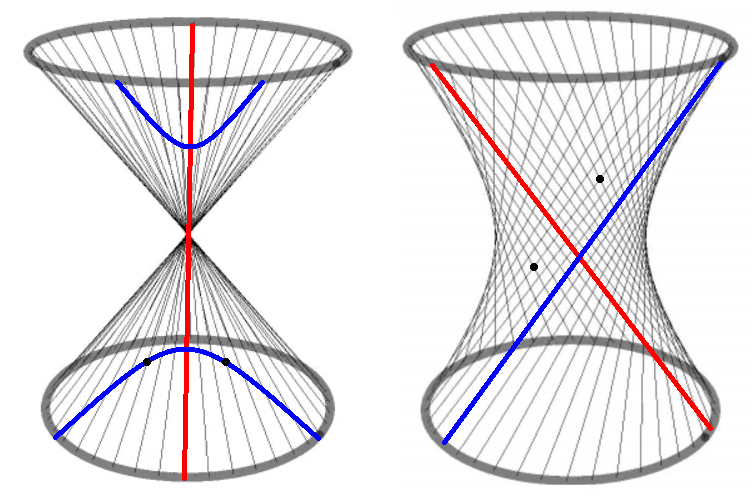}
\end{center}
\caption{The cone and smooth quadric with two cameras on a line are conjugate to one another. The line spanned by the camera centers (right) is mapped to the vertex on the cone. The other lines in this family map to conics. Lines in the other family are preserved.}
\label{fig:conjugates2}
\end{figure}

\section{Preliminaries for the three-view case}
\label{sec:preliminary_three_views}
We start with the observation that any critical configuration for three views is critical for each pair of views.
\begin{theorem}
\label{thr:critical_only_if_critical_for_fewer_views}
Let $(P_1,\ldots,P_n,X)$ be a non-trivial critical configuration, and let $\mathcal{P}\subset\Set{P_1,\ldots,P_n}$ be a subset of at least two cameras. Then $(\mathcal{P},X)$ is a critical configuration.

\end{theorem}
\begin{proof}
Since $(P_1,\ldots,P_n,X)$ is a critical configuration, there exists a conjugate configuration $(Q_1,\ldots,Q_n,Y)$. Each camera $P_i$ has a conjugate camera $Q_i$. Let $\mathcal{Q}\subset\Set{Q_1,\ldots,Q_n}$ be the set of cameras conjugate to $\mathcal{P}$. Then $(\mathcal{P},X)$ and $(\mathcal{Q},Y)$ produce the same images. By the non-triviality, these two configurations are inequivalent and hence conjugate to one another.
\end{proof}

\begin{definition}
Let $(P_1,\ldots,P_n,X)$ be a configuration of cameras and points, and let $\Set{S_P^{ij}}$ be a set of $n\choose2$ quadrics. The configuration $(P_1,\ldots,P_n,X)$ is said to \emph{lie on the intersection of the quadrics $\Set{S_P^{ij}}$} if each quadric $S_P^{ij}$ contains the camera centers $p_i,p_j$, and $X$.
\end{definition}
 
\begin{corollary}
\label{cor:critical_lies_on_intersection_of_quadrics}
For $n\geq2$ views, the critical configurations always lie on the intersection of $n\choose2$ critical quadrics. 
\end{corollary}
\begin{proof}
This follows from \cref{thr:critical_only_if_critical_for_fewer_views} along with the fact that for two views, the critical configurations lie on critical quadrics. It also follows directly from \cref{cor:critical_configurations_lie_on_quadrics_and_cubics}.
\end{proof}

In particular, for the three-view case, the critical configurations lie on the intersection of three critical quadrics. The task of classifying the critical configurations for three views, however, is not as simple as just classifying the possible intersections of three quadrics. Even if a configuration is critical for each pair of views, it need not be critical for all three views. There are two issues, namely the \emph{compatibility of the quadrics}, and the \emph{ambiguous points}. Before we tackle these issues, let us state a result on the generators of the multi-view ideal for three views.

\begin{lemma}[{\cite[Proposition 5.6]{threeViews}}]
\label{lem:three_bilinear_and_10_trilinear}
For the three non-collinear cameras, the multi-view ideal is generated by three bilinear and one trilinear polynomial.\footnote{In the collinear case, it is generated by three bilinear and three trilinear polynomials, but this is harder to prove. See \cite[Proposition 5.6]{threeViews} for proof.}
\end{lemma}
\begin{proof}
A full proof, that also shows that one needs three trilinear polynomials in the collinear case, can be found in \cite{threeViews}. We give a shorthand version here.

The fact that it takes three bilinear polynomials (the fundamental matrices) is simply an application of \cref{thr:ideal_of_multiview_variety} to the case of three cameras. Next, we show that the ideal contains a ten-dimensional family of trilinear polynomials. There is a 27-dimensional family of trilinear polynomials on $\pxpxp$, which pulls back to a 17-dimensional family of cubics passing through the three camera centers (20 cubics, less three for the constraint of passing through the camera centers). Hence a 10-dimensional family of trilinear polynomials pull back to the zero polynomial and thus lie in the multi-view ideal.

Each of the bilinear polynomials generates three trilinear polynomials in the ideal, by multiplying it with one of the three variables from the last $\p2$, for instance multiplying $\textnormal{\textbf{x}}^{T}F_P^{12}\textnormal{\textbf{y}}$ with $z_1,z_2$ and $z_3$ will give three trilinear polynomials. As such, one only needs one final trilinear polynomial to generate the ideal in the general case, that is, in the case where the cameras are not collinear.
\end{proof}

\subsection{Compatible triples of quadrics}
\label{sec:compatible_triples_of_quadrics}
In the two-view case, any fundamental matrix ($3\times3$ matrix of rank 2) comes from some pair of cameras (\cref{thr:fundforms_and_camera_pairs_are_1:1}). In the three-view case, however, three arbitrary fundamental matrices will generally not come from some triple of cameras. This motivates the following definition:

\begin{definition}
Let $\Set{F_Q^{ij}\mid 1\leq i\leq j\leq n}$ be a set of $n\choose 2$ fundamental matrices. The matrices are said to be \emph{compatible} if there exist $n$ cameras $Q_1,...,Q_n$ such that $F_Q^{ij}$ is the fundamental matrix of $Q_i,Q_j$.
\end{definition}

\begin{definition}
Given $n$ cameras $P_1,...,P_n$ we say a set of quadrics $\Set{S_P^{ij}}$ is \emph{compatible} (with respect to the cameras $P_i$), if there exists a compatible set of fundamental matrices $F_Q^{ij}$ such that
\begin{align*}
S_P^{ij}=P_i^{T}F_Q^{ij}P_j.
\end{align*}
\end{definition}

For a configuration $(P_1,P_2,P_3,X)$ to be critical, it is not enough that it lies on the intersection of three critical quadrics, the three quadrics also need to be compatible. Only then can we find three conjugate cameras $Q_1,Q_2,Q_3$ which in turn can be used to construct a conjugate configuration. A detailed study of compatible fundamental matrixes can be found in \cite{compatibility,hartleyzisserman}. Recall that the \emph{epipoles} $e_{Q_i}^{j}$ are the unique points satisfying
\begin{align}
\tag{\ref{eq:epipole_is_nullspace}}
(e_{P_1}^{2})^{T}F_P=F_Pe_{P_2}^{1}=0.
\end{align}

\begin{theorem}[{\cite[Proposition 5.8]{compatibility}}]
\label{thr:compatible_forms_collinear}
Let $F_Q^{12}$, $F_Q^{13}$, $F_Q^{23}$ be a triple of fundamental matrices, and let $e_{Q_i}^{j}$ be their epipoles. Then the triple is compatible and comes from a collinear triple of cameras if and only if
\begin{align}
\label{eq:collinear}
e_{Q_1}^{2}=e_{Q_1}^{3}, \quad e_{Q_2}^{1}=e_{Q_2}^{3}, \quad e_{Q_3}^{1}=e_{Q_3}^{2},
\end{align}
and
\begin{align}
\label{eq:compatible_collinear}
(F_Q^{12})^{T}[e_{Q_1}^{2}]_\times F_Q^{13}=F_Q^{23}.\footnotemark
\end{align}
\end{theorem} 
\footnotetext{Here $[a]_{\times}$ denotes the unique $3\times3$ matrix such that $[a]_{\times}b=a\times b$ for all $b$}

\begin{theorem}[{\cite[Section 15.4]{hartleyzisserman}}]
\label{thr:compatible_forms_noncollinear}
Let $F_Q^{12}$, $F_Q^{13}$, $F_Q^{23}$ be a triple of rank 2 fundamental matrices, and let $e_{Q_i}^{j}$ be their epipoles. Then the triple is compatible and comes from a non-collinear triple of cameras if and only if
\begin{align}
\label{eq:non_collinear}
e_{Q_1}^{2}\neq e_{Q_1}^{3}, \quad e_{Q_2}^{1}\neq e_{Q_2}^{3}, \quad e_{Q_3}^{1}\neq e_{Q_3}^{2},
\end{align}
and
\begin{align}
\label{eq:compatible_non_collinear}
(e_{Q_1}^{3})^{T}F_Q^{12}e_{Q_2}^{3}=(e_{Q_1}^{2})^{T}F_Q^{13}e_{Q_3}^{2}=(e_{Q_2}^{1})^{T}F_Q^{23}e_{Q_3}^{1}=0.
\end{align}
\end{theorem} 

Since we are working with quadrics in $\p3$, we want to translate these conditions on the fundamental matrices, to conditions on the quadrics. The remainder of this subsection is used to give two results stating these conditions. We start with the collinear case, that of \cref{thr:compatible_forms_collinear} (clarification: this is NOT the case where the camera centers $p_i$ are collinear, but the case where the camera centers $q_i$ of the conjugate configuration are collinear).

\begin{proposition}[\textbf{Collinear case}]
\label{prop:compatible_collinear}
Let $P_1$, $P_2$, $P_3$ be a triple of cameras, and let $S_P^{12}$, $S_P^{13}$, $S_P^{23}$ be a triple of critical quadrics. Then the three quadrics come from a triple of fundamental matrices that satisfy the conditions in \cref{thr:compatible_forms_collinear}, if and only if there exist three lines $L_{1}$, $L_{2}$, $L_{3}$ satisfying:
\begin{enumerate}
\item $p_i\in L_i$.
\item $L_i$ and $L_j$ form a permissible pair on $S_P^{ij}$.
\item $L_i\subseteq S_P^{ij}\cap S_P^{ik}$.
\item Any point $x\in S_P^{12}\cap S_P^{13}$ not lying on $L_1$ also lies on $S_P^{23}$.
\end{enumerate}
\end{proposition}
\begin{proof}
($\Rightarrow$):
Let $S_P^{12}$, $S_P^{13}$, $S_P^{23}$ be a compatible (in the collinear sense) triple of critical quadrics. This means they come from a triple of fundamental matrices satisfying the conditions in \cref{thr:compatible_forms_collinear}. By the first part of \ref{prop:conjugates_and_permissible_lines}, each quadric contains a pair of permissible lines, the pullback of the epipoles. By \cref{eq:collinear} the two epipoles coincide in each image, so the two epipolar lines passing through $p_i$ have to coincide as well. This gives us three lines satisfying the first three conditions in \cref{prop:compatible_collinear}. 

Next, let $x$ be any point in $S_P^{12}\cap S_P^{13}$ not lying on $L_1$, and denote $x_i=P_i(x)$. Since $x$ lies on $S_P^{1i}$, $F_Q^{1i}x_i$ is the unique line in the first image spanned by $x_1$ and the epipole $e_1^{2}=e_1^{3}$. Now, since $F_Q^{12}x_2$ and $F_Q^{13}x_3$ define the same line, we have
\begin{align*}
0&=x_2^{T}(F_Q^{12})^{T}[e_{Q_1}^{2}]_\times F_Q^{13}x_3,\\
&=x_2^{T}F_Q^{23}x_3,\\
&=x^{T}P_2'^{T}F_Q^{23}P_3'x,\\
&=x^{T}S_P^{23}x.
\end{align*}
It follows that $x$ lies on $S_P^{23}$ as well

($\Leftarrow$):
Let $S_P^{ij}$ be three quadrics satisfying the conditions above. Since $L_i$ and $L_j$ form a permissible pair on $S_P^{ij}$, there exists (by \cref{prop:conjugates_and_permissible_lines}) fundamental matrices $F_Q^{ij}$ such that $S_P^{ij}=P_i^{T}F_Q^{ij}P_j$, and the lines $L_i,L_j$ are the epipolar lines on $S_P^{ij}$. Since the line $L_i$ is the preimage of both the epipole $e_{Q_i}^{j}$ and the epipole $e_{Q_i}^{k}$, the two epipoles are equal. Hence, the epipoles satisfy \cref{eq:collinear}. 

Let $S_P^{23'}$ denote the quadric defined by the matrix
\begin{align*}
P_2'^{T}F_Q^{12T}[e_{Q_1}^{2}]_\times F_Q^{13}P_1.
\end{align*}
For the next part, we want to show that the $S_P^{23'}=S_P^{23}$. If so, the three quadrics $S_P^{ij}$ come from the three fundamental matrices $F_Q^{12},F_Q^{13}$ and $F_Q^{12T}[e_{Q_1}^{2}]_\times F_Q^{13}$, since this triple of fundamental matrices is compatible, so is the triple of quadrics. By assumption, the set $(S_P^{12}\cap S_P^{13})-L_1$ lies on all three quadrics. If the quadrics do not contain a common surface, this set is a (possibly reducible) cubic curve $C$. The curve $C$ lies on both $S_P^{23}$ and $S_P^{23'}$. Moreover, since $F_Q^{12T}[e_{Q_1}^{2}]_\times F_Q^{13}$ has the same left and right nullspaces as $F_Q^{23}$, $S_P^{23'}$ must also contain the two lines $L_2,L_3$. The union of a cubic curve with two lines uniquely determines a quadric, so we have $S_P^{23}=S_P^{23'}$. A similar argument can be made if the quadrics contain a common surface, and in particular, if they are equal.
\end{proof}

Among other things, this shows that as long as the cameras are not in the position shown in \cref{fig:impossible_configurations}, three copies of the same quadric form a compatible triple. For the remainder of the section, we assume that compatible triples of quadrics come from triples of fundamental matrices that satisfy the non-collinear conditions (\cref{eq:compatible_non_collinear,eq:non_collinear}), rather than \cref{eq:compatible_collinear,eq:collinear}.

\begin{theorem}[\textbf{General case}]
\label{thr:compatible_quadrics}
Let $P_1,P_2,P_3$ be a triple of cameras, and let $S_P^{12}$, $S_P^{13}$, $S_P^{23}$ be a triple of critical quadrics. Then the three quadrics come from a triple of fundamental matrices that satisfy the conditions in \cref{thr:compatible_forms_noncollinear}, if and only if there exist six lines $g_{P_i}^{j}$ such that:
\begin{enumerate}
\item $g_{P_i}^{j}$ and $g_{P_j}^{i}$ form a permissible pair on $S_P^{ij}$ (see \cref{def:permissible_lines}).
\item The intersection of the plane spanned by $g_{P_i}^{j}$ and $g_{P_i}^{k}$ with the plane spanned by $g_{P_j}^{i}$ and $g_{P_j}^{k}$ lies on $S_P^{ij}$.
\end{enumerate}
\end{theorem}

\begin{proof}
($\Rightarrow$):
If the triple of quadrics is compatible, they come from a compatible triple of fundamental matrices. Such quadrics always come with two epipolar lines, which are always permissible. Hence, the first part is satisfied.

Among the six epipolar lines, there are 2 lines on each quadric and 2 through each camera center. By (\ref{eq:non_collinear}), for each camera center, the two lines passing through it are distinct. Hence the two lines through each camera center $p_i$ span a plane, which we denote by $\Pi_i$.

Let $y$ be any point lying in the intersection of $\Pi_{i}$ and $\Pi_j$. Recall that epipolar lines satisfy $P_i(g_{P_i}^{j})=e_{Q_i}^{j}$. We then have
\begin{align*}
P_iy=\alpha e_{Q_i}^{j}+\beta e_{Q_i}^{k},\\
P_jy=\gamma e_{Q_j}^{i}+\delta e_{Q_j}^{k}.
\end{align*}
for some $[\alpha:\beta],[\gamma:\delta]\in\p1$. Hence
\begin{align*}
y^{T}S_P^{ij}y &=y^{T}P_i^{T}F_Q^{ij}P_jy\\
&=(\alpha e_{Q_i}^{j}+\beta e_{Q_i}^{k})^{T}F_Q^{ij}(\gamma e_{Q_j}^{i}+\delta e_{Q_j}^{k})\\
&=(\beta e_{Q_i}^{k})^{T}F_Q^{ij}(\delta e_{Q_j}^{k})\\
&=0,
\end{align*}
(final equality follows from (\ref{eq:compatible_non_collinear})) which proves that $y$ lies on $S_P^{ij}$. Hence $S_P^{ij}$ contains the whole intersection of $\Pi_i,\Pi_j$.

($\Leftarrow$): For the converse, assume that the quadrics contain six lines satisfying the two conditions. Since $1.$ is satisfied, there exists a $F_Q^{ij}$ such that $S_P^{ij}$ is the pullback of $F_Q^{ij}$, and the lines $g_{P_i}^{j}$ and $g_{P_j}^{i}$ are its epipolar lines (by \cref{prop:conjugates_and_permissible_lines}). Denote the plane spanned by $\gp{i}{j}$ and $\gp{i}{k}$ by $\Pi_i$.

Let $x$ and $y$ be the two points where $\Pi_i\cap\Pi_j$ intersects $\gp{i}{k}$ and $\gp{j}{k}$ respectively (in the case where $\Pi_i\cap\Pi_j$ is a plane, let $x$ and $y$ be any two points lying on $\gp{i}{k}$ and $\gp{j}{k}$ respectively). If $x$ and $y$ are equal, we get
\begin{align*}
0&=x^{T}S_P^{ij}x\\
&=x^{T}P_i^{T}F_Q^{ij}P_jy\\
&=(e_{Q_i}^{k})^{T}F_Q^{ij}e_{Q_j}^{k}),
\end{align*}
proving that $F_Q^{ij}$ satisfies (\ref{eq:compatible_non_collinear}). On the other hand, if they are different, they span a line lying in $\Pi_i\cap\Pi_j$, which in turn lies in $S_P^{ij}$. In other words:
\begin{align*}
\alpha x+\beta y \in S_{P}^{ij}
\end{align*}
for all $[\alpha:\beta]\in\p1$. Now
\begin{align*}
P_ix&=e_{Q_i}^{k} &&P_jx=\gamma' e_{Q_j}^{i}+\delta' e_{Q_j}^{k}\\
P_iy&=\gamma e_{Q_i}^{j}+\delta e_{Q_i}^{k} &&P_jy=e_{Q_j}^{k}
\end{align*}
for some $[\gamma:\delta],[\gamma':\delta']\in\p1$, so
\begin{align*}
0&=(\alpha x+\beta y)^{T}S_P^{ij}(\alpha x+\beta y)\\
&=(P_i(\alpha x+\beta y))^{T}F_Q^{ij}P_j(\alpha x+\beta y)\\
&=(\alpha e_{Q_i}^{k}+\beta(\gamma e_{Q_i}^{j}+\delta e_{Q_i}^{k}))^{T}F_Q^{ij}(\alpha(\gamma' e_{Q_j}^{i}+\delta' e_{Q_j}^{k})+\beta e_{Q_j}^{k})\\
&=(\alpha e_{Q_i}^{k}+\beta\delta e_{Q_i}^{k})^{T}F_Q^{ij}(\alpha\delta' e_{Q_j}^{k}+\beta e_{Q_j}^{k})\\
&=(\alpha+\beta\delta)(\alpha\delta'+\beta)(e_{Q_i}^{k})^{T}F_Q^{ij}e_{Q_j}^{k}.
\end{align*}
This holds for all $[\alpha:\beta]\in\p1$. In particular, we can take $[\alpha:\beta]$ to be such that $(\alpha+\beta\delta)(\alpha\delta'+\beta)\neq0$, proving that $F_Q^{ij}$ satisfies (\ref{eq:compatible_non_collinear}). Hence, $S_P^{12}$, $S_P^{13}$, $S_P^{23}$ form a compatible triple of quadrics in this case also.
\end{proof}

\subsection{Ambiguous points}
\label{sec:ambigious_points}
Having dealt with the first issue, the fact that not all triples of quadrics provide three conjugate cameras $Q_1,Q_2,Q_3$, we move on to the second issue, namely that even if a configuration $(P_1,P_2,P_3,X)$ lies on the intersection of three compatible, critical quadrics, it might not be a critical configuration.

By \cref{cor:critical_configurations_lie_on_quadrics_and_cubics}, the set of critical points is the intersection of three quadrics and some number of cubic surfaces. Hence, there might be some points in the intersection $S_P^{12}\cap S_P^{13}\cap S_P^{23}$ which are not critical.

\begin{proposition}
Given two triples of cameras $(P_1,P_2,P_3)$ and $(Q_1,Q_2,Q_3)$, and a point $x\in S_P^{12}\cap S_P^{13}\cap S_P^{23}$. The three lines
\begin{align*}
l'_{Q_i}=\overline{Q_i^{-1}(P_ix)}\subset\p3
\end{align*}
either have a point in common or all lie in the same plane.
\end{proposition}
\begin{proof}
For any point $x\in S_P^{12}$ there exists a point $y$ (its conjugate with respect to the two cameras $Q_1,Q_2$) satisfying $P_1x=Q_2y$ and $P_2x=Q_2y$. It follows that the two lines $l'_{Q_1}$ and $l'_{Q_2}$ intersect. So if $x\in S_P^{12}\cap S_P^{13}\cap S_P^{23}$ each pair of lines $l'_{Q_i}$ will intersect. This leaves two options, the $l'_{Q_i}$ are either three lines lying in the same plane, or they have at least one point in common.
\end{proof}

\begin{definition}
Given two triples of cameras, the intersection $S_P^{12}\cap S_P^{13}\cap S_P^{23}$ is the union of two subvarieties: the critical points, for which $l'_{Q_i}$ pass through a common point, and the \emph{ambiguous points}, for which the lines $l_{Q_i}$ are coplanar.
\end{definition}

\begin{remark}
The subvarieties are not always disjoint, three lines may all lie in the same plane \emph{and} have a point in common.
\end{remark}

The issue of these ambiguous points is well-known, and descriptions can be found in \cite[Section 15.3]{hartleyzisserman}, and, for a more reconstruction-oriented angle, in \cite[Section 6]{HK}. In this section, we recall a key result by Hartley and Kahl \cite[Theorem 6.17]{HK}, although approached from a more algebraic setting; we also provide a few additional results on the ambiguous points in the case where the cameras are collinear. 

\subsubsection{The non-collinear case}
\begin{lemma}
\label{lem:residual_is_intersection_of_three_planes}
Let $P_1$, $P_2$, $P_3$, and $Q_1$, $Q_2$, $Q_3$ be two triples of cameras, such that the camera centers $q_i$ are not collinear. Then the ambiguous points on $S_P^{12}\cap S_P^{13}\cap S_P^{23}$ are the intersection of three planes, one through each camera center, where the plane through the $i$-th camera center is:
\begin{align*}
\Pi_i&=\vspan(g_{P_i}^{j},g_{P_i}^{k})\\
&=\vspan(P_i^{-1}(e_{Q_i}^{j}),P_i^{-1}(e_{Q_i}^{k}))\\
&=\vspan(P_i^{-1}(Q_i(q_j)),P_i^{-1}(Q_i(q_k))
\end{align*} 
\end{lemma}
\begin{proof}
The ambiguous points on $S_P^{12}\cap S_P^{13}\cap S_P^{23}$ correspond exactly to where the three lines $l'_{Q_i}$ all lie in the same plane. Since $l'_{Q_i}$ passes through $q_i$ and the three camera centers are not collinear, the three lines $l'_{Q_i}$ lie in the same plane only if they lie in the plane $\Pi_Q$ spanned by the camera centers $q_i$. In each image, $Q_i(\Pi_Q)$ is the line spanned by the two epipoles $e_{Q_i}^{j},e_{Q_i}^{k}$, so $P_i^{-1}(Q_i(\Pi_Q))$ is the plane $\Pi_i$ spanned by $g_{P_i}^{j}$ and $g_{P_i}^{k}$. The ambiguous points are the intersection of these three planes.
\end{proof}

\begin{theorem}[{\cite[{Theorem 6.17}]{HK}}]
\label{thr:compatible_implies_critical}
Let $P_1,P_2,P_3$ be a triple of cameras, let $S_P^{12}$, $S_P^{13}$, $S_P^{23}$ be a compatible (in the non-collinear sense) triple of quadrics, and let $X$ be their intersection. The compatible triple of quadrics comes with six permissible lines, spanning three planes. Let $x$ be the intersection of these three planes. Then the cameras $P_i$ along with the points in $\overline{X-x}$ constitute a critical configuration.
\end{theorem}
\begin{proof}
Since the three quadrics $S_P^{ij}$ are compatible, there exist three non-collinear cameras $Q_1,Q_2,Q_3$ giving rise to these quadrics, so we can find a conjugate configuration. By \cref{lem:residual_is_intersection_of_three_planes}, the ambiguous points on $X$ are exactly the points $x$, so $(P_1,P_2,P_3,X-x)$ is a critical configuration. By continuity, so is $(P_1,P_2,P_3,\overline{X-x})$.
\end{proof}

\begin{remark}
In the case where $x$ is a plane, it might happen that the configuration $(P_1,P_2,P_3,\overline{X-x})$ is not maximal, as the intersection of the critical and ambiguous points need not lie in the closure of $X-x$. For instance, in the case where the three quadrics are all equal to the union of two planes, $x$ is the plane containing the camera centers. $(P_1,P_2,P_3,\overline{X-x})$ is not a maximal critical configuration in this case, as there is a conic curve in $x$ which also lies in the set of critical points (see \cref{prop:plane+conic_critical}).
\end{remark}

\cref{thr:compatible_implies_critical} tells us all we need to classify the (non-collinear) critical configurations for three views. We need only check in what ways three compatible quadrics can intersect, and then remove a point/line/plane if needed. In particular, if the intersection $S_P^{12}\cap S_P^{13}\cap S_P^{23}$ does not contain a plane, a line, or any isolated points, the whole intersection is part of a critical configuration.

\subsubsection{The collinear case}
We now consider the other case, where $(P_1,P_2,P_3)$ and $(Q_1,Q_2,Q_3)$ are two triples of cameras, such that the camera centers $q_1,q_2,q_3$ lie on a line. In this case, the ambiguous points form all of $S_P^{12}\cap S_P^{13}\cap S_P^{23}$. Indeed, whenever the lines $l'_{Q_i}$ intersect in a point, the three lines are coplanar, since the camera centers $q_i$ lie on a line.

Since the set of critical points is a subset of the ambiguous points, it is not enough to simply remove a point, line, or plane. It follows that \cref{thr:compatible_implies_critical} does not hold in the case where the cameras on the other side are collinear. Unlike the non-collinear case, the permissible lines on the quadrics $S_P^{ij}$ are not enough to determine the critical points either. We do, however, have the following result, which holds for both the collinear and non-collinear cases:
\begin{proposition}
\label{prop:reconstructible_codim_1_in_residual}
The intersection of the critical points with the ambiguous points forms a subvariety of the latter of codimension at most one\footnote{exactly one if we assume that the ambiguous points are not contained in the set of critical points}.
\end{proposition}
\begin{proof}
Given three cameras $Q_1,Q_2,Q_3$, consider the variety 
\begin{align*}
A=V(F_Q^{12},F_Q^{13},F_Q^{23})\subset\pxpxp.
\end{align*}
For each point $(x_1,x_2,x_3)\in A$, the three lines $l'_{Q_i}=\overline{Q_i^{-1}(P_ix)}\subset\p3$ either lie in the same plane or have a common point, or both. In the family of triples of coplanar lines, one through each camera center, requiring the third line to pass through the intersection of the first two is (generally) a linear constraint. Hence the triples of coplanar lines all having at least one point in common form a codimension one subvariety of this family. To get to $S_P^{12}\cap S_P^{13}\cap S_P^{23}$ however, we need to intersect $A$ with the multi-view variety $\mathcal{M}_P$. This intersection is a point, line, or plane (as discussed before), in the event of a line or plane, the critical points form a codimension one subvariety of the line/plane. If the intersection is a single point, this point may or may not be critical, we will see both happen in the following section.
\end{proof}

\section{Critical configurations for three views}
\label{sec:critical_configurations_for_three_views}
In this section, we classify all the critical configurations for three views. By \cref{cor:critical_lies_on_intersection_of_quadrics}, and the discussion in \cref{sec:compatible_triples_of_quadrics}, any critical configuration for three views lies on the intersection of three compatible quadrics. Furthermore, any such intersection is critical, with the possible exception of some number of ambiguous points (usually a point, line, or plane). As such, our approach is to consider all possible intersections of three quadrics, and for each case check whether it exists as the intersection of a compatible triple, and if so, remove any potential ambiguous points. This gives us all critical configurations for three views. We start with the cases where the quadrics intersect in a surface (\cref{ssec:quadrics_intersecting_in_surface}), then a curve (\cref{sec:quadrics_intersecting_in_a_curve}), and finally in a finite number of points (\cref{sec:quadrics_intersecting_in_points}).

\begin{remark}
The intersection of quadrics may sometimes contain non-reduced components\footnote{Components with multiplicity greater than one}. One such example is three lines through a single point, with one line having multiplicity 2, which appears as the intersection of three reducible quadrics. However, since the definition of critical configurations is a set-theoretical one, it does not see multiplicity of components. As such, we consider all components to be reduced.
\end{remark}

\begin{remark}
While the critical configurations for two views tend to have a finite number of conjugates, this is generally not the case for three views. The only critical configurations for three views that have a finite number of conjugates are the ones consisting of six or seven points. The rest all have an infinite number of conjugates (although the dimension of the family of conjugates may vary). The proofs usually revolve around constructing a single conjugate, but the choices made in the construction should make it clear that many conjugates exist when different choices are made.
\end{remark}

\subsection{Quadrics intersecting in a surface}
\label{ssec:quadrics_intersecting_in_surface}
In the two-view case, the points in a maximal critical configuration always form a surface. We begin by checking whether any such configurations remain critical in the three-view case. In particular, we want to examine whether there exist any critical configurations where the points form a surface. There are two ways three quadrics can intersect in a surface, either by all being equal or by all being reducible and sharing a common plane.

\subsubsection{The smooth quadric/cone}
\label{sec:all_quadrics_coincide}
We start with the case where the quadrics $S_P^{ij}$ are all the same smooth quadric or cone (denoted simply $S_P$). In this case, it is impossible to find six lines satisfying the conditions in \cref{thr:compatible_quadrics}, so one cannot find a non-collinear conjugate triple of cameras $Q_1,Q_2,Q_3$. However, as long as $S_P$ is not one of the following configurations (illustrated in \cref{fig:impossible_configurations})
\begin{itemize}
\item $S_P$ is smooth and there is a camera center such that the two lines passing through it, each pass through one other camera center.
\item $S_P$ is a cone and two camera centers lie on the same line, but neither of the two lies at the vertex
\end{itemize} 
then three copies of $S_P$ satisfy the conditions of \cref{prop:compatible_collinear}, and hence form a compatible triple, coming from a triple of collinear cameras $Q_i$.

\begin{figure}[]
\begin{center}
\includegraphics[width = 0.8\linewidth]{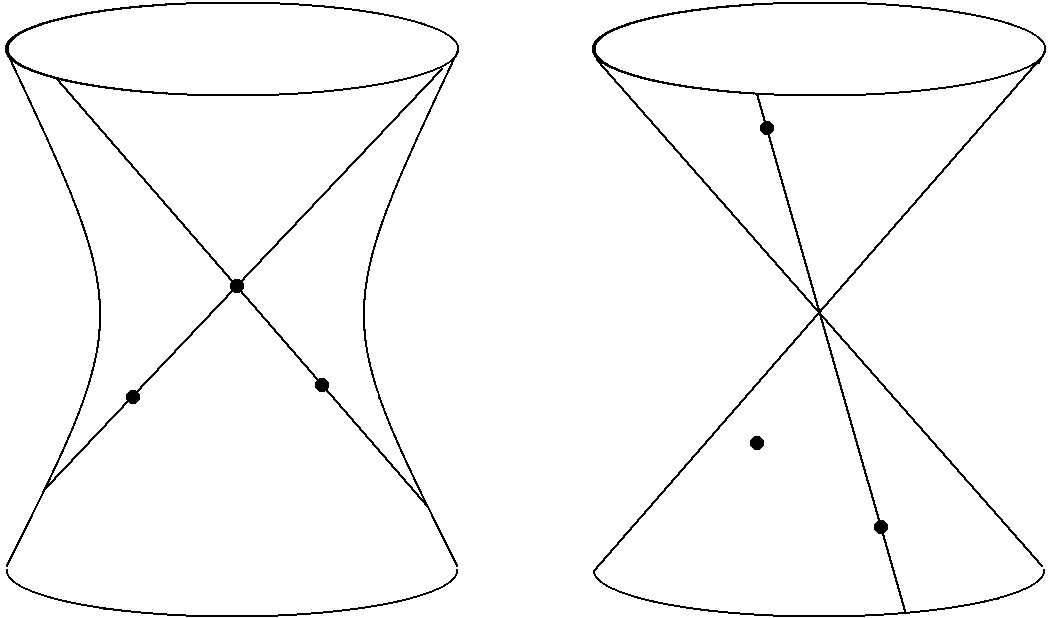}
\end{center}
\caption{If the camera centers are positioned as shown, the three quadrics (all three equal) do not constitute a compatible triple.}\label{fig:impossible_configurations}
\end{figure}

\begin{proposition}
\label{prop:rational_curves}
Let $P_1,P_2,P_3$ be three cameras, and let $S_P^{12},S_P^{13},S_P^{23}$ be a compatible triple of quadrics. If the three quadrics are all equal to the same irreducible quadric $S_P$, the set of critical points $X$ is a rational curve of degree at most four passing through all camera centers. The conjugate configuration is a curve of one degree less, not passing through any camera centers (plus the line passing through the camera centers, if $S_P$ is a cone).
\end{proposition}

\begin{remark} 
The two cases in \cref{fig:impossible_configurations} are exactly the two where $S_P$ does not contain a smooth rational curve passing through all three camera centers.
\end{remark}

\begin{proof}
Since the three cameras $Q_i$ on the other side have to be collinear, our critical configuration does not consist of all the points on $S_P$. By \cref{prop:reconstructible_codim_1_in_residual}, the critical points on $S_P$ form a codimension one subvariety, that is, a curve $C_P\subset S_P$.

By \cref{cor:critical_configurations_lie_on_quadrics_and_cubics}, the set of critical points is the intersection of three quadric and some number of cubic surfaces. The three quadrics are all equal to $S_P$, so $C_P$ is the intersection of $S_P$ with one or more cubic surfaces, all of which contain all three camera centers. It follows that $C_P$ is of type $(a,b,c_1,c_2)$ (recall \cref{def:type_of_curve}) with $a,b\leq3$ and $c_1,c_2\geq1$.

The curve $C_P$ lies on all three quadrics $S_P^{ij}$ and its type $(a,b,c_i,c_j)$ depends on which quadric we consider since the multiplicity in each camera center could be different. However, the total degree of the conjugate curve $C_Q$ depends (by \cref{prop:bidegree_on_P_->_bidegree_on_Q}) on the multiplicity of $C_P$ in the camera centers. Since the curve $C_Q$ is the same regardless of which quadric we consider, it follows that $C_P$ needs to have the same multiplicity in all three camera centers, the same is true for $C_Q$.

If we assume that the three camera centers $p_i$ are not collinear, the three quadrics $S_Q^{ij}$ on the other side cannot be the same quadric. It follows that the conjugate curve $C_Q\subset S_Q^{12}\cap S_Q^{13}\cap S_Q^{23}$ is of type $(a',b',c_1',c_2')$ with $a',b'\leq2$. If, on the other hand, the camera centers $p_i$ are collinear, the resulting curve is still the intersection of quadrics and cubics, and so needs to satisfy $a',b'\leq3$.

Furthermore, if $S_P$ is smooth, the curve $C_Q$ cannot pass through any of the three camera centers $q_i$, since this would imply that all three quadrics $S_Q$ contain the line spanned by the camera centers, which (by \cref{tab:configurations_and_their_conjugates}) violates our condition that $S_P$ is smooth.

On the other hand, if $S_P$ is a cone, then the three quadrics $S_Q^{ij}$ must (by \cref{tab:configurations_and_their_conjugates}) contain the line $l_Q$ spanned by the camera centers. However, the whole line $l_Q$ is conjugate to a single point, namely the vertex of $S_P$ (by \cref{lem:relations_between_quadrics}). So in the case where $S_P$ is a cone, $C_Q$ is the union of $l_Q$ and some other curve with $a',b'\leq2$.

This leaves only a finite number of options for what kind of curve $C_P$ can be. Finally, the curves should satisfy the result in \cref{prop:bidegree_on_P_->_bidegree_on_Q}: that if $C_P$ is of type $(a,b,c_1,c_2)$, then $C_Q$ is of type $(a,a+b-c_1-c_2,a-c_2,a-c_1)$. We can now check every possible type for $C_P$ and remove all cases where $a'$ or $b'$ are either negative or greater than 3, the cases where $c_i'$ is negative, and the cases where the multiplicity in the camera centers is greater than what the degree of the curve allows (for instance, a curve of degree one can not have multiplicity 2 in a camera center).

This leaves only three options, listed in \cref{tab:curves_on_same_quadric}. Note that if $S_P$ is a cone, then the curve on the other side will also include the line $l_Q$ in addition to the component given in the table.
\end{proof}

\begin{table}[h]
\centering
\captionsetup{width=1\linewidth}
\begin{tabular}{@{}p{0.18\textwidth}p{0.18\textwidth}@{}}
\toprule
Curve on $S_P$ & Curve on $S_Q$ \\ \midrule
(1,3,1,1) & (1,2,0,0) \\ 
(1,2,1,1) & (1,1,0,0) \\ 
(1,1,1,1) & (1,0,0,0) \\ \bottomrule
\end{tabular}
\caption{Recall that a curve of type $(a,b,c_1,c_2)$ intersects the family containing the epipolar lines $a$ times, the lines in the other family $b$ times, and passes $c_i$ times through the camera center $p_i$.}
\label{tab:curves_on_same_quadric}
\end{table}

\begin{remark}
This proves that any critical configuration where the three quadrics on one side coincide has to be one of the types in \cref{tab:curves_on_same_quadric}. In \cref{sec:quadrics_intersecting_in_a_curve} we prove the converse, that any configuration of such type is critical. The general case is the rational quartic \footnote{degree four curve} on one side, and a twisted cubic on the other, while the two other configurations appear as degenerates of this one. 
\end{remark}

\subsubsection{Two planes}

\begin{proposition}
\label{prop:plane+conic_critical}
When the three quadrics $S_P^{ij}$ are all equal to the same two planes, the set of critical points is the union of a plane and a conic passing through all three camera centers. Conversely, any configuration where the points lie on the union of a plane and a conic curve passing through the camera centers is a critical configuration. Its conjugate is of the same type (plane + conic curve).
\end{proposition}

\begin{proof}
Let the three quadrics $S_P^{ij}$ all be equal to the same two planes $S_P$. By \cref{thr:critical_conf_for_two_views}, all three camera centers need to lie in the same plane, with any number of them lying on the intersection of both. In this case, unlike the previous one, there exist six lines satisfying the conditions in \cref{thr:compatible_quadrics} as long as not all three camera centers lie on the intersection of the two planes. An example is given in \cref{fig:two_planes}. Hence, the three cameras $Q_i$ on the other side can be chosen to not be collinear. 

Since the three cameras $Q_i$ on the other side can be assumed to not be collinear, \cref{thr:compatible_implies_critical} proves that the plane not containing all camera centers is part of the set of critical points. Furthermore, \cref{prop:reconstructible_codim_1_in_residual} proves that there is a curve $C_P$ in the plane with the camera centers which is also part of this set. Since the cameras $Q_i$ are not collinear, \cref{lem:three_bilinear_and_10_trilinear} tells us that the set of critical points is the intersection of the two planes with a single cubic surface passing through the three camera centers, so the curve $C_P$ is the union of a conic passing through the three camera centers and the line where the two planes intersect. 

For the converse, we want to determine the conic curve $C$ in the set of critical points. We already know that $C_P$ passes through the three camera centers. Furthermore, the intersection of the two permissible lines $g_{P_i}^{j}$ and $g_{P_k}^{j}$ (denote it by $a_j$) also lies in the set of critical points since it has a conjugate point, namely $q_j$. This gives us three more points on the conic. Usually, there is no conic passing through six given points. In this case, however, the six points $p_1,p_2,p_3,a_1,a_2,a_3$ along with the six lines $g_{P_i}^{j}$ span a hexagon whose opposite sides ($g_{P_i}^{j}$ and $g_{P_j}^{i}$) meet at three points which lie on a line (the intersection of the two planes) so by the Braikenridge–Maclaurin theorem \cite[p.76]{conversePascal} the six points $p_1,p_2,p_3,a_1,a_2,a_3$ lie on a conic.

Finally, for any conic $C_P$ through $p_1,p_2,p_3$ one may choose the six lines $g_{P_i}^{j}$ in such a way that the three points $a_i$ lie on $C_P$ \emph{and} each pair of lines form a permissible pair. This can be done by picking $a_1$ to be any\footnote{not equal to a camera center and not lying in the intersection of the two planes} point on $C_P$ and then taking $a_2$ to be the intersection of $g_{P_1}^{2}$ with $C_P$. Since we can get any conic by making the appropriate choice of permissible lines, this proves that any plane + conic configuration is critical. 
\end{proof}

\begin{figure}[]
\begin{center}
\includegraphics[width = 0.90\linewidth]{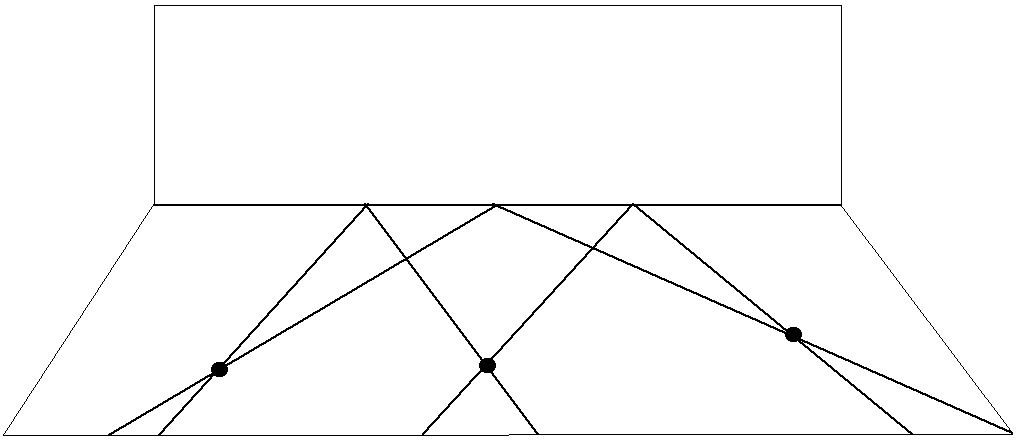}
\end{center}
\caption{One example of six lines satisfying the conditions in \cref{thr:compatible_quadrics}}\label{fig:two_planes}
\end{figure}

\subsubsection{Single plane}
With this, there is only one way left that the three quadrics $S_P^{ij}$ can intersect in a surface, namely if they are all reducible, and contain only one common plane. In this case, the set of critical points is at most a plane plus a line. Hence, this is not a maximal critical configuration, as it is always a subconfiguration of the plane + conic configuration in \cref{prop:plane+conic_critical}.

\subsection{Quadrics intersecting in a curve}
\label{sec:quadrics_intersecting_in_a_curve}
Curves that appear as the intersection of quadrics have at most degree 4. Moreover, any curve that lies on the union of a plane and a conic passing through the camera centers is critical by \cref{prop:plane+conic_critical}, which further reduces the number of curves we need to check. We analyze all curves appearing as the intersection of quadrics sorted by the degree of their highest degree component. In \cref{sec:quartic_curves}, we analyze the curves where the highest degree component is of degree 4, covering the elliptic quartic as well as the singular rational quartics. \cref{sec:cubic_curves} covers the curves containing a twisted cubic, while \cref{sec:curves_of_lower_degree} covers the curves containing no components of degree higher than 2.

For most curves, the main issue is to find a compatible (in the non-collinear sense) triple of quadrics containing it. The issue of having to remove a plane, line, or isolated point (as per \cref{thr:compatible_implies_critical}) is less relevant since the curves will usually not contain any such component. To find compatible triples of quadrics, the following lemmas are useful:

\begin{lemma}
\label{lem:compatible}
Let $P_1$, $P_2$, $P_3$ be a triple of cameras, and let $S_P^{12}$, $S_P^{13}$, $S_P^{23}$ be a triple of critical quadrics, each passing through two camera centers (indicated by the superscript). Then the three quadrics are compatible if
\begin{enumerate}
\item[A)] The intersection of $S_P^{ij}$ and $S_P^{ik}$ does not contain a line passing through $p_i$.
\item[B)] There exists a point $x$ lying on all three quadrics, as well as three distinct lines $l_{12}$, $l_{13}$, $l_{23}$ passing through $x$, satisfying
\begin{enumerate}
\item[1.] $l_{ij}\subset S_P^{ij}$.
\item[2.] $p_i\in \vspan(l_{ij},l_{ik})$.
\item[3.] $l_{ij}\notin \vspan(l_{ik},l_{jk})$.
\item[4.] $p_i,p_j\notin l_{ij}.$
\end{enumerate}
\end{enumerate}
\end{lemma}

\begin{proof}
Given three quadrics $S_P^{ij}$, assume there exists a point $x$ and three lines $l_{ij}$ satisfying the four conditions. For each quadric $S_P^{ij}$, pick two lines $g_{P_i}^{j}$ and $g_{P_j}^{i}$ such that they both intersect $l_{ij}$, and form a permissible pair. The only case where such a choice is not possible is if the quadric $S_P^{ij}$ contains the line spanned by $p_i$ and $p_j$, and the line $l_{ij}$ intersects this line, but this would violate either condition 2, 3 or 4.

This gives us six lines $g_P$, two on each quadric, and two through each camera center. Condition $A)$ ensures that the two lines through each camera center are distinct, so they span a plane. By construction, the plane spanned by $l_{ij}$ and $l_{ik}$ is the same as the one spanned by $\gp{i}{j}$ and $\gp{i}{k}$. This means that the lines $g_P$ satisfy the conditions in \cref{thr:compatible_quadrics}. It follows that the three quadrics are compatible.
\end{proof}

\begin{lemma}
\label{lem:compatible_line}
Let $P_1$, $P_2$, $P_3$ be a triple of cameras, and let $S_P^{12}$, $S_P^{13}$, $S_P^{23}$ be a triple of critical quadrics. If there exists a line $L$ lying on all three quadrics, but not passing through the camera centers such that the intersection of $S_P^{ij}$ and $S_P^{ik}$ does not contain a line passing through $p_i$ and intersecting $L$, then the quadrics form a compatible triple.
\end{lemma}

\begin{proof}
On each quadric $S_P^{ij}$, take $g_{P_i}^{j}$ to be the line through $p_i$ intersecting $L$. By the same arguments as in the proof of \cref{lem:compatible}, this gives us six lines satisfying \cref{thr:compatible_quadrics}.
\end{proof}

\begin{remark} 
In \cref{lem:compatible,lem:compatible_line}, the lines $l_{ij}$ and $L$ all intersect the permissible lines, meaning they are of type $(1,0,0,0)$ on their respective quadrics.
\end{remark}

\begin{remark}
\cref{lem:compatible,lem:compatible_line} have the condition that the intersection of two quadrics does not contain a line passing through their common camera center. This is to ensure that the two lines $g_{P_i}^{j}$ and $g_{P_i}^{k}$ span a plane, which is a necessary condition for the quadrics to be compatible (in the non-collinear sense). This is an open condition, and in the following sections, whenever we apply either of the two lemmas, the cases where two lines through a camera center coincide can easily be avoided, if such cases exist at all. To keep the arguments relatively short, the verification of this fact is left to the reader.
\end{remark}

\subsubsection{Quartic curves}
\label{sec:quartic_curves}
We begin with the case where the intersection of the quadrics is a quartic curve. This happens when the quadrics lie in the same pencil\footnote{pencil: a one-dimensional family of objects}, that is, when the intersection of the three quadrics is equal to the intersection of any two of them.

\begin{remark}
Not every quartic curve appears as the intersection of a pencil of quadrics, for instance, four skew lines do not. The ones that do are the elliptic quartic curve, the singular rational quartic (cusp or node), the union of a twisted cubic and a secant\footnote{secant to a curve: a line spanned by two points on the curve} line, and the union of two (possibly reducible) conics intersecting in two (possibly complex) points. When we talk about quartics in the rest of this section, it is to be understood that we mean quartics appearing as the intersection of quadrics.
\end{remark}

As it turns out, a set of three cameras and any number of points lying on a quartic curve $C$ form a critical configuration if every camera center lies on $C$, but is not a singular point on $C$. We give a different proof of this statement than what appears in \cite{threeViews}, instead giving one similar to that in \cite{HK} since this will more easily generalize to the case of 4+ views. To do this we first introduce the notion of quasi-parallel lines.

\begin{definition}
Given a quartic curve $C$, two secants $L_1$ and $L_2$ to $C$ are called \emph{quasi-parallel} if they lie on the same quadric in the pencil of quadrics containing $C$, and form a permissible pair on this quadric. We use the notation $L_1\parallel L_2$ when $L_1$ is quasi-parallel to $L_2$.
\end{definition}

\begin{lemma}[{\cite[Theorem 8.21]{HK}}]
\label{lem:quasi_parallel}
Let $p_1,p_2,p_3,y_1,y_2,y_3$ be six distinct points in the non-singular locus of a quartic curve $C$, such that $\overline{p_1y_2}\parallel\overline{p_2y_1}$ and such that $\overline{p_1y_3}\parallel\overline{p_3y_1}$. Then $\overline{p_2y_3}\parallel\overline{p_3y_2}$.
\end{lemma}

\begin{proof}
Project from a generic point $x$ on $C$, if a plane contains more than one of the lines $\overline{p_iy_j}$, we pick a point $x$ not lying in this plane. This projection gives us a plane (possibly reducible) cubic curve $C'$ along with six points $p_1',p_2',p_3',y_1',y_2',y_3'$ on $C'$. The secants $\overline{p_iy_j}$ and $\overline{p_jy_i}$ are only quasi-parallel if the two lines $\overline{p_i'y_j'}$ and $\overline{p_j'y_i'}$ intersect in a point on $C'$. Since $\overline{p_2'y_1'}$ intersects $\overline{p_1'y_2'}$ in a point on $C'$, and the same is the case for $\overline{p_1'y_3'}$ and $\overline{p_3'y_1'}$, the Cayley–Bacharach theorem \cite{bacharach} (see \cref{fig:cayley_bacharach}) states that $\overline{p_2'y_3'}$ intersects $\overline{p_3'y_2'}$. This means that $\overline{p_2y_3}\parallel\overline{p_3y_2}$, thus completing the proof.
\begin{figure}[]
\begin{center}
\includegraphics[width = 0.9\linewidth]{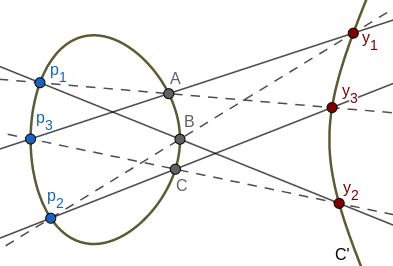}
\end{center}
\caption{Illustration of the Cayley–Bacharach theorem. Given 9 points appearing as the intersection of two cubics, any cubic containing 8 of them also contains the ninth. This figure shows three cubics, two of which are reducible (the union of three lines).}\label{fig:cayley_bacharach}
\end{figure}
\end{proof}

\begin{lemma}
\label{lem:quartic_lemma}
Let $C$ be a quartic curve, let $p_1,\ldots,p_n\in C$ be camera centers, and let $y_1,\ldots,y_n\in C$ be points such that $\overline{p_1y_i}\parallel\overline{p_iy_1}$ for all $i\leq n$. Let $S_P^{ij}$ be the unique quadric containing $C$ and the line $\overline{p_iy_j}$, then for all $i,j,k\leq n$ the triple $S_P^{ij},S_P^{ik},S_P^{jk}$ is compatible.
\end{lemma}
\begin{proof}
By \cref{lem:quasi_parallel}, $\overline{p_1y_i}\parallel\overline{p_iy_1}$ for all $i$, implies $\overline{p_iy_j}\parallel\overline{p_jy_i}$ for all $i,j$. Let $S_P^{ij}$ be the unique quadric containing $C$ and having $\overline{p_iy_j}\parallel\overline{p_jy_i}$ as its epipolar lines. Since $y_k$ lies on $C$, and hence on $S_P^{ij}$, we have 
\begin{align*}
e_{Q_i}^{k}F_Q^{ij}e_{Q_j}^{k}&=y_k^{T}P_i^{T}F_Q^{ij}P_jy_k\\
&=y_k^{T}S_P^{ij}y_k,\\
&=0,
\end{align*}
Which by \cref{thr:compatible_forms_noncollinear} means that each triple of quadrics is compatible.
\end{proof}

\begin{proposition}
\label{prop:elliptic_curve_is_critical}
Let $C$ be a quartic curve appearing as the intersection of a pencil of quadrics. Let $P_1,P_2,P_3$ be three cameras whose centers are non-singular points on $C$. Then $(P_1,P_2,P_3,C)$ is a critical configuration. The conjugate is a quartic curve passing through the camera centers.
\end{proposition}

\begin{proof}
Let $y_1$ be any smooth point on $C$ not lying in the plane spanned by the camera centers. Let $S_P^{12}$ denote the unique quadric containing $C$ and the secant $\overline{p_2y_1}$. The quadric $S_P^{12}$ contains a unique line $L$ through $p_1$ such that $l$ and $\overline{p_2y_1}$ form a permissible pair. The line $L$ intersects $C$ in two points, one being $p_1$, we denote the other one by $y_2$. We construct $S_P^{13}$ and $y_3$ in a similar fashion, using $p_3$ in place of $p_2$. Let $S_P^{23}$ be the quadric containing $C$ and the line $\overline{p_2y_3}$. By construction, $y_1,y_2,y_3$ satisfy the conditions in \cref{lem:quartic_lemma}, proving that the three quadrics $S_P^{ij}$ are compatible, and thus, that the configuration is critical, with the possible exception of some ambiguous points. 

By \cref{lem:residual_is_intersection_of_three_planes}, the ambiguous points form a plane, line or an isolated point lying on $C$. In particular, they are the intersection of the three planes
\begin{align*}
\Pi_i=\vspan(p_i,y_j,y_k) \quad\quad\forall i\neq j\neq k.
\end{align*}
The quartic curve $C$ can never contain a plane or an isolated point, and in most cases, it does not contain any line $L$ either. If it does, one can always choose the $y_i$ such that the three planes $\Pi_i$ do not intersect in $L$. Hence avoiding any trouble with ambiguous points. 

Assuming at least one of the quadrics $S_P^{ij}$ is irreducible, the curve $C$ (a quadric passing through all camera centers) is of type $(2,2,1,1)$ on $S_P^{ij}$, so by \cref{prop:bidegree_on_P_->_bidegree_on_Q}, its conjugate is also of type $(2,2,1,1)$. If none of the quadrics are irreducible, it is of type $(2,2,1,1,1)$, in which case the conjugate, $(2,2,1,1,1)$ is still of the same type.
\end{proof}


In the construction in the proof of \cref{prop:elliptic_curve_is_critical}, we take $y_1$ to not lie in the plane spanned by the camera centers, which will generally ensure that $y_2,y_3$ do not lie in this plane either, this in turn means that the conjugate configuration consists of points on a curve isomomorphic to $C$. There is, however, one case where it is unavoidable that at least one of the $y_i$ lie in the plane spanned by the camera centers, this is in the case where the curve $C$ contains a line $L$ passing through at least two of the camera centers. In this case, the conjugate curve is singular, with the singular point being conjugate to the entire line $L$ (recall \cref{lem:relations_between_quadrics}, \cref{item3}). Moreover, if the line $L$ passes through \emph{exactly} two of the camera centers, say $p_1, p_2$, then the conjugate configuration has the conjugate to the final camera, in this case, $q_3$, lying in the singular locus. We prove any such configuration is critical:

\begin{proposition}
\label{prop:elliptic_curve_is_critical_even_if_singular}
Let $C$ be a singular quartic curve appearing as the intersection of a pencil of quadrics. Let $P_1,P_2,P_3$ be three cameras whose centers lie on $C$ and with the camera center of $P_3$ lying in the singular locus of $C$. Then $(P_1,P_2,P_3,C)$ is a critical configuration. The conjugate is a reducible curve consisting of a line $L$ through the conjugate camera centers $q_1,q_2$ and a twisted cubic curve through $q_3$, having $L$ as a secant. 
\end{proposition}

\begin{proof}
Let $S_P^{12}$ be the unique cone containing $C$ and having a vertex at $p_3$. We take the epipolar lines $g_{P_1}^{2}$ and $g_{P_2}^{1}$ to be the lines $\overline{p_1p_2}$ and $\overline{p_2p_3}$, these form a permissible pair on $S_P^{12}$. Next, consider the set of planes containing the line spanned by $p_1,p_2$, each plane intersects $C$ in two, usually distinct, points. There are however three planes where the two intersection points coincide, one where they are both equal to $p_3$ and two where they coincide as a point on $C$ distinct from any camera center, denote either of these points by $y_3$. Let $S_P^{13}$ and $S_P^{23}$ both be equal to the unique quadric containing both the secants $\overline{p_1y_3}$ and $\overline{p_2y_3}$ (by construction of $y_3$, these secants lie on the same quadric). Let the epipolar line $g_{P_i}^{3}$ be the line $\overline{p_iy_3}$ and let the epipolar lines $g_{P_3}^{i}$ be such that $g_{P_i}^{3}$ and $g_{P_3}^{i}$ form a permissible pair on $S_P^{i3}$. Since each pair of epipolar lines $e_{P_i}^{k}, e_{P_j}^{k}$ intersect in a point on $S_P^{ij}$, we get that $(e_{Q_i}^{k})F_Q^{ij}e_{Q_j}^{k}=0$, hence proving that the quadrics are compatible (by \cref{thr:compatible_forms_noncollinear}). This in turn means that the configuration $(P_1,P_2,P_3,C)$ is critical.

The second half of the statement follows from \cref{prop:bidegree_on_P_->_bidegree_on_Q} along with \cref{lem:relations_between_quadrics}, \cref{item3}.
\end{proof}

We round off the quartic curve section by proving that the curves in 

\begin{lemma}The critical configurations in \cref{prop:elliptic_curve_is_critical,prop:elliptic_curve_is_critical_even_if_singular} are the only critical configurations where the points lie on a quartic curve.
\end{lemma}
\begin{proof}
By \cref{thr:critical_only_if_critical_for_fewer_views} any quartic curve that is part of a critical configuration must lie on the intersection of three critical quadrics. Any quartic curve appearing as the intersection of three different quadrics where each quadric contains two camera centers contains all three camera centers. The only way for such a quartic curve to not contain all camera centers is if two of the quadrics are equal. Then the curve only contains two camera centers. In this case the curve is of type $(2,2,1,0)$ on the two equal quadrics and $(2,2,1,1)$ on the last. If such a configuration was critical, the conjugate curve would have to be of type $(2,3,2,1)$ on two of the quadrics and $(2,2,1,1)$ on the last, but a curve of degree five does not appear as the intersection of quadrics, so this can not be a critical configuration.
\end{proof}

\subsubsection{The twisted cubic}
\label{sec:cubic_curves}
Having covered the quartic curves, we move one degree lower, to the cubics. Plane cubics do not appear as the intersection of quadrics, so the only irreducible cubic curve that might be critical is the twisted cubic curve. The results in \cref{prop:elliptic_curve_is_critical,prop:elliptic_curve_is_critical_even_if_singular} already prove that any configuration of points and cameras lying on the union of a twisted cubic and a secant line are critical, with a conjugate of the same type. There are, however, some cases where the conjugate is not of the same type, we show these here. We also give a critical configuration that is not a subconfiguration of the quartic curve ones.

\begin{proposition}
\label{prop:new_configurations}
Any configuration consisting of three cameras along with a set of points lying on a twisted cubic $C$ passing through two of the camera centers is critical. The same is the case if $C$ passes through only one camera center, as long as the two remaining camera centers lie on a secant to $C$.
\end{proposition}
\begin{proof}
Both configurations consist of cameras and points lying on the union of a twisted cubic + secant line, so by \cref{prop:elliptic_curve_is_critical}, they are critical.
\end{proof}

\begin{remark}
Neither of the configurations in \cref{prop:new_configurations} appear in \cite{HK}, the latter does not appear in \cite{threeViews} either, demonstrating that both previous classifications are incomplete.
\end{remark}

\begin{proposition}
\label{prop:twisted_cubic_is_critical}
Any configuration consisting of three cameras and any number of points all lying on a twisted cubic is critical. A conjugate consists of points on a conic curve not passing through the camera centers. 
\end{proposition}

\begin{proof}
Given a twisted cubic, there is a two-dimensional family of quadrics containing it. This leaves us 2 degrees of freedom when choosing each of the quadrics. In other words, for each line intersecting the cubic, there is a unique quadric containing both the cubic and the line (unless the line is a secant, in which case a whole pencil of quadrics contains it). 

Take a line $L$, secant to $C$, and not passing through the camera centers. There exists a pencil of quadrics containing both $C$ and $L$. By \cref{lem:compatible_line}, any three distinct quadrics in this pencil form a compatible triple of quadrics. The line $L$ has bidegree $(1,0)$ and is a secant to $C$, so it intersects $C$ twice. It follows that the twisted cubic is of type $(1,2,1,1)$, so the conjugate is of type $(1,1,0,0)$, a conic not passing through the camera centers.
\end{proof}

\cref{prop:twisted_cubic_is_critical} shows that three cameras on a twisted cubic also has other conjugates than the ones in \ref{prop:elliptic_curve_is_critical}. As for the case where the camera centers do not all lie on the twisted cubic, these will generally only be critical if the cameras not on the twisted cubic all lie on the same secant line. There is, however, one exception:

\begin{proposition}
\label{prop:twisted_cubic+line_critical}
Any configuration consisting of three collinear cameras along with a set of points lying on a twisted cubic not intersecting the line with the camera centers is critical. The conjugate is a rational quartic curve containing the three camera centers.
\end{proposition}
 
To prove this, we first give this, seemingly unrelated, lemma:

\begin{lemma}
\label{lem:cubic_and_line}
Let $C\in\p3$ be a twisted cubic, let $l$ be a line in $\p3$, not secant to $C$, and let $p$ be a fixed point on $l$ not lying on the twisted cubic. There exists a point $x\in C$ such that for every point $q\in l$, such that $q\neq p$ and $q\notin C$, the unique quadric surface containing $C$, $p$, and $q$, also contains the line $\overline{xq}$.
\end{lemma}

\begin{proof}
Let $s$ denote the unique secant line of $C$ passing through $p$. The plane spanned by $s$ and $l$ intersects $C$ in three points. Two of these points lie on $s$, the last one does not. We denote this one by $x$. 

In the pencil of quadrics containing $C$ and $p$, each quadric also contains the secant $s$. For any point $q$ as above, the quadric containing $C$, $p$, and $q$, intersects the line $\overline{xq}$ in three points, $x$, $q$, and a third point lying on $s$. Since it contains three points on $\overline{xq}$, it contains the whole line. 
\end{proof}

\begin{proof}[Proof of \cref{prop:twisted_cubic+line_critical}]
Let $C$ denote the twisted cubic, let $p_i$ denote the camera centers, and let $L$ denote the line spanned by the camera centers.

Let $S_P^{ij}$ be the unique quadric containing $C$, $p_i$, and $p_j$. If $S_P^{ij}$ is smooth there are two lines through $p_i$. One is a secant to $C$, the other line intersects $C$ exactly once; denote the latter line by $g_{P_i}^{j}$. If $S_P^{ij}$ is a cone, there is only one line through $p_i$, in this case, we take this one to be $\gp{i}{j}$. Repeating this process for each pair of cameras on each of the three quadrics, we get $6$ lines, two on each quadric, and two through each camera center. 

By \cref{lem:cubic_and_line}, the lines $g_{P_i}^{k}$ and $g_{P_j}^{k}$ intersect $C$ at the same point, which we denote by $x_k$ (see \cref{fig:cubicline}). Since $x_k$ lies on $S_P^{ij}$, we have:
\begin{align*}
x_k^{T}S_P^{ij}x_k&=x_k^{T}P_i^{T}F_Q^{ij}P_jx_k,\\
&=e_{Q_i}^{k}F_Q^{ij}e_{Q_j}^{k},\\
&=0.
\end{align*}
Meaning the quadrics $S_P^{12},S_P^{13},S_P^{23}$ do indeed form a compatible triple.

\begin{figure}
\begin{center}
\includegraphics[width = 0.7\linewidth]{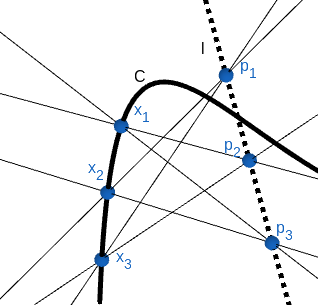}
\end{center}
\caption{The twisted cubic $C$ and the line $l$. The line $g_{p_i}^{j}$ is the line spanned by $p_i$ and $x_j$.}
\label{fig:cubicline}
\end{figure}

Since the twisted cubic is of type $(1,2,0,0)$, its conjugate is of type $(1,3,1,1)$, a rational quartic curve passing through all three camera centers. This is one of the curves from \cref{prop:rational_curves}, appearing when all the quadrics $S_Q^{ij}$ on the other side coincide.

In this construction, the pair of lines $g_{P_i}^{j}$ through each camera center may coincide. In this case, the three points $x_i$ also coincide, and the three quadrics $S_P^{ij}$ are all cones. By \cref{prop:compatible_collinear} the quadrics are compatible in this case also, and the camera centers on the rational quartic on the other side are collinear.
\end{proof}

In the case where the line containing the cameras intersects the twisted cubic, but is not a secant or tangent, the three quadrics $S_P^{ij}$ all have to be equal. In this case, the only way the lines $g_{P_i}^{j}$ can be permissible is if they are all secant to the twisted cubic, but in this case, the twisted cubic is of type $(2,1,0,0)$ (or type $(3,1,1,1)$ if you include the line) and neither of these appear as critical configurations when the three quadrics intersect. Hence no such critical configurations exist.

\subsubsection{Curves of lower degree}
\label{sec:curves_of_lower_degree}
We are now left with curves appearing as the intersection of quadrics where all components are of degree at most 2. These are all either subsets of degenerate quartic curves or degenerate quartic curves themselves, or subsets of the \enquote{plane + conic} configuration. Hence critical in all three cases, having conjugates that are of the same type (see \cref{prop:plane+conic_critical,prop:elliptic_curve_is_critical}). 

However, among the configurations that \emph{are} subsets of the \enquote{plane + conic} configuration or the degenerate quartic configuration, some also have conjugates that are not of the same type, as was the case for three cameras on a twisted cubic. This happens because for curves of low degree one has more freedom when choosing the quadrics, so more conjugates may appear. For the sake of completeness, we give a summary of them here. No proof is given, but the reader can verify that the curves lie on the intersection of three compatible irreducible quadrics. This task should be fairly straightforward when the curves are of a low degree.
\begin{enumerate}
\item Two intersecting lines on one side, and two disjoint lines on the other. The number of cameras lying on the two lines is equal on both sides.
\item Cameras lying on a conic curve. The conjugate is a line not passing through the camera centers.
\item Two lines not passing through the camera centers. This one has two conjugates, the first is three cameras lying on the union of three lines (one intersecting the other two). The second conjugate consists of points lying on the union of a conic and a line, and the cameras lying on the conic. This triple of conjugates is a degeneration of the \enquote{twisted cubic $\leftrightarrow$ conic} configuration shown in \cref{prop:twisted_cubic_is_critical}.
\end{enumerate}

\subsection{Quadrics intersecting in a finite number of points}
\label{sec:quadrics_intersecting_in_points}
Lastly, we have the case where three quadrics intersect in a finite number of points. This is the \enquote{general} case in the sense that with two general triples of cameras, the set of critical points on each side is finite. Indeed, on each side, we get a compatible triple of quadrics, whose intersection is 8 points (counted with multiplicity and including possible complex points), seven of these points lie in the set of critical points, whereas the last one is an ambiguous point (recall \cref{lem:residual_is_intersection_of_three_planes}).

Conversely: given a set of seven points $X$ and three cameras $P_i$, when is $(P_1,P_2,P_3,X)$ a critical configuration? Since 9 points\footnote{7 points and two camera centers.} are enough to fix a quadric, there is generally only one triple of quadrics $S_P^{ij}$ such that $p_i,p_j,X\in S_P^{ij}$. Now IF this triple of quadrics is compatible, it comes with six epipolar lines, spanning three planes. The three planes intersect in one of the 8 points lying on the intersection of the three quadrics. This point is, by \cref{lem:residual_is_intersection_of_three_planes}, not critical. As such, the configuration $(P_1,P_2,P_3,X)$ is critical if and only if the triple of quadrics is compatible \emph{and} the three planes do not intersect in one of the points in $X$ (by \cref{thr:compatible_implies_critical}). This gives us the compatibility condition for seven general points.

\begin{proposition}
\label{prop:7_points_critical}
Let $P_1,P_2,P_3$ be three cameras, and let $X$ be a set consisting of seven points. Then $(P_1,P_2,P_3,X)$ is critical if and only if there exists three quadrics $S_P^{ij}$ satisfying:
\begin{enumerate}
\item $p_i,p_j,X\in S_P^{ij}$.
\item $S_P^{12},S_P^{13},S_P^{23}$ is compatible.
\item The intersection of the three planes spanned by the six epipolar lines $\gp{i}{j}$ does not lie in $X$.
\end{enumerate}
\end{proposition}

\begin{remark}
A sextuple of epipolar lines $\gp{i}{j}$ only exists if the quadrics are compatible, so the first two conditions need to be satisfied for the third to make sense. 
\end{remark}

There are, of course, certain positions of seven points and three camera centers such that they do not fix all three quadrics, for instance when points and camera centers coincide. If $S_P^{ij}$ is not fixed, this is because there exists a (possibly degenerate or reducible) elliptic quartic curve that passes through $p_i,p_j$ and $X$. In the previous section we have shown that if the curve also passes through the final camera, the whole curve is part of a critical configuration, so the seven points along with the cameras is a critical (but not maximal) configuration in this case.

For the sake of completeness, we examine what happens if $X$ does not consist of exactly seven points. For eight or more points, a configuration is critical if and only if it is a subconfiguration of one of the critical configurations mentioned earlier in the section, i.e. if the points lie on a critical curve or surface. For six or fewer points, the configuration is always critical \cite{6points}, but not maximal, since any configuration is contained in a critical configuration consisting of at least seven points. In particular, the generic 6-point configuration is contained in exactly three critical 7-point configurations.

This completes the classification of the critical configurations for three views:

\begin{theorem}
\label{thr:critical_configuration_three_views}
A configuration of three projective cameras $P$ and any number of points $X$ form a critical configuration if and only if the points $X$ form one of the surfaces/curves/point sets described in \cref{fig:12critical}, or if they form a subset of one of these.
\end{theorem}

\begin{figure*}[p]
\begin{center}
\includegraphics[width = 0.85\textwidth]{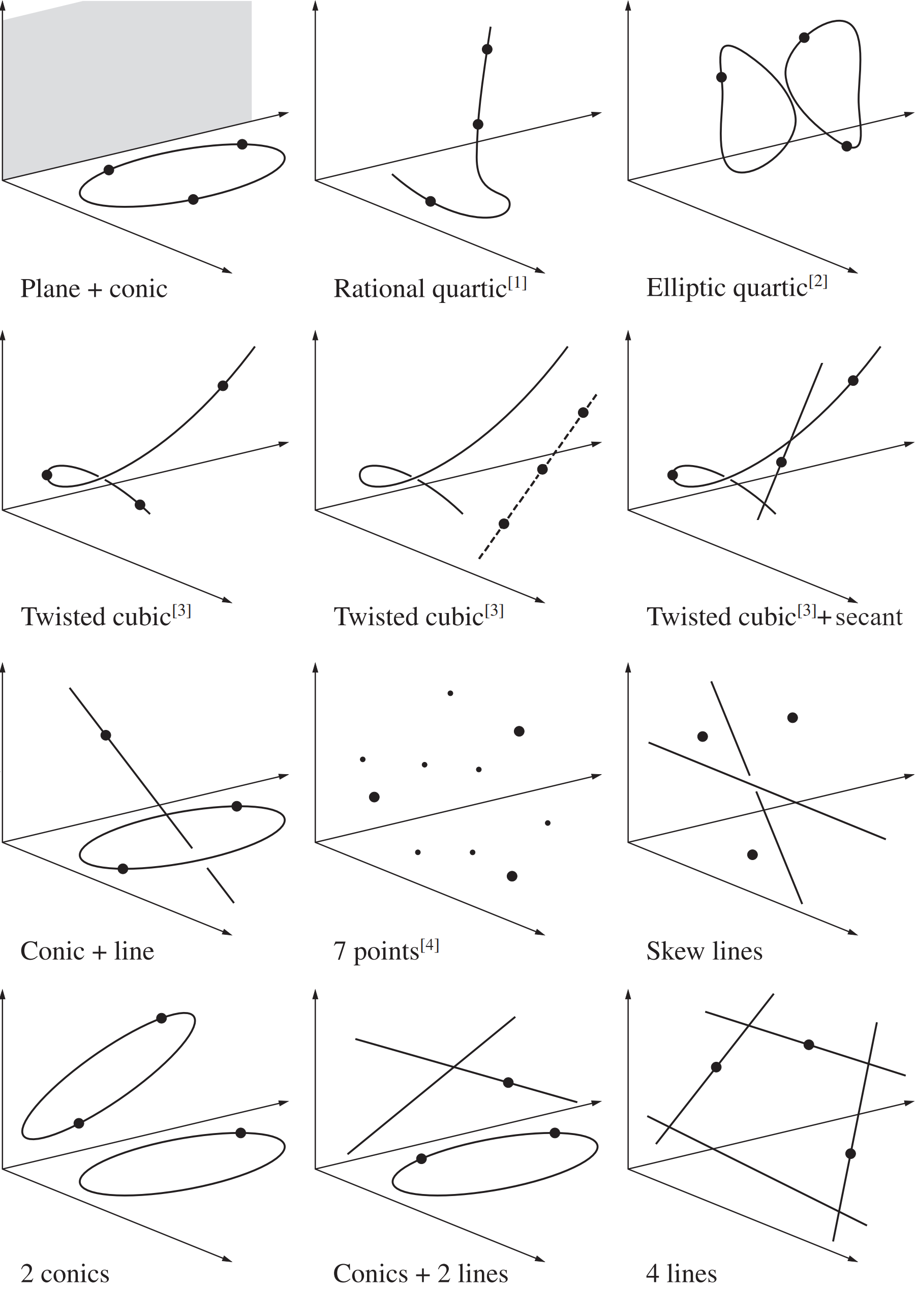}
\end{center}
\caption{\begin{small}
Critical configurations for three views.\\
{[1]} The rational quartic may have a node or cusp.\\
{[2]} The elliptic quartic may degenerate to the union of a twisted cubic and a secant line, or to the union of two (possibly reducible) conics intersecting in two points.\\
{[3]} The twisted cubic may degenerate to the union of a conic and a line, or to three lines.\\
{[4]} The 7-point configuration must satisfy the conditions in \cref{prop:7_points_critical}.
\end{small}}
\label{fig:12critical}
\label{fig:cubic+line}
\end{figure*}

\section{More than three views}
\label{sec:four_views}
\subsection{A counterexample}
\label{sec:counterexample}
By \cref{thr:critical_only_if_critical_for_fewer_views}, any critical configuration for four or more views has to be critical for each triple of views. Contrary to what is claimed in \cite[Theorem 9.34]{HK}, however, the converse is not true. Although a configuration is critical for each triple of views, the whole configuration might not be critical. The simplest counterexample is that of four cameras and six points in general position.

\begin{theorem}[\cite{6points}]
\label{thm:6points}
A set of six points and any number of cameras is a critical configuration if and only if the six points and all the cameras all lie on a smooth ruled quadric.
\end{theorem}

While six points and three cameras in general position always lie on a smooth quadric, this is not true for four cameras. So while such a configuration is critical for each triple of views, it is not critical for all four views.

The proof of \cite[Theorem 9.34]{HK} fails on two points
\begin{enumerate}
\item Given a configuration $(P_1,P_2,P_3,P_4,X)$ which is critical for each triple of views, there is a compatible triple of quadrics $S_P^{12},S_P^{13},S_P^{23}$ containing $p_1,p_2,p_3$ and $X$, and there is a compatible triple of quadrics $S_P^{12'},S_P^{14'},S_P^{24'}$ containing $p_1,p_2,p_4$ and $X$. There is however, no reason why the quadric $S_P^{12}$ should equal $S_P^{12'}$, which seems to be assumed in \cite{HK}. This is what gives rise to the 6-point counterexample.
\item Until the recent results in \cite{compatibility}, it was commonly believed that a sextuple of quadrics $S_P^{12},S_P^{13},S_P^{23},S_P^{14},S_P^{24},S_P^{34}$ was compatible as long as each triple was compatible. As it turns out, this is not the case, and additional conditions are needed to ensure compatibility. Since a critical configuration for four views is the intersection of a compatible sextuple of quadrics, we need to check some additional conditions that do not appear in the three-view case.
\end{enumerate}

As such, the case of four cameras needs to be handled on its own and does not follow directly from the three-view case. That being said, while the proof of \cite[Theorem 9.34]{HK} is flawed, the result itself is not far from being true. As we will soon show, the only three-view critical configurations that fail to remain critical when we add cameras turn out to be the 6- and 7-point configurations.

\subsection{Preliminaries}

Since any critical configuration for four or more views has to be critical for each triple of views, any critical configurations for four or more views has to be either among the ones in \cref{fig:12critical}, or has to be the proper intersection of the ones in \cref{fig:12critical} with some number of quadric and cubic surfaces. Intersecting any of the critical configurations in \cref{fig:12critical} with a surface, however, leaves either a finite set of points or something that is already in \cref{fig:12critical}. As such, in order to classify the critical configurations for four views, one needs only check which of the critical configurations from the three-view case remain critical when one adds a fourth camera. We cover these cases one by one, in the same order as in \cref{sec:critical_configurations_for_three_views}, but using more computational proofs than in the previous section. For the most part, we either prove compatibility by explicitly giving a conjugate, or by checking that the quadrics come from a set of fundamental matrices satisfying:

\begin{proposition}[{\cite[Theorem 3.6]{compatibility}}] 
\label{thm: 4tuple-condition} Let $\Set{F^{ij}}$ be a sextuple of fundamental matrices such that the three epipoles in each image do not lie on a line. Then $\Set{F^{ij}}$ is compatible if and only if the triple-wise conditions hold and
\begin{small}
\begin{align}
\begin{aligned}
\label{eq:quadruple_compatibility}
&(e_1^4)^TF^{12}e_2^3(e_1^2)^TF^{13}e_3^4(e_1^3)^TF^{14}e_4^2\\ \cdot&(e_2^4)^TF^{23}e_3^1(e_2^1)^TF^{24}e_4^3(e_3^2)^TF^{34}e_4^1\\
=&(e_1^3)^TF^{12}e_2^4(e_1^4)^TF^{13}e_3^2(e_1^2)^TF^{14}e_4^3\\ \cdot&(e_2^1)^TF^{23}e_3^4(e_2^3)^TF^{24}e_4^1(e_3^1)^TF^{34}e_4^2.
\end{aligned}
\end{align}
\end{small}
\end{proposition}

The case of more than four cameras is further simplified by the following fact:

\begin{proposition} [{\cite[Theorem 3.12]{compatibility}}]
\label{thm:compatible_if_each_sextuple_is_compatible}
Let $\Set{F^{ij}}$ be a complete set of $n\choose2$ fundamental matrices such that for all $i,j,k$ the triple $F^{ij},F^{ik},F^{jk}$ is compatible, and such that for all $i,j,k,l$, the sextuple $F^{ij},F^{ik},F^{jk},F^{il},F^{jl},F^{kl}$ is compatible. Then $\Set{F^{ij}}$ is compatible.

Moreover, if all epipoles in each image coincide, then triple-wise compatibility implies that $\Set{F^{ij}}$ is compatible. The reconstruction in this case is a set of cameras whose centers all lie on a line.
\end{proposition}

\begin{corollary}
\label{cor:compatible_if_each_sextuple_is_compatible_quadrics}
For $n\geq4$, an $n\choose2$-tuple of quadrics is compatible if and only if each sixtuple $(S^{ij},S^{ik},S^{jk},S^{il},S^{jl},S^{kl})$ is. 

Moreover, if the cameras on the other side are all collinear, an $n\choose2$-tuple of quadrics is compatible if and only if each triple $S^{ij},S^{ik},S^{jk}$ is.
\end{corollary}
\begin{proof}
By definition, a set of quadrics is compatible if and only if it is the pullback of a set of compatible fundamental matrices. So this result follows immediately from \cref{thm:compatible_if_each_sextuple_is_compatible}.
\end{proof}

\begin{corollary}
Any configuration that is critical for each quadruple of views is critical. 
\end{corollary}

\begin{lemma}
\label{lem:pencil_has_one_compatible}
Let $\Set{P_i}$ be a set of four cameras, and let $\Set{S_P^{ij}}$ be a set of six irreducible critical quadrics such that each triple of quadrics $S_P^{ij},S_P^{ik},S_P^{jk}$ is compatible (in the non-collinear sense). If we fix five of the quadrics and let the final one vary over all quadrics in space, then
\begin{enumerate}
\item The final quadric can move in a linear family without breaking triple-wise compatibility.
\item There is exactly one quadric in the pencil such that the whole sextuple of quadrics is compatible.
\end{enumerate}
\end{lemma}
\begin{proof}
Each irreducible critical quadric $S_P^{ij}$ is the pullback of a unique fundamental matrix $F_Q^{ij}$. The set of quadrics is compatible if it is the pullback of a set of compatible fundamental matrices. Fixing five of the quadrics corresponds to fixing five of the fundamental matrices. Let the final quadric (the one we do not fix) be the quadric $S_P^{34}$, the pullback of $F_Q^{34}$.
\begin{enumerate}
\item A triple of fundamental matrices is compatible if and only if
\begin{align*}
e_{Q_i}^{k}F_Q^{ij}e_{Q_j}^{k}=0
\end{align*} 
for all triples of indices. The matrix $F_P^{34}$ appears in two triples of cameras, namely $134$ and $234$. For each of these triples, triple-wise compatibility gives three conditions on $F_P^{34}$ for a total of 6. Since a fundamental matrix has seven degrees of freedom, the triple-wise conditions restrict $F_P^{34}$ to a pencil. 
\item By \cref{thm: 4tuple-condition}, a sextuple of fundamental matrices is compatible if each triple is compatible and it satisfies 
\begin{small}
\begin{align}
\tag{\ref{eq:quadruple_compatibility}}
\begin{aligned}
&(e_1^4)^TF^{12}e_2^3(e_1^2)^TF^{13}\underline{e_3^4}(e_1^3)^TF^{14}e_4^2\\ \cdot&(e_2^4)^TF^{23}e_3^1(e_2^1)^TF^{24}\underline{e_4^3}(e_3^2)^T\underline{F^{34}}e_4^1\\
=&(e_1^3)^TF^{12}e_2^4(e_1^4)^TF^{13}e_3^2(e_1^2)^TF^{14}\underline{e_4^3}\\ \cdot&(e_2^1)^TF^{23}\underline{e_3^4}(e_2^3)^TF^{24}e_4^1(e_3^1)^T\underline{F^{34}}e_4^2.
\end{aligned}
\end{align}
\end{small}
For ease of reading, the variables that are not fixed by fixing the first five quadrics are underlined. Moreover, the fact that $F_Q^{34}$ has to satisfy triple-wise compatibility with the triples 134 and 234 fixes its left and right nullspaces $e_{Q_3}^{4}$ and $e_{Q_4}^{3}$. This leaves $F_Q^{34}$ as the only non-fixed variable in \cref{eq:quadruple_compatibility}, turning it into a linear equation in the entries of $F_Q^{34}$. This in turn means that there is exactly one value of $F_Q^{34}$ satisfying \cref{eq:quadruple_compatibility} and hence exactly one quadric making the sextuple compatible. \qedhere
\end{enumerate}
\end{proof}

\subsection{Critical configurations}

\begin{proposition}
A set of points $X$ lying on the union of a plane and a conic curve together with any number of cameras whose centers lie on that same conic form a critical configuration.
\end{proposition}

\begin{proof}
By \cref{cor:compatible_if_each_sextuple_is_compatible_quadrics}, it suffices to show that this configuration is critical for four cameras. Let $X$ be the union of a plane $\Pi$ and a conic curve $C$ and let $P_1,\ldots,P_4$ be a set of four cameras whose centers all lie on the conic curve. Since we are free to choose coordinates in $\p3$ without affecting criticality, we can take the four camera centers to be 
\begin{align*}
p_1=\begin{bmatrix}
1:0:0:0
\end{bmatrix}, \quad p_3=\begin{bmatrix}
0:0:1:0
\end{bmatrix},\\  p_2=\begin{bmatrix}
0:1:0:0
\end{bmatrix}, \quad p_4=\begin{bmatrix}
1:1:1:0
\end{bmatrix}.
\end{align*}
By also requiring the plane $\Pi$ to pass through the point $[0:0:0:1]$, we have fixed the coordinate frame in $\p3$. Moreover, the only property of cameras that affects criticality is their camera centers, so we can choose coordinates in each image so that the four cameras are:
\begin{align*}
P_1&=\begin{bmatrix}
0&1&0&0\\
0&0&1&0\\
0&0&0&1
\end{bmatrix}, & \quad
P_2&=\begin{bmatrix}
1&0&0&0\\
0&0&1&0\\
0&0&0&1
\end{bmatrix},\\
P_3&=\begin{bmatrix}
1&0&0&0\\
0&1&0&0\\
0&0&0&1
\end{bmatrix}, & \quad
P_4&=\begin{bmatrix}
0&0&0&1\\
1&-1&0&0\\
0&1&-1&0
\end{bmatrix}.
\end{align*}
When $X$ is the union of a plane $\Pi$ and a conic $C$, there is only one quadric containing $X$, namely the reducible quadric consisting of the two planes $\Pi$ and $\Pi_C$ where $\Pi_C$ denotes the plane containing $C$. Denote the points in $\p3$ by $[x_1:x_2:x_3:x_4]$. We have chosen coordinates so that $\Pi_C$ is the plane where $x_4=0$. Assuming $\Pi$ does not pass through $p_3$ (which can always be avoided through a relabelling of the cameras) it will be the plane where $x_3-\lambda_1 x_1+\lambda_2 x_2=0$ for some $(\lambda_1,\lambda_2)\in\mathbb{R}^{2}$. Likewise, the conic $C$ will be given by $\mu_1x_2(x_1-x_3)+\mu_2x_3(x_1-x_2)=0$ for some $[\mu_1:\mu_2]\in\p1$.

The quadrics $S_P^{ij}$ must all be taken to be equal to the quadric consisting of the two planes $\Pi$ and $\Pi_C$. Moreover, the twelve epipolar lines $g_{P_i}^{j}$ must be chosen in such a way that:
\begin{enumerate}
\item $g_{P_i}^{j}$ passes through $p_i$ and lies in $\Pi_C$.
\item $g_{P_i}^{j}$ and $g_{P_j}^{i}$ intersect in a point on $\Pi$.
\item $g_{P_i}^{k}$ and $g_{P_j}^{k}$ intersect in a point on $C$.
\end{enumerate}
Once such a choice of lines is made (a one-dimensional family of valid choices exists), one can compute the fundamental matrices $F_Q^{ij}$ by letting $F_Q^{ij}$ be the unique fundamental matrix that pulls back to $S_P^{ij}$ (under the joint camera map of $P_i,P_j$) and whose epipoles pull back to $g_{P_i}^{j}$ and $g_{P_j}^{j}$ (under that same map). If the matrices $F_Q^{ij}$ are then compatible, and in this particular case, they always are, one can recover the conjugate cameras $Q_i$, and then find the conjugate points $Y$ by triangulation, thus constructing a conjugate configuration. One example of such a conjugate configuration consists of the cameras 
\begin{align*}
Q_1&=\begin{bmatrix}
1&0&0&0\\
0&1&0&0\\
0&0&1&0
\end{bmatrix}, \quad
Q_2=\begin{bmatrix}
\lambda_2& -1& 0& 0\\
0&0&0&1\\
0& 0& \lambda_1& 0
\end{bmatrix},\\
Q_3&=\begin{bmatrix}
-\lambda_2\mu_1& \mu_1 + \lambda_1\mu_1 + \lambda_1\mu_2&   0& \mu_1 + \mu_2\\
 \lambda_1\mu_1&           \lambda_1\mu_2&   0&     \mu_2\\
   0&             0& \lambda_1\mu_1&     0
\end{bmatrix},\\
Q_4&=\begin{bmatrix}
         0&   0& \lambda_1\mu_1&     0\\
-\mu_1(\lambda_1 + \lambda_2)&   \mu_1&   0&     0\\
       \lambda_1\mu_1& \lambda_1\mu_2&   0& \mu_1 + \mu_2
\end{bmatrix},
\end{align*}
where the conjugate to any point \begin{align*}
\begin{bmatrix}x_1\\ x_2\\ \lambda_1 x_1+\lambda_2 x_2\\ x_4\end{bmatrix}\in\Pi
\end{align*}
is the point
\begin{align*}
\begin{bmatrix} x_2\\ \lambda_1x_1 + \lambda_2 x_2\\ x_4\\ - \lambda_1^2x_1 - \lambda_1\lambda_2 x_2\end{bmatrix},
\end{align*} 
and the conjugate to any point 
\begin{align*}
\begin{bmatrix}	(\mu_1 + \mu_2)x_2x_3\\
				(\mu_1x_2 + \mu_2x_3)x_2\\
				(\mu_1x_2 + \mu_2x_3)x_3\\
 0\end{bmatrix}\in C 
 \end{align*}
is the point 
\begin{align*}
\begin{bmatrix} 
(\lambda_1+ \lambda_2)x_2^2\\
(\lambda_1+ \lambda_2)x_2x_3\\
0\\
-(x_3 - \lambda_2x_2)(\mu_1x_2 + \mu_2x_3)\end{bmatrix}.
\end{align*} \qedhere

\end{proof}

\subsubsection*{Curves of degree four}

\begin{proposition}[{\cite[Corollary 8.31]{HK}}]
\label{prop:rational_quartic_critical_4+}
Any configuration of points and cameras all lying on a smooth rational quartic curve is critical. So is any configuration of points all lying on a twisted cubic and cameras all lying on a line not intersecting the cubic. These two configurations are conjugate to one another.
\end{proposition}

\begin{proof}
We have shown that for three views, these configurations are both critical and conjugate to one another. Since they form a conjugate pair, it suffices to show that one of them remains critical when we add more cameras. Consider a rational quartic curve with any number of cameras all lying on the curve. For each triple of cameras this is a critical configuration with the camera centers of the conjugate configuration all lying on a line. By the second half of \cref{cor:compatible_if_each_sextuple_is_compatible_quadrics} however, triple-wise compatibility is sufficient in the case where the conjugate cameras all lie on a line, which is the case here. Hence the whole configuration is critical.
\end{proof}

\begin{proposition}
\label{prop:singular_quartic_critical_4+}
Any configuration of points and cameras all lying on a singular rational quartic curve is critical. So is any configuration of points all lying on a twisted cubic and cameras all lying on a line secant to the cubic: These two configurations are conjugate to one another.
\end{proposition}

\begin{proof}
This is proven the same way as \cref{prop:rational_quartic_critical_4+}.
\end{proof}

Next, we show that any configuration consisting of points and cameras all lying on a quartic curve $C$ appearing as the intersection of a pencil of quadrics is critical as long as no camera center lies on a singular point on $C$. We begin by giving three useful lemmas for proving criticality in these cases.

%

\begin{proposition}
\label{prop:elliptic_quartic_4+}
A set of any number of cameras and points all lying on an irreducible elliptic or singular quartic curve is critical as long as no camera center is a singular point on $C$.
\end{proposition}
\begin{proof}
By \cref{cor:compatible_if_each_sextuple_is_compatible_quadrics}, we need only prove the result for four cameras. Denote the quartic curve by $C$. A conjugate configuration can be constructed in the following way:

Using the same construction as in the three-view case, let $y_1$ be any regular point on $C$ such that it does not lie in a plane spanned by any three camera centers. In the pencil of quadrics containing $C$, there is exactly one, $S_P^{12}$, containing the secant line $\overline{p_2y_1}$. This quadric contains a unique line $L$ through $p_1$ such that $L\parallel\overline{p_2y_1}$. $L$ intersects $C$ in two points, one is $p_1$, denote the other one by $y_2$. In a similar fashion, we can get the two remaining points $y_3$ and $y_4$. These are constructed in such a way that the conditions of \cref{lem:quartic_lemma} are satisfied, so $C$ appears as the intersection of six quadrics $S_P^{ij}$ where each triple is compatible. 

Next, we need to show that the entire sextuple is compatible. This can be done by verifying that \cref{eq:quadruple_compatibility} is satisfied. As it turns out, this is always the case, regardless of which point we take to be the initial $y_1$. While this verification is fairly straightforward for any explicitly given $C$, proving this in the general case is a rather lengthy computation. Instead, we prove that it is critical by giving the conjugate configuration for a representative example (the conjugate is constructed using the algorithm described above). This is done in the appendix.
\end{proof}

\begin{proposition}
\label{prop:cubic+secant_4_views}
Let $C$ be the union of a twisted cubic and a secant line. A set of cameras and points all lying on $C$ is critical.
\end{proposition}
\begin{proof}
By \cref{cor:compatible_if_each_sextuple_is_compatible_quadrics}, we need only prove the result for four cameras. Moreover, the case where all cameras lie on the secant line is already covered by \cref{prop:singular_quartic_critical_4+}, so we need only cover the cases where 1,2,3, or all 4 cameras lie on the twisted cubic. The approach generally follows the irreducible case, but we need to take a bit more care in which points we pick as the $y_i$:
\begin{enumerate}

\item \label{i1}If at most one camera lies on the secant line: By a relabelling, we can always take the camera centers $p_1,p_2,p_3$ to lie on the cubic. Let $y_1$ be a generic point on the twisted cubic. For each camera center $p_i\neq p_1$, let $S_P^{1i}$ be the unique quadric containing the line $\overline{p_iy_1}$. The quadric $S_P^{1i}$ contains a unique line through $p_1$ quasi-parallel to $\overline{p_iy_1}$. This line intersects the $C$ in two points, $p_i$ and a point which we label $y_i$. This construction gives us four points $y_i$ satisfying the conditions in \cref{lem:quartic_lemma}, hence proving that each triple of quadrics is compatible. The final part of the proof, showing that the entire sextuple is indeed compatible, is left to the appendix.

\item \label{i2}If two or more cameras lie on the secant line: This situation is somewhat special since the conjugate curve $C_Q$ is singular and every point on the secant line on $C$ is conjugate to the singular point on $C_Q$. In this case, like in case \ref{i1}., for each camera $p_i$ on the twisted cubic, we must take the point $y_i$ to lie on the secant line. But since the points $y_1$ map to camera centers on $C_Q$, and since the secant line collapses to a singular point on $C_Q$, this means that more than one of the conjugate cameras $Q_i,Q_j$ may lie in the singular point, which in turn means that one can not recover any fundamental matrices $F_Q^{ij}$ or quadric $S_P^{ij}$. Nonetheless, one can construct such a conjugate.

Let the camera center $p_1$ lie on the twisted cubic, and let $p_4$ lie on the secant line. Let $y_1$ be any generic point on the secant line, and let $y_4$ be any generic point on the twisted cubic. Let $S_P^{i4}$ be the unique quadric containing $C$ and the line $\overline{p_iy_4}$, and let $S_P^{1i}$ be the unique quadric containing $C$ and the line $\overline{p_iy_i}$ ($S_P^{14}$ contains both). The quadrics $S_P^{1i}$ and $S_P^{i4}$ will each contain a line quasi-parallel to $\overline{p_iy_1}$ and $\overline{p_iy_4}$ respectively. These two lines intersect in a point which we denote by $y_i$. This gives us four points $y_i$ satisfying \cref{lem:quartic_lemma} (illustrated in \cref{fig:vridd}), so each triple is compatible. Proving that the entire sextuple is compatible, is left to the appendix. \qedhere
\end{enumerate}
\end{proof}

\begin{figure}
\begin{center}
\includegraphics[width = 0.85\linewidth]{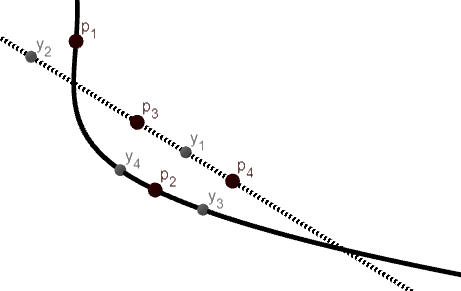}
\end{center}
\caption{One possible choice of four points $y_i$}
\label{fig:vridd}
\end{figure}

\begin{proposition}
\label{prop:two_conics_critical_4_views}
Let $C$ be the union of two (possibly reducible) conics appearing as the intersection of two quadrics. Any set of cameras and points all lying on $C$ is critical.
\end{proposition}
\begin{proof}
As before, we just need to prove this for four cameras. This result covers several cases, which we cover one by one. For all cases, denote the conic containing $p_1$ by $C_1$ and the other conic by $C_2$
\begin{enumerate}
\item In the event that all four cameras lie on the same conic, this is a subcase of the plane+conic configuration and as such, is critical.
\item \label{itemp2} In the event that three of the cameras lie on the same conic, and the final one (say, $p_1$) does not, we do as follows:

Let $y_1$ be a generic point on $C_1$, not lying in the plane containing $C_2$. Let $S_P^{1i}$ be the unique quadric containing $C_1$ and $C_2$ as well as the secant $\overline{p_iy_1}$. Chosen this way, the quadric $S_P^{1i}$ contains a unique line $L_i$ through $p_1$ such that $L_i$ and $\overline{p_iy_1}$ form a permissible pair. 
Denote the intersection point of $L_i$ with $C_2$ by $y_i$. The points $p_i$ and $y_i$ now satisfy the conditions in \cref{lem:quartic_lemma}, proving triple-wise compatibility. The proof that the whole quadruple is compatible is in the appendix.

\item \label{itemp3}In the event that there are two camera centers on each conic, we do as follows (let $p_3,p_4$ be the two camera centers on $C_2$):

First, take $S_P^{12}$ to be the reducible quadric consisting of the two planes containing $C_1$ and $C_2$. Let $y_1$ be a generic point on $C_1$, not lying in the plane containing $C_2$ (if $C_1$ is reducible, let $y_1$ be a generic point on the component not containing $p_1$) and let $y_2$ be the unique point on $C$ such that the lines $\overline{p_1y_2}$ and $\overline{p_2y_1}$ form a permissible pair on $S_P^{12}$ (i.e. they intersect in a singular point on $S_P^{12}$). 
Next, let $S_P^{13}$ be the unique quadric containing $C_1,C_2$, and the line $\overline{p_jy_1}$. The quadric $S_P^{13}$ contains a unique line $L$ such that $L$ and $\overline{p_3y_1}$ form a permissible pair on $S_P^{13}$. $L$ intersects $C_2$ in a point which we denote by $y_3$. In a similar fashion, we can construct a point $y_4$. The points $p_i$ and $y_i$ now satisfy the conditions in \cref{lem:quartic_lemma}, proving triple-wise compatibility. The proof that the whole quadruple is compatible is in the appendix. \qedhere
\end{enumerate}
\end{proof}

\subsubsection*{Curves of degree three}

\begin{proposition}
A configuration consisting of any number of points and cameras all lying on a (possibly reducible) twisted cubic is critical. Its conjugate consists of points all lying on a conic not passing through the camera centers.
\end{proposition}

\begin{proof}
By \cref{cor:compatible_if_each_sextuple_is_compatible_quadrics}, we need only prove this for the four camera case. Let $\textbf{P}$ be a set of four cameras, and let $X$ be a set of points all lying on a smooth twisted cubic $C$ passing through all camera centers. Let $L$ be a secant to the twisted cubic not passing through any camera centers. By \cref{lem:compatible_line}, any three distinct quadrics containing both $C$ and $L$ form a compatible triple. Fix five of the quadrics $S_P^{ij}$ (all but $S_P^{34)}$) to be five distinct quadrics all containing both $C$ and $L$. Then by the first item of \cref{lem:pencil_has_one_compatible}, the final quadric can move along a linear pencil without breaking triple-wise compatibility. This is exactly the pencil of quadrics containing $C$ and $L$. Moreover, by the second item of \cref{lem:pencil_has_one_compatible}, there is one quadric in this pencil which would make the whole sextuple compatible. Let $S_P^{34}$ be this quadric. Then we have a compatible sextuple of quadrics all containing $X$, making $(\textbf{P},X)$ a critical configuration. 

The same argument holds in the cases where the twisted cubic degenerates to the union of a conic and a line, or to three lines. In this case, take the line $L$ to be a secant such that $C\cup L$ is a degenerate elliptic quartic curve. That is:
\begin{enumerate}
\item If $C$ is the union of a line and a conic, let $L$ be a generic line intersecting both.
\item If $C$ is the union of three lines, two of these lines do not intersect. Let $L$ be a generic line intersecting these two lines.
\end{enumerate}
The remaining proof is as in the smooth case.
\end{proof}

\begin{proposition}
A configuration consisting of any number of cameras, along with points all lying on a twisted cubic passing through all but one camera center is critical.
\end{proposition}

\begin{proof}
By \cref{cor:compatible_if_each_sextuple_is_compatible_quadrics} it is sufficient to prove the statement for the four camera case. Denote the twisted cubic by $C_P$ and let $p_4$ be the camera center not lying on $C_P$. By changing coordinates, one can take $C_P$ to be the curve of points on the form
\begin{align*}
[\alpha^3:\alpha^2\beta:\alpha\beta^2:\beta^3] \text{ for } [\alpha:\beta]\in\p1
\end{align*}
and take the four camera centers to be 
\begin{align*}
p_1&=\begin{bmatrix}
0:0:0:1
\end{bmatrix}, & \quad p_3&=\begin{bmatrix}
1:1:1:1
\end{bmatrix},\\ p_2&=\begin{bmatrix}
1:0:0:0
\end{bmatrix}, & \quad p_4&=\begin{bmatrix}
x_1:x_2:x_3:1
\end{bmatrix}\footnotemark.
\end{align*}
\footnotetext{By choosing coordinates for $C_P$ and the first three camera centers, the coordinate frame is fixed. In this example, we assume that $p_4$ does not lie on the plane at infinity in this coordinate frame. If it does, one can amend this by relabelling the cameras $P_1,P_2,P_3$}
By \cref{prop:only_camera_centers_matter}, the only property of cameras that affects criticality is their camera centers. So we can choose coordinates in each image so that the four cameras are:
\begin{align*}
P_1&=\begin{bsmallmatrix}
1&0&0&0\\
0&1&0&0\\
0&0&1&0
\end{bsmallmatrix}, & \quad
P_2&=\begin{bsmallmatrix}
0&1&0&0\\
0&0&1&0\\
0&0&0&1
\end{bsmallmatrix},\\
P_3&=\begin{bsmallmatrix}
1&-1&0&0\\
0&1&-1&0\\
0&0&1&-1
\end{bsmallmatrix}, & \quad
P_4&=\begin{bsmallmatrix}
1&0&0&-x_1\\
0&1&0&-x_2\\
0&0&1&-x_3
\end{bsmallmatrix}.
\end{align*}

This configuration is critical. One example of a conjugate consists of cameras:

\begin{align*}
Q_1&=\begin{bsmallmatrix}
1&0&0&0\\
0&1&0&0\\
0&0&1&0
\end{bsmallmatrix}, & \quad
Q_2&=\begin{bsmallmatrix}
1&0&0&-1\\
0&2&0&1\\
0&-1&2&-1
\end{bsmallmatrix},\\
Q_3&=\begin{bsmallmatrix}
2&0&0&1\\
0&1&0&-1\\
0&1&1&1
\end{bsmallmatrix}, & \quad
Q_4&=\begin{bsmallmatrix}
1&x_1&-x_1&x_1+1\\
0&x_2&-x_2&x_2-1\\
0&x_3+1&-x_3&x_3+1
\end{bsmallmatrix}.
\end{align*}

The conjugate to any point
\begin{align*}
\begin{bsmallmatrix}
    \alpha^3:\alpha^2\beta:\alpha\beta^2:\beta^3
\end{bsmallmatrix}\in C_P
\end{align*}
is the point
\begin{align*}
\begin{bsmallmatrix}
    -\alpha^2(\alpha - \beta): \alpha\beta(\alpha - \beta):-\beta^2(\alpha - \beta): -\alpha^2\beta
\end{bsmallmatrix}\in C_Q.
\end{align*}
\end{proof}

\subsubsection*{Finite point sets}
\begin{lemma}
\label{lem:7points_2cameras_fix_third_camera}
Let $X\subset\p3$ be a set of seven points in general position and let $P_1,P_2$ be two cameras such that there is only one (possibly reducible) quadric containing $p_1$, $p_2$ and $X$. Then there exists at most one camera $P_3$ such that $(P_1,P_2,P_3,X)$ is a critical configuration.
\end{lemma}

\begin{proof}
Given $X$ and the cameras $P_1,P_2$, we want to find all cameras $P_3$ such that $(P_1,P_2,P_3,X)$ is a critical configuration. If $(P_1,P_2,P_3,X)$ is a critical configuration, then by \cref{prop:7_points_critical} there exists three planes $\Pi_i$, one through each camera center $p_i$, such that their intersection is the point $x_0$ (uniquely determined by $X$). Moreover, $X$ must appear as the intersection of three compatible quadrics $S_P^{ij}$. 

Let $S_P^{12}$ be the unique quadric containing $X,p_1,p_2$. This quadric comes with a pair of permissible lines, $g_{P_1}^{2}$ and $g_{P_2}^{1}$, one through each camera center. These are enough to fix the two planes $\Pi_1,\Pi_2$. The line $g_{P_1}^{3}$ must pass through $p_1$ and lie in the plane $\Pi_1$. In the pencil of quadrics containing $X$ and $p_1$ there are only two quadrics containing such lines, one is $S_P^{12}$, we take $S_P^{13}$ to be the other one. In a similar fashion, there is only one choice for the quadric $S_P^{23}$. 

With this, we have fixed the two lines $g_{P_1}^{3}$ and $g_{P_2}^{3}$ as well. The quadric $S_P^{13}$ contains only one line through $x_0$ intersecting $g_{P_1}^{3}$, and similarly $S_P^{23}$ contains only one line through $x_0$ intersecting $g_{P_2}^{3}$. These two lines span the third plane $\Pi_3$. The camera center $p_3$ must lie on this plane and on the two quadrics $S_P^{13},S_P^{23}$, but not on the lines $g_{P_1}^{2}$ or $g_{P_1}^{3}$. Hence, $p_3$ is uniquely determined, and, up to a choice of coordinates, so is $P_3$.
\end{proof}

\begin{proposition}
A configuration of seven points and four or more cameras is never a maximal critical configuration.
\end{proposition}

\begin{proof}
Let $P_1,\ldots,P_4$ be four cameras and let $X\subset\p3$ be a set of $7$ points. Assume there exists at least one pair of cameras, for instance, $P_1,P_2$ such that there is only one quadric containing the 9 points $p_1,p_2,X$. By \cref{lem:7points_2cameras_fix_third_camera}, there exists only one camera $P_\textnormal{x}$ such that $(P_1,P_2,P_3,X)$ is a critical configuration. It follows that both $P_3=P_4=P_\textnormal{x}$ for this to be a critical configuration, violating our assumption of distinct cameras.

On the other hand, if there is no pair of cameras such that the two camera centers and the seven points fix the quadric, then all cameras and all of $X$ lies on the intersection of a pencil of quadrics, that is, some quartic curve. In this case, the configuration is not maximal.
\end{proof}

\begin{proposition}
\label{prop:finite_points_4_views}
A set of six points and any number of cameras is critical if there is a smooth critical quadric containing all the points and all the cameras. A set of five or fewer points and two or more cameras is always critical.
\end{proposition}
\begin{proof}
The case of six points is covered in \cite{6points}. As for the the case of five or fewer points, we make use of the fact that with five fixed points in space, there are many inequivalent choices of cameras that produce the same image in $\p2$. Assume that among the five points, no four lie in the same plane. One can choose a coordinate frame so that the five points are:
\begin{align*}
x_1&=\begin{bmatrix}
1:0:0:0
\end{bmatrix}, &
x_4&=\begin{bmatrix}
0:0:0:1
\end{bmatrix},\\
x_2&=\begin{bmatrix}
0:1:0:0
\end{bmatrix}, &
x_5&=\begin{bmatrix}
1:1:1:1
\end{bmatrix}, \\
x_3&=\begin{bmatrix}
0:0:1:0
\end{bmatrix}
\end{align*}
Assuming no camera lies in the plane at infinity, we can choose coordinates in each image such that all cameras are of the form
\begin{align*}
P_i=\begin{bsmallmatrix}
1&0&0&\alpha_i^1\\
0&1&0&\alpha_i^2\\
0&0&1&\alpha_i^3
\end{bsmallmatrix}.\footnotemark
\end{align*}
\footnotetext{The superscript in $\alpha_i^j$ is an index, not an exponent.}Recall that the choice of coordinates in space does not affect criticality, so it suffices to show that this configuration is critical. Let $y_i=x_i$, and let 
\begin{align*}
Q_i=\begin{bsmallmatrix}
\lambda_i(\alpha_i^1 + 1) - \alpha_i^1&0&0&\alpha_i^1\\
0&\lambda_i(\alpha_i^2 + 1) - \alpha_i^2&0&\alpha_i^2\\
0&0&\lambda_i(\alpha_i^3 + 1) - \alpha_i^3&\alpha_i^3
\end{bsmallmatrix},
\end{align*}
for some $\lambda_i\in\mathbb{R}$. Then the cameras $Q_i$ along with points $y_i$ produce the same images as the cameras $P_i$ along with the points $x_i$. Since we are free to choose the $\lambda_i$ however we like, this gives us a large family of conjugates, hence proving that the configuration is critical. In a similar fashion (not given here) one can construct conjugates in the cases where one or more camera lies at the plane at infinity and in the case where there are four or more points in the same plane.
\end{proof}

\cref{prop:finite_points_4_views} covers the last of the three view cases, thus completing our classification. To summarize:

\begin{theorem}
\label{thr:critical_configuration_four_views}
A configuration of four or more cameras $P$ and seven or more points $X$ is critical if and only if each subconfiguration of three cameras and all points $X$ is critical. A configuration of four or more cameras and six points is critical if and only if there is a smooth critical quadric containing all the points and all the cameras. A configuration of four or more cameras and five or fewer points is always critical.
\end{theorem}

\section{Conclusion}
This paper completes the classification of all critical configurations for any number of projective cameras, showing that these configurations always have to lie on the intersection of ruled quadrics, and showing exactly which such intersections are critical. By solving the 4+ camera case left unfinished in \cite{twoViews,threeViews,HK}, it completes a long story starting with Krames work from 1941. The idea of classifying critical configurations by classifying intersections of multi-view varieties is not restricted to projective cameras, and we believe that the techniques demonstrated here can be used for other camera models in the future.

\clearpage

\appendix
\onecolumn
Some proofs in the appendix rely on \matlab code for computations, this code can be found at \url{https://github.com/mabraate/Critical-Configurations}

\begin{proof}[Proof of \cref{prop:elliptic_quartic_4+}]
We start by choosing $y_1$ to be such that there exists at least one quadric containing $C$ but not containing both the line $\overline{x_1p_2}$ and the line $\overline{x_1p_3}$ (only a few points $y_1$ fail to fulfill this criterion, and the configuration is critical even if it does not). Assume for starters that the four camera centers do not lie in a common plane, then one can choose coordinates such that $y_1=[1:1:1:1]$ and such that the four camera centers are
\begin{align*}
p_1&=\begin{bmatrix}
1:0:0:0
\end{bmatrix}, & p_3&=\begin{bmatrix}
0:0:1:0
\end{bmatrix}, \\
p_2&=\begin{bmatrix}
0:1:0:0
\end{bmatrix}, & p_4&=\begin{bmatrix}
0:0:0:1
\end{bmatrix},
\end{align*}
By \cref{prop:only_camera_centers_matter}, the only property of cameras that affects criticality is their camera centers, so we can choose coordinates in each image so that the four cameras are:
\begin{align*}
P_1&=\begin{bsmallmatrix}
0&1&0&0\\
0&0&1&0\\
0&0&0&1
\end{bsmallmatrix}, & \quad
P_2&=\begin{bsmallmatrix}
1&0&0&0\\
0&0&1&0\\
0&0&0&1
\end{bsmallmatrix},\\
P_3&=\begin{bsmallmatrix}
1&0&0&0\\
0&1&0&0\\
0&0&0&1
\end{bsmallmatrix}, & \quad
P_4&=\begin{bsmallmatrix}
1&0&0&0\\
0&1&0&0\\
0&0&1&0
\end{bsmallmatrix}.
\end{align*}
Next, since the quadrics $S_P^{12}$ and $S_P^{13}$ contain all the camera centers, the point $y_1$, and the secant line $\overline{p_2y_1}$ and $\overline{p_3y_1}$ respectively, they are on the form
\begin{align*}
S_P^{12}&=\begin{bsmallmatrix}
0&a_1&a_2&a_3\\
a_1&0&a_4&-a_1-a_4\\
a_2&a_4&0&-a_2-a_3\\
a_3&-a_1-a_4&-a_2-a_3&0
\end{bsmallmatrix}, & S_P^{13}&=\begin{bsmallmatrix}
0&b_1&b_2&b_3\\
b_1&0&b_4&-b_1-b_3\\
b_2&b_4&0&-b_2-b_4\\
b_3&-b_1-b_3&-b_2-b_4&0
\end{bsmallmatrix}.
\end{align*}
The remaining four quadrics can be written as a linear combination of these two. Moreover, every irreducible elliptic or singular quartic curve passing through the four camera centers and $y_1$ is the intersection of two quadrics on this form. This configuration turns out to be critical, with its conjugate consisting of the four cameras:
\begin{align*}
\tiny Q_1&=\begin{bsmallmatrix}
1&0&0&0\\
0&1&0&0\\
0&0&1&0
\end{bsmallmatrix}, \quad
Q_2=\begin{bsmallmatrix}
-a_1-2a_4&-a_2&a_3&1\\
2a_1+a_4&2a_2&a_3&1\\
a_4-a_1&-a_2&-2a_3&1
\end{bsmallmatrix},\\
Q_3&=\begin{bsmallmatrix}\vspace{4pt}
\begin{smallarray}{c}
(a_1b_3+2a_1b_4-3a_4b_1\\
-a_4b_3+a_4b_4) 
\end{smallarray}	&\begin{smallarray}{c}
(a_2b_3+2a_2b_4-3a_4b_2\\
+3a_3b_4-3a_4b_4) 
\end{smallarray}	&-a_3(b_3-b_4)			&b_4-b_3\\
\begin{smallarray}{c}
(a_1b_3-3a_3b_1+2a_1b_4\\
-a_4b_3+a_4b_4) 
\end{smallarray}\vspace{4pt}	&a_2b_3-3a_3b_2+2a_2b_4					&-a_3(b_3-b_4)			&b_4-b_3\\
\begin{smallarray}{c}
(a_1b_3+2a_1b_4-3a_4b_1\\
-a_4b_3+a_4b_4) 
\end{smallarray}	&a_2b_3+2a_2b_4-3a_4b_2					&2a_3b_3+a_3b_4-3a_4b_3	&b_4-b_3
\end{bsmallmatrix},&\quad\\
Q_4&=\begin{bsmallmatrix}\vspace{4pt}
\begin{smallarray}{c}
(a_1b_3+2a_1b_4-3a_4b_1\\
-a_4b_3+a_4b_4) 
\end{smallarray} & a_2(b_3-b_4) &\begin{smallarray}{c}
(3a_2b_1-3a_1b_2+3a_3b_1\\
-3a_1b_4+3a_2b_3+2a_3b_3\\
-3a_4b_2+a_3b_4-3a_4b_4) 
\end{smallarray}&b_4-b_3\\
\vspace{4pt}
\begin{smallarray}{c}
(3a_2b_1-3a_1b_2+a_1b_3\\
-a_1b_4-a_4b_3+a_4b_4) 
\end{smallarray}	&a_2(b_3-b_4)							&\begin{smallarray}{c}
(3a_2b_3-3a_3b_2\\
+2a_3b_3-2a_3b_4) 
\end{smallarray}		&b_4-b_3\\
\begin{smallarray}{c}
(a_1b_3+2a_1b_4-3a_4b_1\\
-a_4b_3+a_4b_4) 
\end{smallarray} &\begin{smallarray}{c}
(3a_2b_1-3a_1b_2+a_2b_3\\
+2a_2b_4-3a_4b_2) 
\end{smallarray}	&\begin{smallarray}{c}
(3a_3b_1-3a_1b_3+2a_3b_3\\
+a_3b_4-3a_4b_3) 
\end{smallarray}	&b_4-b_3
\end{bsmallmatrix},
\end{align*}

and points lying at the intersection of the two quadrics $S_Q^{12}$ and $S_Q^{13}=S_Q^{13'}+{S_Q^{13'}}^{T}$, where
\begin{align*}
S_Q^{12}&=\begin{bsmallmatrix}
0&a_1-a_4&2a_1+a_4&0\\
a_1-a_4&2a_2&2a_2+2a_3&-1\\
2a_1+a_4&2a_2+2a_3&2a_3&1\\
0&-1&1&0
\end{bsmallmatrix}\\
S_Q^{13'}&=\begin{bsmallmatrix}\vspace{4pt}
\begin{smallarray}{c}
(6a_4b_1-4a_1b_4-2a_1b_3\\
+2a_4b_3-2a_4b_4) 
\end{smallarray}&0&0&0\\ \vspace{4pt}
3a_4b_2-2a_2b_4-a_2b_3&0&0&0\\ \vspace{4pt}
\begin{smallarray}{c}
(a_1b_3-3a_3b_1+2a_1b_4-2a_3b_3\\
-a_3b_4+2a_4b_3+a_4b_4) 
\end{smallarray}&a_2b_3-3a_3b_2+2a_2b_4&-2a_3(b_3-b_4)&0\\
b_3-b_4&0&b_4-b_3&0
\end{bsmallmatrix}.
\end{align*}
In particular, the conjugate to a point 
\begin{align*}
x_P=\begin{bsmallmatrix}
a_1wy + a_4wy + a_2wz + a_3wz - a_4yz\\
y(a_3w + a_1y + a_2z)\\
z(a_3w + a_1y + a_2z)\\
w(a_3w + a_1y + a_2z)
\end{bsmallmatrix}
\end{align*}
on $S_P^{12}$ is the point
\begin{align*}
y_Q=\begin{bsmallmatrix}
y(w-z)\\
z(w-z)\\
w(w-z)\\
a_4yz-a_2z^2-2a_1wy-a_4wy-2a_2wz-2a_3wz-a_1yz-a_3w^2
\end{bsmallmatrix}
\end{align*}
on $S_Q^{12}$. The pair $(x,y)$ satisfies $P_ix=Q_iy$ for all $i$ whenever $x$ also lies on $S_P^{13}$, in which case the conjugate point $y$ lies on $S_Q^{13}$.

This completes the case where the four camera centers do not lie in a common plane. The case where they do can be covered in a similar way (though not repeated here) by choosing coordinates such that $y_1=[0:0:0:1]$ and such that the camera centers are
\begin{align*}
p_1&=\begin{bmatrix}
1:0:0:0
\end{bmatrix}, & p_3&=\begin{bmatrix}
0:0:1:0
\end{bmatrix}, \\
p_2&=\begin{bmatrix}
0:1:0:0
\end{bmatrix}, & p_4&=\begin{bmatrix}
1:1:1:0
\end{bmatrix},
\end{align*}
and then repeating the process from the non-coplanar case.
\end{proof}

\begin{proof}[Proof of \cref{prop:cubic+secant_4_views}]
By a change of coordinates, one can always take the twisted cubic to be on the form $[s^3:s^2t:st^2:t^3]$ for $[s:t]\in\p1$. Moreover, one can fix three points on the twisted cubic. By taking the two intersection points with the secant to be $[1:0:0:0]$ and $[0:0:0:1]$, the secant line is on the form $[u:0:0:v]$ for $[u:v]\in\p1$. We also fix a point $y_i$ on the twisted cubic, distinct from all camera centers, to be the point $[1,1,1,1]$. Finally, by a change of coordinates in each image, we can take the cameras to be:
\begin{align*}
P_i=\begin{bsmallmatrix}
1,0,0,-a_i^3\\
0,1,0,-a_i^2\\
0,0,1,-a_1
\end{bsmallmatrix}, \quad
P_j=\begin{bsmallmatrix}
1,0,0,b_j\\
0,1,0,0\\
0,0,1,0
\end{bsmallmatrix}.
\end{align*}
for $p_i$ lying on the cubic and $p_j$ lying on the secant line.

\textit{Continued proof of \cref{i1}, no cameras on the secant:}\\
We take $y_1=[1:1:1:1]$, the remaining $y_i$ turn out to be 
\begin{align*}
y_i=[a_i^3: a_1a_i^2: a_1^2a_i: a_1^3].    
\end{align*}
Taking $S_P^{ij}$ to be the quadric containing $C$ and $\overline{p_iy_j}$, the quadrics turn out to be
\begin{align}
\label{eq:quadrics_cubic+line}
S_P^{1i}=\begin{bsmallmatrix}
0&0&1&0\\
0&-2&0&-a_i\\
1&0&2a_i&0\\
0&-a_i&0&0
\end{bsmallmatrix}, &\quad & 
S_P^{ij}=\begin{bsmallmatrix}
0&0&a_1&0\\
0&-2a_1&0&-a_ia_j\\
a_1&0&2a_ia_j&0\\
0&-a_ia_j&0&0
\end{bsmallmatrix},
\end{align}
for $i,j\neq 1$. These come from the fundamental matrices
\begin{align*}
F_Q^{1i}&=\begin{bsmallmatrix}
0&  -a_1&  a_1(a_i + 1)\\
a_1&  (a_1 - a_i)(a_1 - 1)& - a_1^2a_i - a_1^2 + a_1a_i - a_i^2 - a_i\\
- a_1^2 - a_i& a_1^2a_i + a_1^2 - a_1a_i + a_i^2 + a_i&-a_i(a_1 - a_i)(a_1 - 1)
\end{bsmallmatrix},\\
F_Q^{ij}&=\begin{bsmallmatrix}
  0&a_1^2& -a_1(a_i + a_1a_j)\\
 -a_1^2&-a_1(a_i - a_j)(a_1 - 1)& a_1^2a_ia_j + a_1a_i^2 - a_1a_ia_j + a_1a_j^2 + a_ia_j\\
a_1(a_j + a_1a_i)& - a_1^2a_ia_j - a_1a_i^2 + a_1a_ia_j - a_1a_j^2 - a_ia_j&  a_ia_j(a_i - a_j)(a_1 - 1)
\end{bsmallmatrix},
\end{align*}
which satisfy the condition in \cref{thm: 4tuple-condition}\footnote{\label{foot:attachment}See attached code.}, hence proving that the sextuple is compatible.

\textit{Continued proof of \cref{i1}, one camera on the secant:}\\
We take the camera center on the secant to be $p_4$, and let $y_1=[1:1:1:1]$. In this case, $y_i=[a_i^3: a_1a_i^2: a_1^2a_i: a_1^3]$ for $i\in\Set{2,3}$ and $y_4=[-a_1b_4:0:0:1]$. Taking $S_P^{ij}$ to be the quadric containing $C$ and $\overline{p_iy_j}$, the quadrics turn out to be the same as in \cref{eq:quadrics_cubic+line}. These come from a set of fundamental matrices where $F_Q^{12},F_Q^{13},F_Q^{23}$ are as in \cref{eq:quadrics_cubic+line}, and 
\begin{align*}
F_Q^{14}&=\begin{bsmallmatrix}
 0&         1&           -1\\
-1&    1 - a_1&      a_1 + b_4\\
a_1& - a_1 - b_4& -b_4(a_1 - 1)
\end{bsmallmatrix},\\
F_Q^{i4}&=\begin{bsmallmatrix}
      0&            a_1a_i&              -a_i^2\\
 -a_1a_i&   -a_i^2(a_1 - 1)&     b4a_1^2 + a_i^3\\
a_1a_i^2& - b_4a_1^2 - a_i^3& -a_1a_ib_4(a_1 - 1)
\end{bsmallmatrix},
\end{align*}
which also satisfy the condition in \cref{thm: 4tuple-condition}\footnoteref{foot:attachment}, hence proving that the sextuple is compatible.

\textit{Continued proof of \cref{i2}, two or more cameras on the secant:}\\
In this case we may have two or three cameras on the secant line and 2 or 1, respectively, on the twisted cubic. We prove both cases at once by showing that the 5-view configuration with three cameras on the secant and two on the cubic is critical. We do this by explicitly giving a conjugate configuration:

Let $p_1,p_2$ be the camera centers on the twisted cubic, and let $y_3=[1:1:1:1]$. By the construction described in \cref{i2}, we get a conjugate configuration consisting of the cameras\footnoteref{foot:attachment}:
\begin{align*}
Q_1&=\begin{bsmallmatrix}1&0&0&0\\0&1&0&0\\0&0&1&0\end{bsmallmatrix},
Q_2=\begin{bsmallmatrix}1&a_2-a_1&-a_2(a_1-a_2)&0\\0&1&a_2-a_1&0\\0&0&1&0\end{bsmallmatrix},\\
Q_3&=\begin{bsmallmatrix}-a_1-1&(a_1+1)^2-1&a_1-a_1(a_1+1)+1&-1\\a_1^2+a_1+1&-(a_1+1)(a_1^2+a_1+1)&a_1(a_1^2+a_1+1)&0\\0&a_1^2+a_1+1&-(a_1+1)(a_1^2+a_1+1)&0\end{bsmallmatrix},\\
Q_4&=\begin{bsmallmatrix}\vspace{4pt}
b_3+a_1b_3+a_1^2b_3-a_1^2b_4&\begin{smallarray}{c}
-a_1(b_3+b_4+a_1b_3+\\ \vspace{4pt}
a_1^2b_3-a_1^2b_4) 
\end{smallarray}&\begin{smallarray}{c}
b_4(b_3-b_4+a_1b_3-a_1b_4+\\ \vspace{4pt}
a_1^2b_3-a_1^2b_4+a_1^2-1) 
\end{smallarray}&b_4\\ \vspace{4pt}
-b_4(a_1^2+a_1+1)&(b_3+a_1b_4)(a_1^2+a_1+1)&-a_1b_3(a_1^2+a_1+1)&0\\0&-b_4(a_1^2+a_1+1)&(b_3+a_1b_4)(a_1^2+a_1+1)&0\end{bsmallmatrix},\\
Q_5&=\begin{bsmallmatrix}\vspace{4pt}
b_3+a_1b_3+a_1^2b_3-a_1^2b_5&\begin{smallarray}{c}
-a_1(b_3+b_5+a_1b_3+\\ \vspace{4pt}
a_1^2b_3-a_1^2b_5) 
\end{smallarray}&\begin{smallarray}{c}
b_5(b_3-b_5+a_1b_3-a_1b_5+\\ \vspace{4pt}
a_1^2b_3-a_1^2b_5+a_1^2-1) 
\end{smallarray}&b_5\\ \vspace{4pt}
-b_5(a_1^2+a_1+1)&(b_3+a_1b_5)(a_1^2+a_1+1)&-a_1b_3(a_1^2+a_1+1)&0\\0&-b_5(a_1^2+a_1+1)&(b_3+a_1b_5)(a_1^2+a_1+1)&0\end{bsmallmatrix},
\end{align*}
and the points lying on the curve:
\begin{align*}
\begin{bsmallmatrix}
st(a_1^2t^2+a_1st+s^2)\\st^2(s+a_1t)\\st^3\\-(s-t)(a_1^2s^3+b_3a_1^2t^3+a_1s^3+a_1s^2t+b_3a_1t^3+s^3+s^2t+st^2+b_3t^3)
\end{bsmallmatrix},
\end{align*}
This is a quartic curve with a singularity in $[0:0:0:1]$.
\end{proof}

\begin{proof}[Proof of \cref{prop:two_conics_critical_4_views}]
If the four cameras all lie in the same plane, this is a subconfiguration of the plane+conic configuration, which we have already proven to be critical. As such, we assume that the four camera centers do not lie in the same plane, so we can (with the right choice of coordinates) take them to be:
\begin{align*}
P_1&=\begin{bsmallmatrix}
0&1&0&0\\
0&0&1&0\\
0&0&0&1
\end{bsmallmatrix}, & \quad
P_2&=\begin{bsmallmatrix}
1&0&0&0\\
0&0&1&0\\
0&0&0&1
\end{bsmallmatrix},\\
P_3&=\begin{bsmallmatrix}
1&0&0&0\\
0&1&0&0\\
0&0&0&1
\end{bsmallmatrix}, & \quad
P_4&=\begin{bsmallmatrix}
1&0&0&0\\
0&1&0&0\\
0&0&1&0
\end{bsmallmatrix}.
\end{align*}
\textit{Continued proof of \cref{itemp2}} (the case of three cameras on one conic):

Let $[x_1:x_2:x_3:x_4]$ denote a generic point in $\p3$. The two planes spanned by the two conics are the planes where $x_4=0$ (the plane containing $p_1,p_2,p_3$) and $a_1x_1 + a_2x_2 + a_3x_3=0$ for some $(a_1,a_2,a_3)\in\mathbb{R}^3$ (the plane containing $p_4$). Since we are free to take $y_1$ to be any point on the conic through the first three camera centers, we take $y_1=[1:1:1:0]$\footnote{One can make a change of coordinates such that $y_1$ lies on the conic while fixing the camera centers and the two planes.}. Now, every possible union of two conics lying in these planes and containing the four camera centers and $y_1$ is given as the intersection of two quadrics on the form:
\begin{align*}
S_P^{12}&=\begin{bsmallmatrix}
0&0&0&a_1\\
0&0&0&a_2\\
0&0&0&a_3\\
-a_1&-a_2&-a_3&0
\end{bsmallmatrix}, & S_P^{14}&=\begin{bsmallmatrix}
0&b_1&b_2&b_3\\
b_1&0&-b_1-b_2&b_4\\
b_2&-b_1-b_2&0&-b_3-b_4\\
b_3&b_4&-b_3-b_4&0
\end{bsmallmatrix}.
\end{align*}
Using the construction described in the main text, we take the remaining four quadrics to be:
\begin{align*}
S_P^{13}&=S_P^{23}=S_P^{12}\\
S_P^{24}&=(a_1b_1b_4 - a_2b_2b_3 - a_2b_2b_4 + a_3b_1b_4)S_P^{12}+a_1a_2b_1b_2S_P^{14}\\
S_P^{34}&=(a_1b_2b_3 + a_1b_2b_4 + a_2b_2b_3 + a_2b_2b_4 - a_3b_1b_4)S_P^{12}-a_1a_3b_1S_P^{14}
\end{align*}
Then these six quadrics come from the fundamental matrices:
\begin{align*}
F_Q^{12}&=\begin{bsmallmatrix}
  0& 0& a_2\\
 0&  0& a_1+a_3\\
a_1& -a_1&0
\end{bsmallmatrix}, \quad F_Q^{13}=\begin{bsmallmatrix}
  0& 0& a_1+a_2\\
 0&  0& a_3\\
a_1& -a_1&0
\end{bsmallmatrix}, \quad F_Q^{14}=\begin{bsmallmatrix}
b_1&0& -b_1\\
b_2& -b_2&0\\
b_3&  b_4& - b_3 - b_4
\end{bsmallmatrix},\\
F_Q^{23}&=\begin{bsmallmatrix}
  0&0& a_1b_2 + a_2b_2 - a_3b_1\\
  0&0& a_3(b_1 + b_2)\\
a_1b_1 - a_2b_2 + a_3b_1& a_2(b_1 + b_2)&0
\end{bsmallmatrix}&\\
F_Q^{24}&=\begin{bsmallmatrix}\vspace{4pt}
  0&  a_1a_2b_1b_2&  a_1b_1b_2(a_1 + a_3)\\ \vspace{4pt}
-a_1b_2(a_1b_1 - a_2b_2 + a_3b_1)&-a_1a_2b_2(b_1 + b_2)& 0\\
 a_1b_4(a_1b_1 - a_2b_2 + a_3b_1)& \begin{smallarray}{c}
a_2(a_1b_1b_4 + a_1b_2b_4 - a_2b_2b_3\\
 - a_2b_2b_4 + a_3b_1b_4) 
\end{smallarray}& \begin{smallarray}{c}
-(a_1 + a_3)(a_2b_2b_3\\
 + a_2b_2b_4 - a_3b_1b_4)
\end{smallarray}
\end{bsmallmatrix},\\
F_Q^{34}&=\begin{bsmallmatrix}\vspace{4pt}
0&  a_1b_1b_2(a_1 + a_2)& a_1a_3b_1b_2\\ \vspace{4pt}
 -a_1b_1(a_1b_2 + a_2b_2 - a_3b_1)& 0&  -a_1a_3b_1(b_1 + b_2)\\
\begin{smallarray}{c}
-a_1(b_3 + b_4)(a_1b_2\\
 + a_2b_2 - a_3b_1)
\end{smallarray}&\begin{smallarray}{c}
 -(a_1 + a_2)(a_2b_2b_3\\
  + a_2b_2b_4 - a_3b_1b_4)
\end{smallarray}&\begin{smallarray}{c}
 -a_3(a_1b_1b_3 + a_1b_1b_4\\
  + a_1b_2b_3 + a_1b_2b_4 + a_2b_2b_3\\
   + a_2b_2b_4 - a_3b_1b_4) 
\end{smallarray}
\end{bsmallmatrix},\\
\end{align*}
which can be verified to be compatible (see attached code).

\textit{Continued proof of \cref{itemp3}} (the case of two cameras on each conic):

Let $[x_1:x_2:x_3:x_4]$ denote a generic point in $\p3$. We can take the two planes spanned by the two conics to be the planes where $x_3=x_4$ (the plane containing $p_1,p_2$) and $a_1x_1=-a_2x_2$ (the plane containing $p_3,p_4$)\footnote{Here we assume that if any camera center lies in both planes, it has to be either $p_3$ or $p_4$. We can always ensure that this is the case by a relabelling of the cameras}. Now, every possible union of two conics lying in these planes and containing the four camera centers and $y_1$ is given as the intersection of two quadrics on the form:
\begin{align*}
S_P^{12}&=\begin{bsmallmatrix}
0&0&a_1&-a_1\\
0&0&a_2&-a_2\\
a_1&a_2&0&0\\
-a_1&-a_2&0&0
\end{bsmallmatrix}, & S_P^{13}&=\begin{bsmallmatrix}
0&b_1&b_2&b_3\\
b_1&0&b_4&-b_1-b_3\\
b_2&b_4&0&-b_2-b_4\\
b_3&-b_1-b_3&-b_2-b_4&0
\end{bsmallmatrix}.
\end{align*}
Using the construction described in the main text, we take the remaining four quadrics to be:
\begin{align*}
S_P^{23}&=(a_2b_2b_3-a_1b_1b_4+a_2b_3b_4)S_P^{12}+(-a_1a_2(b_3-b_4))S_P^{13}\\
S_P^{14}&=(b_1+b_2+b_4)S_P^{12}+(a_1+a_2)S_P^{13}\\
S_P^{24}&=(a_1b_1+a_1b_3-a_2b_2)(a_2b_2-a_1b_1+a_2b_4)S_P^{12}-a_1a_2(a_1+a_2)(b_3-b_4)S_P^{13}\\
S_P^{34}&=S_P^{12}
\end{align*}
Then these six quadrics come from the fundamental matrices:
\begin{align*}
F_Q^{12}&=\begin{bsmallmatrix}
  0& a_4& -a_4\\
 a_2&  0& -a_2\\
-a_2& a_2&0
\end{bsmallmatrix}, \quad F_Q^{13}=\begin{bsmallmatrix}
b_1&0& -b_1\\
b_2&  b_4& - b_2 - b_4\\
b_3& -b_3&0
\end{bsmallmatrix},\\
F_Q^{23}&=\begin{bsmallmatrix}\vspace{4pt}
  0&  -a_2a_4b_1(b_3 - b_4)&  -a_2^2b_1(b_3 - b_4)\\ \vspace{4pt}
 a_2b_4(a_4b_2 - a_2b_1 + a_4b_3)& \begin{smallarray}{c}
a_4(a_2b_4^2 - a_2b_1b_4\\
 - a_2b_3b_4 + a_4b_2b_3 + a_4b_3b_4)
\end{smallarray}& a_2a_4(b_2 + b_4)(b_3 - b_4)\\
-a_2b_3(a_4b_2 - a_2b_1 + a_4b_3)&\begin{smallarray}{c}
-a_4b_3(a_2b_4 - a_2b_3 - a_2b_1\\
 + a_4b_2 + a_4b_4)
\end{smallarray}  & 0
\end{bsmallmatrix},\\ 
F_Q^{14}&=\begin{bsmallmatrix}\vspace{4pt}
 b_1(a_2 + a_4)& 0&  -b_1(a_2 + a_4)\\ \vspace{4pt}
 a_4b_2 - a_2b_4 - a_2b_1&  a_2b_1 + a_2b_4 - a_4b_2&  0\\
a_2b_1 + a_2b_2 + a_2b_3 + a_2b_4 + a_4b_3& a_4b_2 - a_2b_3 - a_2b_1 - a_4b_3 + a_4b_4& -(a_2 + a_4)(b_2 + b_4)
\end{bsmallmatrix},\\
F_Q^{24}&=\begin{bsmallmatrix}\vspace{4pt}
  0&  a_2a_4b_1(a_2 + a_4)(b_3 - b_4)& a_2^2b_1(a_2 + a_4)(b_3 - b_4)\\ \vspace{4pt}
\begin{smallarray}{c}
-a_2(a_2b_1 + a_2b_4\\
 - a_4b_2)(a_4b_2\\
  - a_2b_1 + a_4b_3)
\end{smallarray}&                                         -\begin{smallarray}{c}
a_4(a_2b_1 + a_2b_4 - a_4b_2)(a_2b_4\\
 - a_2b_3 - a_2b_1 + a_4b_2 + a_4b_4)
\end{smallarray}&                                    0\\
\begin{smallarray}{c}
a_2(a_4b_2 - a_2b_1\\
 + a_4b_3)(a_2b_1 + a_2b_3\\
  - a_4b_2 + a_4b_3 - a_4b_4)
\end{smallarray}&\begin{smallarray}{c}
-a_4(a_2^2b_1^2 + 2a_2^2b_1b_3 - b_4a_2^2b_1\\
 + a_2^2b_3^2 - b_4a_2^2b_3 - 2a_2a_4b_1b_2\\
  + a_2a_4b_1b_3 - 2b_4a_2a_4b_1 - a_2a_4b_2b_3\\
   + a_2a_4b_3^2 - 2b_4a_2a_4b_3 + a_4^2b_2^2 + b_4a_4^2b_2)
\end{smallarray} &\hspace{-0.2cm}\begin{smallarray}{c}
-a_2a_4(a_2 + a_4)(b_2\\
 + b_4)(b_3 - b_4)
\end{smallarray} 
\end{bsmallmatrix},\\
F_Q^{34}&=\begin{bsmallmatrix}
           0&       -a_2b_1& -a_2(b_2 + b_4)\\
       a_2b_1&            0& -a_4(b_2 + b_4)\\
a_2(b_2 + b_4)& a_4(b_2 + b_4)&             0
\end{bsmallmatrix},
\end{align*}
which can be verified to be compatible (see attached code). 
\end{proof}

\twocolumn
\addcontentsline{toc}{section}{References}
\bibliography{references}

\end{document}